\documentclass{report}
\usepackage{silence}
\usepackage[table,dvipsnames]{xcolor}
\definecolor{mydarkblue}{rgb}{0,0.08,0.45}
\definecolor{note_fontcolor}{rgb}{0.800781, 0.800781, 0.800781}

\usepackage{enumerate}

\usepackage{wrapfig}

\PassOptionsToPackage{numbers, compress, sort}{natbib}

\usepackage[final]{neuripsx}
\usepackage{bbm}
\usepackage[utf8]{inputenc} %
\usepackage[T1]{fontenc}    %
\usepackage[hyphens]{url}            %
\usepackage[
    colorlinks,
    bookmarks=true,
    breaklinks=true,
    linkcolor=mydarkblue,
    citecolor=mydarkblue,
    filecolor=mydarkblue,
    urlcolor=mydarkblue,]{hyperref}       %

\usepackage{bookmark}
\usepackage{booktabs}       %
\usepackage{amsfonts}       %
\usepackage{nicefrac}       %
\usepackage{microtype}      %
\usepackage{multicol}
\usepackage{subcaption}
\usepackage{amssymb, amsmath, amsthm}
\usepackage{thmtools}
\usepackage{mathtools}
\usepackage{thm-restate}

\usepackage{stmaryrd}
\usepackage{babel}

\usepackage{scalerel}

\usepackage{tcolorbox}
\usepackage{thm-restate}
\usepackage{subfiles}

\usepackage{wrapfig}

\usepackage{breakurl}
\usepackage[capitalize,nameinlink]{cleveref}
\usepackage{xparse}
\usepackage{bbm}

\usepackage{macros}

\newtheorem{thm}{Theorem}[section]
\newtheorem{cor}[thm]{Corollary}

\newtheorem{assm}[thm]{Assumption}

\newtheorem{setup}[thm]{Setup}
\newtheorem{lemma}[thm]{Lemma}
\newtheorem{prop}[thm]{Proposition}
\theoremstyle{definition}
\newtheorem{defn}[thm]{Definition}
\theoremstyle{remark}
\newtheorem{rem}[thm]{Remark}
\newtheorem{exmp}[thm]{Example}

\title{Tensor Programs IVb:\\
Adaptive Optimization in the $\infty$-Width Limit}

\author{
    \begin{minipage}{.4\textwidth}
        \centering
        Greg Yang\\ \normalfont
        xAI
    \end{minipage}
    \begin{minipage}{.4\textwidth}
        \centering
        Etai Littwin\\ \normalfont
        Apple
    \end{minipage}
}

\begin{document}

\maketitle

\begin{abstract}

  Going beyond stochastic gradient descent (SGD), what new phenomena emerge in wide neural networks trained by adaptive optimizers like Adam?
  Here we show: 
  The same dichotomy between feature learning and kernel behaviors (as in SGD) holds for general optimizers as well, including Adam --- albeit with a \emph{nonlinear} notion of ``kernel.''
  We derive the corresponding ``neural tangent'' and ``maximal update'' limits for any architecture.
  Two foundational advances underlie the above results:
  1) A new Tensor Program language, \nexort{}, that can express how adaptive optimizers process gradients into updates.
  2) The introduction of bra-ket notation (borrowed from quantum physics) to drastically simplify expressions and calculations in Tensor Programs.
  This work summarizes and generalizes all previous results in the \emph{Tensor Programs} series of papers.

\end{abstract}

\begin{center}
\includegraphics{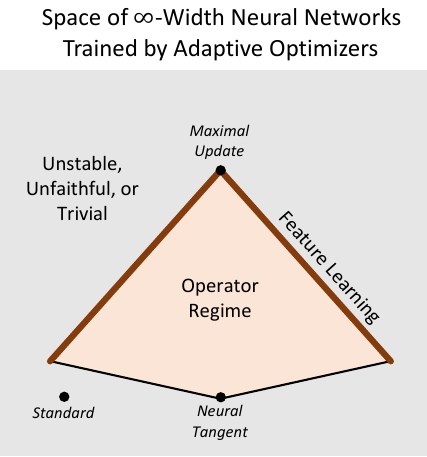}
\end{center}

\pdfbookmark[section]{Table of Contents}{Contents}
\tableofcontents

\chapter{Introduction}

While historically the deep learning theory literature has by-and-large (carelessly in hindsight) identified ``infinite-width neural networks'' with ``neural tangent kernel'' \citep{wide,arora} or ``gaussian process'' \citep{gp4,gp5,gp6,gp7}, by now we understand these are just particular kinds of infinite-width limit with simple mathematics.
Indeed, there also exists a ``feature learning regime'' with much more complex mathematics but also all the actually desired properties of a neural network \citep{lazy,mf1,mf2,mf3,mf4,mf5}.
\citet{yang4} precisely characterized this so-called \emph{Dynamical Dichotomy}: there is no other regime that can happen for MLPs trained by SGD in finite time.

\paragraph{Does Adaptive Optimization Create New Large Width Behavior?}
In practice, of course, most neural networks of importance are trained by adaptive optimizers like Adam \citep{adam}.
Can new phenomenon arise in the infinite-width limit of adaptive optimizers?
For example, if one invokes the (not-quite-correct) intuition that neural network trained with small learning rate exhibits kernel behavior, then one might suppose that these optimizers may adaptively enforce large effective learning rates.
This prevents kernel behavior and may even ``supercharge'' feature learning.

But it turns out there’s nothing special about adaptive optimizers in this regard.
Essentially the exact dichotomy of feature learning vs kernel regime plays out for any adaptive optimizer.
This means there is also a “neural tangent kernel” limit for any adaptive optimizer, where the network evolution can be captured solely by some kind of evolution equation in the function space — albeit no longer a linear equation.
This also means that “maximal update” parametrization can be defined for e.g., Adam, that maximizes feature learning in a suitable sense.
Indeed, \cite{yang5} already contained an intuitive derivation of this $\mu$P for Adam, and showed it preserves optimal hyperparameters as ones scales the width of a model (e.g., Transformer \citep{attention2}).
Part of this work serves to fill the gap in \cite{yang5}'s theoretical foundations.

\paragraph{Tensor Program with Nonlinear Outer Products}

To achieve this result, we leverage the Tensor Programs framework: express the adaptive optimization of a network in any parametrization in a Tensor Program, and invoke the Master Theorem to take the infinite-width limit of the whole computation, obtaining in particular the limit at the end of training.

Yet, there is one problem: no previous Tensor Program language can express adaptive optimization! The main issue is the expression of the entrywise “normalization” of the gradient done by the optimizer: For example, in the first step of training, Adam essentially just takes the sign of the gradient; while previous Tensor Program languages like \netsort{} can express this operation for the input and output weights, they cannot do so for the hidden weights.%
\footnote{More generally, \netsort{} can only express this for vector-like (1 dimension tends to infinity) parameters but not matrix-like (2 dimensions).}

Thus, we extend \netsort{} (pronounced ``NETS-ert'') to a new language \nexort{} (pronounced ``NEK-zort'') by allowing such operations. More precisely, this operation can be construed as a “nonlinear outer product” of vectors.%
\footnote{This was called ``nonlinear \emph{tensor} product'' in \citep{yang5}, but here we adopt the term \emph{outer} to avoid over-using the word \emph{tensor}.}
For example, a nonlinear outer product $z$ of two vectors $x, y \in \R^n$ has $z_{\alpha\beta} = \phi(x_\alpha, y_\beta)$ for some $\phi: \R^2 \to \R$.
In its full generality, nonlinear outer products express the the gradient processing done by Adam and other optimizers.
\nexort{} can be just thought of as ``\netsort{} $+$ nonlinear outer products.''

We prove the Master Theorem for \nexort{} (i.e., how to take the infinite width limit for any \nexort{} program), for both the classical Gaussian case (where matrices are sampled from Gaussians) as well as the general non-Gaussian case, following the strategy of \cite{tp3b}.

\paragraph{Infinite Width Limits for Any Architecture}

The Tensor Programs series is known for \emph{architecturally universal} results — results that hold for all “natural deep learning architectures”, past and future. For example, \cite{yang} established the architectural universality of Neural Network-Gaussian Process (NNGP) Correspondence, \cite{yang2,tp2b} established the same for the Neural Network-Tangent Kernel Correspondence, and \cite{yang3} likewise for Free Independence Principle. Yet, \cite{yang4}’s theoretical development of maximal update parametrization only focused on multi-layer perceptrons.%
\footnote{with a brief discussion of defining $\mu$P for any architecture in the appendix, but nothing about its limit; \cite{yang5} did a thorough empirical investigation of muP for a variety of architectures but nothing theoretical}

Here we write down the $\mu$-limit (as well as the neural tangent limit) for any architecture and any adaptive optimizer. The key innovation here is the \emph{definition} of what “any architecture” means, while the proofs follow essentially the MLP examples.

\paragraph{Notational Advances}
Prior works did not write down the $\mu$-limits for all ``natural'' architectures mainly because the Tensor Programs notation was not efficient enough to deal with this arbitrary complexity. The mundane looking but vital innovation of this work is a new set of Tensor Programs notations that enables concise expression of all of the above: the bra-ket (aka Dirac) notation, borrowed from quantum physics. For readers familiar with prior Tensor Programs papers, in short: $$\ket x= Z^{x},\quad\braket xy=\EV Z^{x}Z^{y}.$$
The expectation inner product becomes notably succinct in the new notation, which also enables much more efficient expressions of the nonlinear outer product that is at the center of adaptive optimization.

\paragraph{Contributions}
\begin{itemize}
  \item We formalize a general notion of adaptive optimization in deep learning, called \emph{entrywise updates} --- the property that gradients are processed entrywise --- satisfied by common optimizers like SGD and Adam \citep{adam}.
  \item We define the neural tangent and maximal update parametrizations for entrywise optimizers and derive their infinite-width limits. While we focus on MLPs in most of this paper for pedagogical purposes, we eventually write down the limits for any ``reasonable'' architecture.
  \item More generally, like \cite{yang4}, we identify all ``natural'' infinite-width limits in this setting and dichotomize them into feature learning vs a nonlinear version of kernel regime. The maximal update limit remains the ``optimal'' feature learning limit for all entrywise optimizers.
  \item All of the above results are made possible by a new version of Tensor Program, called \nexort{}, that can express the adaptive updates using the new instruction of \emph{nonlinear outer product}. This forms the bulk of the technical advances made in this work.
  \item Even so, the most vital contribution of this work is perhaps introducing the bra-ket notation to drastically simplify expressions and calculations common in Tensor Programs. 
\end{itemize}
The infinite-width limits of adaptive optimizers take the headline for this paper and may be what piques readers' interest in the near term.
But the new Tensor Program and the new notation will likely have longer lasting impact, pushing forward our fundamental knowledge of large neural network behavior and lowering the translation tax between how this knowledge is stored on paper and in our heads.

\paragraph{The \emph{Tensor Programs} Series}
In a way, this work gathers and generalizes all previous results in the \emph{Tensor Programs} series about infinite-width limits:
\cref{thm:MemorylessGeneralNTKLimit} generalizes the architecturally universal Neural Network-Gaussian Process correspondence \citep{yang} and Neural Network-Tangent Kernel correspondence \citep{yang2,tp2b};
the new Tensor Program, \nexort{} (\cref{defn:nexort}), generalizes the \netsortp{} of \cite{yang3};
its Master Theorem (\cref{thm:MasterTheorem}) holds for both Gaussian and non-Gaussian matrices, generalizing \cite{tp3b};
the Dynamical Dichotomy for entrywise updates (\cref{cor:dichotomyMain}) generalizes that for SGD \citep{yang4};
our $\mu$-limit equations generalize those of \cite{yang4} for SGD and, for the first time, we even write down the general $\mu$-limit equations for any architecture (\cref{thm:mulimit_anyarch}).

\section{Related Work}\label{sec:rw}

\paragraph{Infinite-Width Neural Networks}
Here we briefly overview past works on infinite-width neural networks, but we recommend the reader to refer to \cite[Sec 2]{yang4} for a more comprehensive review.
A large body of literature exists on both the kernel (NTK) limit \cite{ntk,wide,arora,yang2,tp2b} and the mean field limit for 2-layer neural network \cite{mf1,mf2,mf3,mf4,mf5}. Various papers describe the kernel and feature learning regimes more generally without taking an infinite-width limit. \cite{lazy} describes the ``lazy training'' regime in arbitrary differentiable programs, and is controlled by a single parameter $\alpha$ which scales the output. It is shown that when $\alpha$ is large, the weight need only move slightly to fit the training data, and network essentially performs kernel learning. Many papers \cite{AllenZhu2019ACT,Huang2020DynamicsOD,Mei2018AMF} view the kernel and feature learning regimes as learning in different timescales, explicitly incorporating the time dependence in the infinite-width limit, and others derive finite width corrections to the NTK for finite width networks \cite{boris1, resntk}. 
In this paper, as in \cite{yang4}, we consider training time to be constant, and take only the width to infinity. 

\paragraph{Tensor Programs}
Tensor Programs, first introduced in \cite{yangScalingLimit} and expanded upon in \cite{yang,yang2,yang3}, were developed as a theoretical framework to analyze the infinite-width limits of any architecture expressible in a simple formal language, in an attempt to unify the per-architecture analysis prevalent in the literature \cite{Alemohammad2021TheRN,Du2019GraphNT,Littwin2020OnIH,infiniteattention}. 
\cite{yang4} defined a natural space of neural network parametrizations (abc-parametrizations), and classified all resulting infinite-width limits into two possible catagories: 1) the kernel regime, in which the neural network function evolves as a linear model, and 2) the feature learning regime, in which the representations change and adapt to data over the course of training. The $\mu$ parametrization was then identified as the ``optimal'' parametrization for arbitrary architectures in which all layers learn features, and was later heuristically extended to adaptive optimizers \cite{yang5}.

\paragraph{Adaptive Optimizers}
Adaptive optimizers \cite{adam,Duchi2010AdaptiveSM,Zhou2018ADAPTIVELR} and their variants were developed to accelerate learning by adapting the learning rate on a per parameter basis, and currently is a critical component of large scale pretraining of transformer models \cite{Huang2020ImprovingTO,Liu2020UnderstandingTD,Zhang2019WhyAB}. 
No previous work has developed their theory for infinite-width neural network, but a concurrent work has derived the infinite-width NTK for SignSGD in the batch-size 1 setting \citep{malladi_kernel-based_2022} (which is not equivalent to the general batch-size setting).

\section{Notations}
\label{sec:notations}

One of the the key innovations of this work is a set of much cleaner notations to express ideas in Tensor Programs.
While this sizeable section may be off-putting to some readers, it's better to explain the notation sooner than later.
We recommend skimming until the end of the \emph{Outer Product} subsection and then move on, coming back to read other parts when necessary.

\paragraph{Multi-vectors and Multi-scalars}
Throughout this paper, we expect $n$ to be large. If $x \in \R^n$, then we say $x$ is an \emph{$n$-vector}, or just \emph{vector}.
If $\xx \in \R^{n \times k}$ where $k$ is fixed as $n \to \infty$, then we say $\xx$ is a \emph{multi-vector}, with the intuition that $\xx$ can be thought of as a $k$-tuple of vectors.
Likewise, $c \in \R$ is a \emph{scalar}, but we will call $\cc \in \R^k$ (where $k$ is fixed as $n \to \infty$) a \emph{multi-scalar} (not a vector), with the intuition that $\cc$ can be thought of as a $k$-tuple of scalars.
A more formal definition is given in \cref{sec:multitensor}.

\paragraph{Averaging over $n$}
When $x \in \R^n$, we always use greek subscript $\alpha, \beta, \ldots \in [n]$ to index its entries.
Then $\la x_\alpha \ra_\alpha$ denotes its average entry.
This notation will only be used to average over $n$-dimensions, but not over constant dimensions.

\subsection{The Tensor Program Ansatz: Representing Vectors via Random Variables}
\label{sec:TPansatz}

As we will see, as width becomes large, the entries of the (pre-)activation
vectors and their gradients will become roughly iid (just like in
the SGD case), both at initialization (which is easy to see) and training
(which is harder to see). Hence any such vector's behavior can be
tracked via a random variable that reflects the distribution of its
entries. While we call this the ``Tensor Program Ansatz'', it is
a completely rigorous calculus as seen below in \cref{sec:nexort} (as well as in previous
papers in the Tensor Programs series for SGD).

\subsubsection{Ket Notation}

Concretely, if $x\in\R^{n}$ is one such vector, then we write $\ket x \in \R$ (called a \emph{ket})
for such a random variable, such that $x$'s entries look like iid
samples from $\ket{x}$.%
\footnote{
  The reader may wonder why we write $\ket{x}$ instead of the more conventional $|x\ra$ in quantum mechanics.
  This is mainly because later (\cref{defn:oplim}) we want to write $\oplim{W}$ for the ``limit'' of a matrix $W$ so that $\oplim W x\ra = \ket{Wx}$, but if we used $|$ instead of $\ob$, then $|W|$ looks too much like some norm of $W$.
  }
For any two such vectors $x, y \in \R^n$, $(x_\alpha, y_\alpha) \in \R^2$ for each $\alpha$ will look like iid samples from the random vector $(\ket x, \ket y)$, such that, for example, $\lim_{n\to\infty}\frac{x^{\top}y}{n}=\EV \ket{x}\cdot \ket{y}$, which we write succinctly as just $\braket x y$.
Here $\bra x$ is called a \emph{bra}, interpreted as a sort of ``transpose'' to $\ket x$.
In our convention, $\ket x$ is always a random variable independent of $n$ and $x$ always has $\Theta(1)$ typical entry size.\footnote{i.e., $\|x\|^{2}/n=\Theta(1)$ as $n\to\infty$}.

\paragraph{Multi-Vector Kets}
Furthermore, this
notation cleanly handles the multi-vector case when $\xx=(x^{1},\ldots,x^{k})$
is an $n\times k$ matrix where $k$ is fixed as $n\to\infty$: 
\[
\ket{\xx}=(\ket{x^{1}},\ldots,\ket{x^{k}})\in\R^{k},\quad\braket{\xx}{\yy}\in\R^{k\times l},
\]
if $\yy$ has shape $n\times l$. Here $\ket{\xx}$, a ket, should be thought
of as a row vector (and its ``transpose'' $\bra{\xx}$, a bra, as a column vector),
so that $\ket{\xx}\vv=\sum_{i}\ket{x^{i}}v^{i}$
for any vector $(v^{1},\ldots,v^{k})\in\R^{k}$. Generally, the intuition
is that in the expression $\ket{\xx}$, the $\ob$ side represents
the $n$-dimension (in the limit as $n\to\infty$) while the $\rangle$
side represents the $k$-dimension. Thus 
\begin{center}
$\braket{\xx}{\yy}$ represents
the limit of $\xx^{\trsp}\yy/n$.%
\footnote{Note that later, we will consider $\xx$ of shape $n \times k_1 \times \cdots \times k_r$, in which case $\ket{\xx}$ and $\bra{\xx}$ both have shape $k_1 \times \cdots \times k_r$, and $\braket{\xx}\xx$ has shape $k_1 \times \cdots \times k_r \times k_1 \times \cdots \times k_r$.}
\end{center}

Because we will often need to multiply a ket with a diagonal matrix, we introduce a shorthand:
\begin{equation}
  \ket\xx _{\cchi} = \ket\xx \Diag(\cchi),\label{eqn:multdiagnotation}
\end{equation}
if $\xx$ is $n\times k$ and $\cchi$ is a $k$-dimensional vector.

\subsubsection{Outer Product}
\label{sec:ket_outer}

Likewise, if both $\xx$ and $\yy$ have shape $n\times k$, the expression
\[
\ket{\xx}\bra{\yy}\ \text{ represents the limit of $\xx\yy^{\trsp}\in\R^{n\times n}$.}
\]
More formally,
$\ket{\xx}\bra{\yy}$ is defined as an operator that takes a ket $\ket{\zz}\in\R^{j}$
and return the ket
\[
(\ket{\xx}\bra{\yy})\ket{\zz}=\ket{\xx}(\braket{\yy}{\zz})\in\R^{j}
\]
i.e., it returns the random vector $\ket{\xx}\in\R^{k}$ multiplied
by the deterministic matrix $\braket{\yy}{\zz}\in\R^{k\times j}$
on the right. 
This corresponds to the limit of $\xx \yy^\trsp \zz/n$.
Likewise, $\ket{\xx}\bra{\yy}$ acts on a bra $\bra{\ww}\in\R^{j}$
by
\[
\bra{\ww}(\ket{\xx}\bra{\yy})=(\braket{\ww}{\xx})\bra{\yy}\in\R^{j}.
\]
which corresponds to the limit of $\frac 1 n \ww^\trsp \xx \yy^\trsp$.
This definition of $\ket{\xx}\bra{\yy}$ makes the expressions
\[
\ket{\xx}\braket{\yy}{\zz},\quad\braket{\ww}{\xx}\bra{\yy},\quad\braket{\ww}{\xx}\braket{\yy}{\zz}
\]
unambiguous (since any way of ordering the operations give the same
answer).

\begin{rem}[Potential Confusion]
One should \emph{not} interpret $\ketdbra{\xx}{}\yy$ as the scalar random variable $\ket\xx \cdot \ket \yy  = \sum_{i=1}^k \ket {x^i}\ket{y^i}$, which would act on a ket $\ket{\zz}$ to produce $(\bra{\xx} \cdot \bra\yy)\ket{\zz} = \EV (\ket\xx \cdot \ket \yy)\ket{\zz} $, which is deterministic.
On the other hand, $\ketdbra{\xx}{}\yy \zz \ra$ is always a linear combination of $\ket \xx$, a nondeterministic random variable in general.
In particular, any correlation between $\ket \xx$ and $\ket \yy$ does not directly play a role in their outer product $\ketdbra\xx{}\yy$: we always have $\ketdbra{\xx}{}\yy \zz \ra = \ketdbra{\xx}{}\yy^{\bx1} \ket\zz^{\bx1}$, where $(\ket \yy^{\bx1}, \ket \zz^{\bx1})$ is an iid copy of $(\ket \yy, \ket \zz)$ independent from $\ket \xx$. (See \cref{sec:notation_iid} below for more comment on this notation).
\end{rem}

\paragraph{Outer Product with Diagonal Inserted}
Finally, if $\cchi\in\R^{k}$ is deterministic, then (consistent with \cref{eqn:multdiagnotation}) we define $\ketdbra{\xx}{\cchi}{\yy}$
as the operator that acts on kets $\ket{\zz}\in\R^{j}$ by
\[
(\ketdbra{\xx}{\cchi}{\yy})\ket{\zz}=\ketdbra{\xx}{\cchi}{\yy}{\zz}\ra=\ket{\xx}\Diag(\cchi)(\braket{\yy}{\zz})\in\R^{j}.
\]
Morally, $\ketdbra{\xx}{\cchi}{\yy}$ is just a shorter way of writing
$\ket{\xx}\Diag(\cchi)\bra{\yy}$ and represents the limit of $\xx\Diag(\cchi)\yy^{\trsp}$.
In particular, $\ketdbra{\xx}{\boldsymbol{1}}{\yy}=\ket{\xx}\bra{\yy}$.

\subsubsection{Nonlinear Outer Product}

If $x y^\trsp \in \R^{n\times n}$ is the (linear) outer product of two vectors $x \in \R^n$ and $y \in \R^n$, then $\phi(xy^\trsp)$, the entrywise application of nonlinear $\phi: \R \to \R$ to $xy^\trsp$, is a kind of \emph{nonlinear outer product}.%
\footnote{The general definition of \emph{nonlinear outer product} is given in \cref{def:nonlinouter}, but for the most part here we will only be concerned with this particular type of nonlinear outer product and its generalization below.}
Passing to the ket notation, in general we define $\phi\left(\ketdbra{\xx}{\cchi}{\yy}\right)$
as the operator that acts on kets as
\[
\phi\left(\ketdbra{\xx}{\cchi}{\yy}\right)\ket{\zz}\defeq\EV_{\bx{1}}\phi\left(\sum_{i=1}^{k}\chi^{i}\ket{x^{i}}\ket{y^{i}}^{\bx{1}}\right)\ket{\zz}^{\bx{1}}
\]
where $\left(\ket{y^{1}}^{\bx{1}},\ldots,\ket{y^{k}}^{\bx{1}},\ket{\zz}^{\bx{1}}\right)$
is an iid copy of $\left(\ket{y^{1}},\ldots,\ket{y^{k}},\ket{\zz}\right)$
independent from $\ket{\xx}$ and the expectation is taken only over
the former. 
This is just like, in the finite $n$ case,
\[
\phi\left(\xx \Diag(\cchi) \yy^\trsp \right)\zz/n = \phi\left(\sum_{i=1}^{k}\chi^{i} x^{i}{y^{i\trsp}}\right) {\zz}/n.
\]

Moreover, if $\ket{\ww}\in\R^{j},\ket{\zz}\in\R^{k}$,
then
\begin{align*}
\bra{\ww}\phi\left(\ketdbra{\xx}{\cchi}{\yy}\right)\ket{\zz} & =\bra{\ww}\phi\left(\ketdbra{\xx}{\cchi}{\yy}^{\bx{1}}\right)\ket{\zz}^{\bx{1}}\in\R^{j\times k}\\
 & =\EV\phi\left(\sum_{i=1}^{k}\chi^{i}\ket{x^{i}}\ket{y^{i}}^{\bx{1}}\right)\left(\ket{\ww}\otimes\ket{\zz}^{\bx{1}}\right)
\end{align*}
where $\otimes$ denotes outer product of vectors and expectation
is taken over everything.

More generally, if $\phi: \R^t \to \R$, then $\phi\Big(\ketdbra{\xx_1}{\cchi_1}{\yy_1}, \ldots, \ketdbra{\xx_t}{\cchi_t}{\yy_t}\Big)$ is an operator taking kets to kets, defined by
\begin{align*}
  \phi\Big(\ketdbra{\xx_1}{\cchi_1}{\yy_1}, \ldots, \ketdbra{\xx_t}{\cchi_t}{\yy_t}\Big) \ket \zz \defeq
  \EV_{\bx{1}}\phi\left(\sum_{i=1}^{k}\chi_1^{i}\ket{x_1^{i}}\ket{y_1^{i}}^{\bx{1}}, \ldots, \sum_{i=1}^{k}\chi_t^{i}\ket{x_t^{i}}\ket{y_t^{i}}^{\bx{1}}\right)\ket{\zz}^{\bx{1}}
\end{align*}

\begin{rem}[Potential Confusion]
Note $\phi(\ket \xx \bra \yy)$ is not the image of the operator $\ket \xx \bra \yy$ under $\phi$ in the continuous function calculus of operators, but rather a ``coordinatewise application'' of $\phi$.
For example, if $\phi(t) =t^2$, then $\phi(\ket x \bra y)$ \emph{is not} $\ket x \bra y x\ra \bra y$, the latter being what typically ``squaring an operator'' means, but rather $\ket{x}^2 \bra{y}^2 = \ket{x\odot x}\bra{y\odot y}$.
\end{rem}

\subsubsection{Bar Notation}

In later applications, when $\phi$ is an update function (such as $Q_t$ in \cref{eq:Qt}), this will be clear from context.%
\footnote{\label{footnote:bar_abbreviation}
  For example, the bar notation in $\Qketdbra{d\hh^l_t}{\cchi_t}{\xx^{l-1}_t}$ abbreviates $Q^l_t$ where $l,t$ are the same as in $d\hh^l_t$ inside.
}
Then we use the much lighter ``bar'' notation
\begin{equation}
  \begin{aligned}
    \overline{\ket\xx \vv} &= \phi(\ket\xx \vv)\\
    \Qketdbra{\xx}\cchi\yy &= \phi(\ketdbra\xx\cchi\yy)\\
    \Qketdbra{\xx}\cchi\yy \zz \ra &= \phi(\ketdbra\xx\cchi\yy) \ket \zz\\
    \la \ww \Qketdbra{\xx}\cchi\yy \zz \ra &= \bra \ww \phi(\ketdbra\xx\cchi\yy) \ket \zz.
  \end{aligned}
   \numberthis\label{eqn:barnotation}
\end{equation}

\subsubsection{Random Vector Calculation}

In contrast to the cases above, when an expression involves only kets
(or bras), then the usual calculus of kets as random variables or
vectors apply, e.g., $\ket{\xx}\odot\ket{\yy}$ is just the random
vector formed from entrywise product of $\ket{\xx}$ and $\ket{\yy}$.%
\footnote{From readers with quantum mechanics background, beware that $\ket x \ket y$ in our context is the \emph{product of random variables} $\ket x$ and $\ket y$, which is \emph{not equal} to their ``tensor product'' (which would be written $\ket x \ket y ^{\bx 1}$).}

\subsubsection{Comparison with Previous $Z^\bullet$ Notation}
For readers familiar with the \emph{Tensor Programs} papers, this new ``bra-ket'' notation
(aka Dirac notation) relates to the old $Z^\bullet$ notation by
\[
\ket x= Z^{x},\quad\braket xy=\EV Z^{x}Z^{y}.
\]
The new notation's succinctness of expectation inner product should already be apparent.
Furthermore, the old notation is not very compatible with multi-vectors whereas $\ket x$ makes it clear that $\ra$ represents the constant dimension side.
Consequently, (nonlinear) outer product is awkward to express in it, especially when its contraction with random variables requires an explicit expectation symbol $\EV$.

\subsection{IID Copies}
\label{sec:notation_iid}

As already seen above, if $\XX$ is any random object, then $\XX^{\bx{1}},\XX^{\bx{2}},\ldots$
denote iid copies of $\XX$: $\XX^{\bx{i}}\disteq\XX$ for all
$i$ and $\XX,\XX^{\bx{1}},\XX^{\bx{2}},\ldots$ are mutually
independent. If $\YY$ is another random object, then $\YY^{\bx{1}},\YY^{\bx{2}},\ldots$
are iid copies of $\YY$ that furthermore satisfy $(\XX^{\bx{i}},\YY^{\bx{i}})\disteq(\XX,\YY)$
for each $i$. $\EV_{\bx{i}}$ means taking the expectation over
the iid copies with superscript $\bx{i}$.

\subsection{Big-O Notation}
\label{sec:bigO}

\begin{rem}[Potential Confusion]
Following previous papers of the Tensor Programs series, we adopt the following semantics of big-O notation, which concerns the ``typical size'' of entries of a tensor rather than the norm of the tensor (as is the more common usage of big-O notation).
Therefore, the reader \emph{must} internalize this notation sooner rather than later to avoid confusion.
\end{rem}
\begin{defn}[Big-O Notation]\label{defn:BigO}
  Given a sequence $\xx = \{\xx(n)\}_{n=1}^\infty$ of random tensors, where $\xx(n)$ can have different shapes for different $n$, we write $\xx = \Theta(n^{-a})$ and say \emph{$\xx$ has coordinates (or entries) of size $\Theta(n^{-a})$} if there exist constants $A,B > 0$ such that almost surely,%
  \footnote{Here ``almost surely'' is with respect to the probability of the entire sequence $\xx$.}
  for sufficiently large $n$,
  \begin{equation}
    A \le \frac 1 {\#\xx(n)}\sum_{\aalpha}  \xx(n)_{\aalpha}^2 \le B
  \end{equation}
  where $\#\xx(n)$ is the number of entries in $\xx(n)$.
  We make similar definitions for $O(n^{-a})$ and $\Omega(n^{-a})$.
\end{defn}
Note the constants $A, B$ can depend on everything except $n$; in concrete contexts below, such constants can, for example, depend on neural network architecture, training time, optimizer, etc, but just not width.

Most often, $\xx$ will have ``approximately iid'' coordinates, so the notation $\xx=\Theta(n^{-a})$ can be interpreted intuitively to say $\xx$ has coordinates of ``empirical standard deviation'' $\Theta(n^{-a})$, which justifies the name.

We also define $\tilde O$ in a slightly nonstandard manner that is more suitable for our usage.
\begin{defn}\label{defn:tildeO}
  For a random sequence $\cc = \{\cc(n)\}_{n\ge 1}$ of fixed-sized vectors, we write $\cc = \tilde O(n^k)$ if $n^{-k - \eps} \cc \asto 0$ for every $\eps > 0$.
\end{defn}
Morally, $\tilde O(1)$ objects are those that are typically poly-logarithmic in norm, thus coinciding with the more conventional definition of $\tilde O(1)$ in computational complexity theory.
However, technically our notion is a bit more general since they allow any growth slower than any polynomial.

\chapter{Exposition of Main Results}

Here we explain our main results while later chapters will prove them.
We begin by isolating a concept that captures most of adaptive optimizers, namely \emph{entrywise optimizers} (\cref{sec:optimizer}), which forms the focus of this work.
By considering how to scale the learning rate, initialization, and multipliers (a so-called \emph{abcd-parametrization}), we catalogue all natural ways of taking infinite-width limits (\cref{sec:abcd}).
We study the archetypical examples, the (canonical generalizations of) neural tangent (NT) (\cref{sec:neuraltangent}) and the maximal update ($\mu$) (\cref{sec:maximalupdate,{sec:deepmaximalupdate}}) parametrizations, and describe their infinite-width limits.
More generally, we classify all possible limits of abcd-parametrizations (\cref{sec:classifyabcd}): while most parametrizations are degenerate in one way or another, the rest can be divided into the \emph{feature learning} and the \emph{operator} regimes, the latter being the nonlinear counterpart of kernel regime.
The $\mu$ and NT limits are respectively the ``maximal'' elements of each regime in that all parameters contribute to the function evolution.
Nevertheless, like in the SGD case, all operator regime limits, including the NT limit, do not learn features and trivialize transfer learning.
While all of the above stars the MLP as the instructional architecture, finally we write down the NT and $\mu$ limits for any architecture (\cref{sec:anyarch}).

Underlying these results is the new Tensor Program language, \nexort{}, that expresses the so-called \emph{nonlinear outer products} (\cref{sec:nexort}).
We formulate the algorithm, aka the \emph{Master Theorem}, to compute the infinite-width limit of any \nexort{} program.
New techniques, such as new notions of equivalence of random vectors, are needed to prove the Master Theorem;
we overview the proof in \cref{sec:proofsketch} before giving it in full in \cref{chap:masterproof}.

\section{Optimizers}
\label{sec:optimizer}

What does one mean by \emph{adaptive optimization}?
Both SGD and Adam, prototypical optimizers in deep learning, have \emph{entrywise updates}, where parameter updates take the form of a function of the current and/or past gradients.
This turns out to be a concept that captures most of adaptive optimizers.

\paragraph{Entrywise Updates}
Generically, if $g_{0},g_{1},...,g_{t}\in\R$ denote the gradients
of some scalar parameter $w\in\R$ at steps $0,1,...,t$, a general
notion of an update at step $t$ takes
the form 
\begin{equation}
  w_{t}-w_{t-1}=-\eta Q_{t}(g_{0},\ldots,g_{t})\label{eq:Qt}
\end{equation}
for some function $Q_{t}:\R^{t+1}\to\R$ which we call the \emph{update function}, where $\eta$ is the learning
rate.
We call \cref{eq:Qt} an \emph{entrywise update} and the corresponding optimizer an \emph{entrywise optimizer}.
For example, for SGD, $Q_{t}$ just returns $g_{t}$.
For SGD with momentum $\beta$,
  $Q_t(g_0, \ldots, g_t) = \beta^{t} g_0 + \cdots + \beta^0 g_t.$
These are examples of \emph{linear} entrywise updates, where $Q_t$ is linear.
On the
other hand, ``adaptive updates'' are generally nonlinear, where $Q_{t}$ typically takes the
form%
\footnote{Sometimes, the denominator is $\sqrt{v} + \epsilon$ instead. For practical purposes, there is no difference between these two versions.
However, \cref{eqn:Qexample} is differentiable at $v = 0$, satisfying our smoothness assumption \cref{assm:MLPsmooth} while the alternative is not.}
\begin{equation}
Q_{t}(g_{0},\ldots,g_{t})=\frac{m}{\sqrt{v+\epsilon^2}}\label{eqn:Qexample}
\end{equation}
where $m$ and $v$ are both functions of the past gradients $g_{0},\ldots,g_{t}$ and $\epsilon>0$ is there for numerical stability.
For example, in Adam \citep{adam}, $m$ and $v$ are respectively the exponential
moving averages of them and their squares, resulting in the following unwieldy expression:
\begin{equation}
  Q_t(g_0, \ldots, g_t) = 
\frac{\frac{1}{1-\beta_{1}^{t+1}}\sum_{s=0}^{t}(1-\beta_{1})\beta_{1}^{t-s} g_s}
{\sqrt{\frac{1}{1-\beta_{2}^{t+1}}\sum_{s=0}^{t}(1-\beta_{2})\beta_{2}^{t-s} g_s^2 +\epsilon^2}}, \tag{$\mathtt{Adam}$}
\label{eqn:Adam}
\end{equation}
where $\beta_1,\beta_2$ are Adam's momentum hyperameters.
We can also consider a simpler ``memoryless'' version of this, namely SignSGD \citep{bernstein_signsgd_2018}:
\begin{equation}
  Q_t(g_0, \ldots, g_t) = Q_t(g_t) = \frac{g_t}{\sqrt{g_t^2 + \epsilon^2}}.
  \tag{$\mathtt{SignSGD}$}
  \label{eqn:SignSGD}
\end{equation}

In the context of an $L$-hidden-layer MLP, we can more generally consider a collection $\QQ = \{Q^l_t:\R^{t+1}\to\R\}_{t\ge0, l \in [L+1]}$ of update functions, one for each layer $l$.
\begin{defn}\label{defn:memoryless_stationary}
  We say $Q^l_t:\R^{t+1}\to\R$ is \emph{memoryless} if $Q^l_t(g_{0},\ldots,g_{t})$ only depends on $g_t$
  for any $(g_{0},\ldots,g_{t})\in\R^{t+1}$.
  We say $\QQ$ is \emph{memoryless} if all $Q^l_t$ are memoryless.
\end{defn}
\begin{defn}
  We say a memoryless $\QQ$ is \emph{stationary} if $Q^l_t$ does not depend on $l$ or $t$.
  In this case, we just write $Q: \R \to \R$ instead of $\QQ$.
\end{defn}
We will also write \emph{memoryful} and \emph{nonstationary} for the opposite of \emph{memoryless} and \emph{stationary.}
In this sense, SGD and SignSGD are both memoryless and stationary but Adam is neither.
We will always present the memoryless stationary versions of our theorems first, as they carry across the main ideas.
The full version (i.e., memoryful nonstationary) will always be a straightforward modification, though often requiring more notations.

\paragraph*{Optimizer Coverage}
Many other optimizers are covered by this entrywise update framework \cref{eq:Qt}, including 
RMSProp,
Adagrad,
Adadelta,
NAdam,
Adamax,
etc \cite{adam,shazeer_adafactor_2018,Zhou2018ADAPTIVELR,Dozat2016IncorporatingNM,Loshchilov2017DecoupledWD}.
However, some other ingredients of ``adaptive'' optimization, such as gradient clipping, weight decay, or momentum factoring as in Adafactor \citep{shazeer_adafactor_2018}, are not directly covered.
Nevertheless, using our new extension of Tensor Programs discussed in \cref{sec:nexort}, it is straightforward to derive and classify the infinite-width limits including such ingredients, and we do so in \cref{sec:extensions}.
Such theorems will be by-and-large the same as what we have here (e.g., Dynamical Dichotomy \cref{cor:dichotomyMain} still holds), but the definitions of neural tangent and maximal update parametrizations (\cref{defn:NTP},\ref{defn:mup}) can change, as well as, e.g., the equations characterizing feature learning.
See \citep[Sec B.3]{yang5} for further intuitive discussions.

\paragraph{Insufficient Expressivity of Previous Tensor Programs}
If $Q_t$ in \cref{eq:Qt} is nonlinear, then \netsortp{}, the most advanced version of Tensor Programs before this work, is unable to express \cref{eq:Qt} for hidden weights of a network (weight matrices where both dimensions tend to infinity).
In this work, we extend \netsort{} to \nexort{} that \emph{can} express it and develop its Master Theorem.

\section{abcd-Parametrizations}
\label{sec:abcd}
In this work, we consider how such optimization should be parametrized
wrt the width of a neural network, generalizing the the abc-parametrization
of \cite{yang4}.

For concreteness, consider an $L$-hidden-layer biasless perceptron:
For weight matrices $W^{1}\in\R^{n\times d}$ and $W^{2},\ldots,W^{L}\in\R^{n\times n}$,
and nonlinearity $\phi:\R\to\R$, such a neural network on input $\xi\in\R^{d}$
is given by 
\begin{align}
  h^{l}(\xi)=W^{l}x^{l-1}(\xi)\in\R^{n},\quad x^{l}(\xi)=\phi(h^{l}(\xi))\in\R^{n},\quad\text{for \ensuremath{l=1,\ldots,L},}\label{eqn:MLP}
\end{align}
where $x^0(\xi) = \xi$
and the network output is $f(\xi)=W^{L+1\trsp}x^{L}(\xi)$ for $W^{L+1}\in\R^{n\times 1}$.
\begin{defn}\label{defn:abcd}
  Fix a set of update functions $\QQ = \{Q^l_t:\R^{t+1}\to\R\}_{t\ge0, l \in [L+1]}$.
  An \emph{abcd-parametrization} of the MLP in \cref{eqn:MLP} is specified by a set of numbers $\{a_{l},b_{l},c_{l},d_{l}\}_{l}$
  such that
  \begin{enumerate}[\hspace{30pt} (a)]
  \item We parametrize each weight as $W^{l}=n^{-a_{l}}w^{l}$ for actual
  trainable parameter $w^{l}$ 
  \item We initialize each $w_{\alpha\beta}^{l}\sim\Gaus(0,n^{-2b_{l}})$
  \item The learning rate is $\eta n^{-c_{l}}$ for some width-independent
  $\eta$
  \item The gradients of $w^{l}$ are multiplied by $n^{d_{l}}$ before being
  processed by $Q^l_t$: i.e., the update at time $t$ is
  \begin{equation}
  w^{l}\gets w^{l}-\eta n^{-c_{l}}Q^l_t(n^{d_{l}}g_{0},\ldots,n^{d_{l}}g_{t})
  \label{eqn:abcdupdate}
  \end{equation}
  where $g_{s},s=0,\ldots,t$, are the gradients of $w^{l}$ at time
  $s$ and $Q^l_t$ is applied entrywise.
  \end{enumerate}
\end{defn}

A simple example is the ``standard parametrization'' that is the default for, e.g., PyTorch, where nothing scales with width other than the initialization.

\begin{exmp}\label{exmp:SP}
  The \emph{standard abcd-parametrization (SP)} is defined by
\begin{center}
  \begin{tabular}{cccc}
    \toprule
    $l$ & $[2,L]$ & $1$  & $L+1$\tabularnewline
    \midrule
    $a_l$ & $0$ & $0$  & $0$\tabularnewline
    $b_l$ & $\nicefrac{1}{2}$  & $0$ & $\nicefrac{1}{2}$ \tabularnewline
    $c_l$ & $0$ & $0$  & $0$\tabularnewline
    $d_l$ & $0$ & $0$  & $0$\tabularnewline
    \bottomrule
  \end{tabular}
\end{center}
\end{exmp}

In \cref{defn:abcd}, beyond the obvious addition of scaling exponent $d_l$, compared to abc-parametrization of \cite{yang4}, we also now have layer dependent $c_l$.
This is without loss of generality, because of the redundancy in $a_l, b_l, c_l$, as shown in \cite[Eq 5]{yang4}.
This takes the more general form as follows for abcd-parametrization:
\begin{prop}[abcd Redundancy]
  For every $l \in [L+1]$,
\begin{center}
  for all $\theta \in \R$, at any finite width $n$, \emph{$\ensuremath{f_{t}(\xi)}$ stays fixed for all $\ensuremath{t}$ and $\ensuremath{\xi}$ if we set}
  \begin{equation}
    \ensuremath{a_{l}\gets a_{l}+\theta,\ b_{l}\gets b_{l}-\theta,\ c_{l}\gets c_{l}-\theta,\ d_{l}\gets d_{l}+\theta}.\label{stmt:SGDsymmetry}
  \end{equation}
\end{center}
\end{prop}

\begin{rem}[Reduction to abc-Parametrization for SGD]
In the case of SGD (i.e., when $Q^l_t(g_{0},\ldots,g_{t})=g_{t}$),
an abcd-parametrization reduces to an abc-parametrization with the
mapping
\begin{equation}
(a_{l},b_{l},c_{l})\gets(a_{l},b_{l},c_{l}-d_{l}).
\label{eqn:abc_reduction}
\end{equation}
\end{rem}

\begin{rem}[Omitted Constants]
As in \cite{yang4}, our concern here is the correct way \emph{to scale with width}.
In practice, there should be tunable hyperparameters in front of the powers of $n$ in \cref{defn:abcd}, as investigated in \cite{yang5}.
\end{rem}

\begin{rem}[Alternative Definition of $d_l$ for Adam]
In the idealized case of Adam and similar adaptive optimizers where the $\epsilon$ in \cref{eqn:Qexample} is 0, $Q^l_t$ is degree-0 homogeneous and $d_l$ itself is redundant.
When $\epsilon > 0$, this is no longer true.
But the \emph{almost homogeneity} yields an alterative but equivalent way to define $d_l$: instead of $g_s$ being multiplied by $n^{d_l}$, we let $\epsilon$ be multiplied by $n^{-d_l}$.
\end{rem}

\begin{rem}[For Adam, SP with Tuned LR Learns Features]
  If one sets the global learning rate for SP (\cref{exmp:SP}) to its largest stable value, then for SGD, SP is in the kernel regime \citep{yang4}.
  But for SignSGD and Adam, assuming perfect scale invariance, SP's largest stable learning rate is $\Theta(1/n)$, so that with this setting (i.e., setting $c_l = 1$ for every $l$), SP is actually in the feature learning regime.
  The reader is not expected to understand the underlying reasoning at this point, but the claim above can be derived by calculating $r=0$ from \cref{def:r} and invoking \cref{thm:stablefaithful} and \cref{thm:abcdclassification}.
  This difference in default scaling may be a contributing factor to the success of Adam compared to SGD.
\end{rem}

\section{Setup}
\label{sec:setup}

Here we set up the notation and conventions regarding the data, (pre)activations, and training of the network as well as the main technical assumptions for our rigorous results.

\paragraph{Data and (Pre)Activations}
Consider a set of $\NN$ inputs $\xxi\subseteq\R^{d}$ considered
as a $d\times\NN$ matrix whose columns $\xi^{a},a=1,\ldots,\NN$
represent individual inputs. Then we write $\ff_{t}\in\R^{\NN}$ for
the function outputs on these inputs, and write $\hh_{t}^{l}$ and
$\xx_{t}^{l}$ (with shape $n\times\NN$) for the pre- and post-activations
of all of such inputs in layer $l$ at time $t$. Similarly, let $d\hh_{t}^{l}$
and $d\xx_{t}^{l}$ (with shape $n\times\NN$) represent the scaled gradients $n^{a_{L}+b_{L}}\nabla_{\hh_{t}^{l}}\ff_{t}$
and $n^{a_{L}+b_{L}}\nabla_{\xx_{t}^{l}}\ff_{t}$ at time $t$ (this
scaling ensures that $d\hh_{t}^{l}$ and $d\xx_{t}^{l}$ have typical
size $\Theta(1)$ entrywise at initialization $t=0$).%
\footnote{Notation-wise, we do not bold the $d$ in $d\hh^l_t$ unless the output dimension $\dout$ is larger than 1, in which case $\db \hh^l_t$ represents $n^{a_{L}+b_{L}}J_{\hh_{t}^{l}}\ff_{t} \in \R^{n \times \dout \times \NN}$.
See \cref{defn:backpropagation}.}

\paragraph*{Error Signal and Training Routine}
While \cite{yang4} considered only the batch-size-1 case to simplify notation, here, because of our notational advance, we can afford to consider the following more general setting.
\begin{setup}\label{assm:trainingroutine}
  We consider the function evolution under an
abcd-parametrization and a sequence of error signal functions $\eps_{t}:\R^{\NN}\to\R^{\NN}\ (t\ge 0)$,
that returns the output error signal given the function values.
A \emph{training routine} is the package of 1) the learning rate $\eta$, 2) collection $\QQ$ of update functions (as in \cref{defn:abcd}) and 3) a sequence of $\eps_t$ as above.
\end{setup}
The $\eps_t$ error signals simultaneously encapsulate both full batch as well as mini-batch training.
For example, we can set $\eps_{t}(\ff)=\loss'(\ff,\yy)\odot batchmask_{t}$
for some loss function $\loss$ and a target vector $\yy\in\R^{\NN}$,
where $batchmask_{t}$ is the vector that is 1 on elements in the
batch at time $t$ and 0 otherwise.

Furthermore, we can implement train-test split via $\eps_t$:
For example,  $\xxi$
can be split into two parts, $\xxi=(\xxi^{train},\xxi^{test})$ such
that $\eps_{t}$ is always 0 on the $\xxi^{test}$.
Then the evolution of $\ff_t$ can track the evolution of function values on the test set due to changes from the training set.

The $\eps_t$ framework more generally covers settings like reinforcement and online learning where the error signal is not obtained from just a simple loss function.

\paragraph*{Technical Assumptions}
For all rigorous results in this work, we will consider the following smoothness assumption.
This is sufficient for deriving the NTK and $\mu$P limits but more assumptions, stated later, are required for Dynamical Dichotomy, i.e., the classification of abcd-parametrizations.
\begin{assm}[Smoothness]\label{assm:MLPsmooth}
  Assume $\phi'$, $\eps_t$, and $Q^l_t$ for all $l,t$ are pseudo-Lipschitz.%
  \footnote{Recall a function $f: \RR^k \to \RR$ is called \emph{pseudo-Lipschitz} if $|f(x) - f(y)| \le C\|x-y\|(1 + \sum_{i=1}^k |x_i|^d + |y_i|^d)$ for some $C,d>0$.
  This is morally the same as saying $f$ has a polynomially bounded derivative.
  }
\end{assm}

This is a very weak assumption satisfied by typical loss functions (e.g., MSE or cross entropy), update functions (e.g., SGD or Adam\footnote{specifically, Adam in the form \cref{eqn:Qexample} with $\epsilon >0$}), and nonlinearities (e.g., tanh or gelu).
The notable exception here is that relu itself is not covered because its derivative has a discontinuity.
But this is a common technicality not treated in the theoretical literature.
Nevertheless, we expect all theorems in this work should apply to relu as well and can be proven rigorously in the future.

With this setup in mind, we next describe the two prototypical infinite-width limits, the neural tangent and maximal update limits, for nonlinear entrywise updates before completely classifying the space of abcd-parametrization.

\section{Neural Tangent}
\label{sec:neuraltangent}

The ``classical'' neural tangent (abc-)parametrization (NTP) can be generalized easily to an abcd-parametrization using the intuition that the input to any nonlinear update function should be $\Theta(1)$ (we won't go through this calculation here but c.f.\ \cref{defn:faithful,{lemma:faithfulconditionsinit}}).
After defining this generalization next, we adapt the well-known continuous-time heuristic for deriving the NTK limit to the nonlinear updates case (\cref{eqn:NTQ_intuition}), before writing down a succinct expression of the infinite-width limit made possible by our new bra-ket notation (\cref{eqn:KQ,eqn:NTK_update_memoryless_stationary}).
\begin{defn}\label{defn:NTP}
  The \emph{neural tangent abcd-parametrization (NTP)} is defined (modulo \cref{stmt:SGDsymmetry}) by
\begin{center}
  \begin{tabular}{cccc}
    \toprule
    $l$ & $[2,L]$ & $1$  & $L+1$\tabularnewline
    \midrule
    $a_l$ & $\nicefrac{1}{2}$ & 0  & $\nicefrac{1}{2}$\tabularnewline
    $b_l$ & $0$  & $0$ & $0$\tabularnewline
    $c_l$ & $1$ & $\nicefrac{1}{2}$  & $\nicefrac{1}{2}$\tabularnewline
    $d_l$ & $1$ & $\nicefrac{1}{2}$  & $\nicefrac{1}{2}$\tabularnewline
    \bottomrule
  \end{tabular}
\end{center}
\end{defn}

\paragraph{Recovering ``Classical'' NTP}
Reducing to abc-parametrization in the SGD case via \cref{eqn:abc_reduction} (i.e., subtracting the last 2 rows), we recover exactly the classical NTP \citep[Table 1]{yang4}.

\subsection{Continuous Time Intuition}
If we consider $\ff_t\in \R^\NN$ as a function of parameters $\Theta_t$ over continuous time $t$, then, for SGD, we have the typical equation
\[\ff'_t = -\eta \left\la \pdf {\ff_t} {\Theta_t}, \pdf {\loss_t} {\Theta_t} \right\ra = -\eta \left\la \pdf {\ff_t} {\Theta_t}, \pdf {\loss_t} {\ff_t} \pdf {\ff_t} {\Theta_t} \right\ra\]
where $\loss_t = \loss(\ff_t)$ and $\loss$ is the loss function.
For a general (memoryless stationary) update function $Q$, this just becomes%
\footnote{Technically, we should include terms involving $d_l$ from \cref{defn:abcd} in \cref{eqn:NTQ_intuition}, but for simplicity, let's just assume that $Q$ is temporarily redefined to have already included them}
\begin{equation}
  \ff'_t = -\eta \left\la \pdf {\ff_t} {\Theta_t}, Q\left(\pdf {\loss_t} {\ff_t} \pdf {\ff_t} {\Theta_t} \right) \right\ra,
  \label{eqn:NTQ_intuition}
\end{equation}
with $Q$ applied entrywise.
In both cases, when width is large in NTP (\cref{defn:NTP}), the weights essentially move so little that $\pdf {\ff_t} {\Theta_t}$ is invariant to $t$ (in so far as this contraction in \cref{eqn:NTQ_intuition} is concerned).
What changes from $Q=\id$ (SGD case) to the general case is that
$\ff'_t$ is no longer linear in the error signal $\eps_t(\ff_t) = \pdf {\loss_t} {\ff_t}$; instead, it is the result of a \emph{nonlinear} operator $\calK_Q$ mapping the error signal to the function update:
\begin{equation*}
  \ff'_t = -\eta \calK_Q\left(\pdf {\loss_t} {\ff_t}\right) = -\eta \calK_Q(\eps_t(\ff_t)).
\end{equation*}
When $Q=\id$, $\calK_Q$ reduces to the linear operator represented by the NTK.
But note that $\ff'_t$ is still linear in the learning rate $\eta$ for general $Q$.

Now we turn to make this intuition rigorous.

\subsection{Infinite-Width Limit}
We can associate random vectors $\ket{\xx_{t}^{l}},\ket{\hh_{t}^{l}}$
to $\xx_{t}^{l}$ and $\hh_{t}^{l}$ as discussed in \cref{sec:TPansatz} (and similar
to the SGD case studied in \cite{yang4}). But as in the case with SGD,
these random vectors will turn out to be independent of $t$ (because
of the lack of feature learning). Therefore, we will drop the subscript
$t$ in the notation below. The construction of $\ket{\xx^{l}},\ket{\hh^{l}}$
follow from the general rules of the \nexort{} program we develop
below,%
\footnote{Actually they are the same random vectors as constructed in \cite{yang2} since this can be done in \netsort{}.}
but here they can be stated simply as follows:
\begin{defn}
For each $l=1,\ldots,L$, the ket $\ket{\hh^{l}}\in\R^{\NN}$ is constructed
as a mean-zero Gaussian vector with covariance matrix $\braket{\xx^{l-1}}{\xx^{l-1}}$
when $l>1$ or $\xxi^{\trsp}\xxi$ when $l=1$. Simultaneously, $\xx^{l}\defeq\phi(\ket{\hh^{l}})\in\R^{\NN}$
for $l=1,\ldots,L$.\footnote{$\{\hh^{1},\xx^{1}\},\ldots,\{\hh^{L},\xx^{L}\}$ are mutually independent
by construction as well, but we will not need this fact.}
\end{defn}
Likewise, we can associate random vectors $\ket{d\xx^{l}},\ket{d\hh^{l}}$
to gradients $d\xx^{l},d\hh^{l}$ (again, suppressing subscript $t$ because
it will turn out the kets are independent of $t$). 
\begin{defn}
For each $l=L,\ldots,1$, the ket $\ket{d\xx^{l}}\in\R^{\NN}$ is independent
from $\{\ket{\xx^{l}},\ket{\hh^{l}}\}_{l=1}^{L}$ and is a mean-zero
Gaussian vector with covariance matrix $\braket{d\hh^{l+1}}{d\hh^{l+1}}$
when $l<L$ or the all-1s matrix when $l=L$. Simultaneously, $\ket{d\hh^{l}}\defeq\phi'(\ket{\hh^{l}})\odot\ket{d\xx^{l}}\in\R^{\NN}$
for all $l$.
\end{defn}

\paragraph{Memoryless Stationary Case}
Finally, having constructed these kets, we can define the generalization
of NTK we need.
To make the formula short, we employ the following convention: we write $\ket{\xx^0} \defeq \xxi \in \R^{d \times \NN}$ and $\ket{d\hh^{L+1}} \defeq \boldsymbol 1 \in \R^{1 \times \NN}$ is the row vector of 1s.
Then we set $\bra{\xx^0} = \ket{\xx^0}^\trsp$ and $\braket{\xx^0}{\xx^0} = \xxi^\trsp \xxi \in \R^{\NN \times \NN}$;%
\footnote{
  One perhaps would be inclined to rescale $\braket{\xx^0}{\xx^0} = \xxi^\trsp \xxi / d$, but in our context, $d$ is a constant while we only focus on scaling with $n$.
  So we choose to be slightly more brief by omitting this scaling with $d$.
} likewise for $\braket{d\hh^{L+1}}{d\hh^{L+1}}$.
At this point,
the reader may find it helpful to review \cref{sec:TPansatz}, especially on \emph{Bar Notation}.
\begin{defn}[Neural Tangent Operator, Memoryless Stationary Case]\label{defn:TangentOperator}
For any function $Q: \R \to \R$, 
we define
the \emph{neural tangent $Q$-operator} $\calK_{Q}:\R^{\NN}\to\R^{\NN}$
by the following: for any $\cchi\in\R^{\NN}$,
\begin{align}
\calK_{Q}(\cchi) & \defeq
  \Diag\sum_{l=1}^{L+1}\la{d\hh^{l}}\Qketdbra{d\hh^{l}}{\cchi}{\xx^{l-1}}{\xx^{l-1}}\ra
  \label{eqn:KQ}
\end{align}
where the ``bar'' notation abbreviates application of $Q$ as in \cref{eqn:barnotation}.%
\end{defn}
Let's digest the notation a bit:
In \cref{eqn:KQ}, 
1) the subscript $\cchi \in \R^\NN$ represents multiplication by $\Diag(\cchi)$,
2) $\Qketdbra{d\hh^{l}}{\cchi}{\xx^{l-1}}$ is a random scalar variable, and
3) $\la{d\hh^{l}}\Qketdbra{d\hh^{l}}{\cchi}{\xx^{l-1}}{\xx^{l-1}}\ra$ is a deterministic $\NN\times\NN$ matrix.
The $\Diag$ in \cref{eqn:KQ} takes its diagonal.
Thus,
$\calK_{Q}(\cchi)$ has entries
\begin{align*}
\calK_{Q}(\cchi)^{a} & =
  \sum_{l=1}^{L+1}\la{dh^{la}}\Qketdbra{d\hh^{l}}{\cchi}{\xx^{l-1}} {x^{l-1,a}}\ra
\end{align*}
for each $a\in[\NN]$.

\begin{exmp}[SGD Example]
As an example, when $Q$ is identity, the ``bar'' can be removed, and $\calK_{Q}$ reduces to
the linear operator represented by the NTK:%
\footnote{Again, $\cchi$ is a row vector. In prior works, it's usually treated as a column vector in which case one would write $K\cchi$ instead.}
\begin{align*}
\calK_{\id}(\cchi) & =\cchi K,\quad\text{where \ensuremath{K\in\R^{\NN\times\NN}} is the NTK}\\
K & \defeq
 \sum_{l=1}^{L+1}\braket{d\hh^{l}}{d\hh^{l}}\odot\braket{\xx^{l-1}}{\xx^{l-1}}.
\end{align*}
\end{exmp}

\begin{exmp}[SignSGD Example]
As another example, consider $Q = \sgn$ (\cref{eqn:SignSGD} with $\epsilon=0$).
If the batch size is 1, i.e., $\cchi$ is nonzero on exactly one input, say $\xi^b, b \in [\NN]$, then \cref{eqn:KQ} is linear in $\sgn(\cchi)$ because $\sgn(xy)=\sgn(x)\sgn(y)$.
Thus,
\begin{align*}
\calK_{Q}(\cchi) & =
  \sgn(\cchi) K_{\sgn},\quad \text{where \ensuremath{K_{\sgn}\in\R^{\NN\times\NN}} is}\\
K_{\sgn} &= \sum_{l=1}^{L+1}\la{d\hh^{l}}\overline{\ket{d\hh^l}} \odot
\la{\xx^{l-1}}\overline{\ket{\xx^{l-1}}} 
\end{align*}
for each $a\in[\NN]$.
This expression was concurrently derived in \citep{malladi_kernel-based_2022}.
However, when batch size is larger than 1, $\calK_Q(\cchi)$ is no longer linear in $\sgn(\cchi)$ generally, and there is no simplification like this.
In particular, in the continuous time gradient flow setting, SignSGD with a large batch size is \emph{not} equivalent to that with batch size 1 but very small learning rate, in contrast to SGD.
\end{exmp}

We can now state the generalized Neural Tangent limit for memoryless stationary updates.
In particular, this covers SGD and SignSGD (\cref{eqn:SignSGD}).
\begin{thm}[Neural Tangent Limit, Memoryless Stationary Case]\label{thm:memorylessNTK}
Consider any training routine (\cref{assm:trainingroutine}) with memoryless stationary update function $Q$ (\cref{defn:memoryless_stationary}).
Adopt \cref{assm:MLPsmooth}.
Then
\[\ff_0 \distto \Gaus(0,\braket{\xx^L}{\xx^L}) \] 
and for every $t$, $\ff_{t}$ converges
almost surely to random function $\mathring{\ff_{t}}\in\R^{\NN}$ satisfying
\begin{align}
   \mathring{\ff}_{t+1} - \mathring{\ff}_{t}&=-\eta\calK_{Q}(\eps_{t}(\mathring{\ff}_{t})),\quad \text{for all $t$}.
   \label{eqn:NTK_update_memoryless_stationary}
\end{align}
\end{thm}
The proof can be found in \cref{sec:NTKProof}.
Note, as in the SGD case, $\mathring f_t$ is deterministic conditioned on $\mathring f_0$.
We remind the reader that \cref{eqn:NTK_update_memoryless_stationary} simultaneously covers full batch, mini-batch, train-test split, and other schemes by changing $\eps_t$, as discussed under \cref{assm:trainingroutine}.
This will be the same for all theorems in this work.

\paragraph{Memoryful Nonstationary Case}

The memoryless stationary condition allowed a clean mathematical formulation of the NT limit.
But we can remove it easily at the cost of some more notation.

\begin{defn}[Neural Tangent Operator, Memoryful Nonstationary Case]\label{defn:TangentOperatorMemoryful}
  For memoryless but nonstationary update functions $\QQ_t = \{Q^l_t: \R \to \R\}_l$, 
  we define $\calK_{\QQ_t}: \R^{\NN} \to \R^\NN$ with the same equation as  \cref{eqn:KQ}, except the bar notation abbreviates $Q^l_t$ where $l$ is the same as in the $d\hh^l$ under the bar.

  For general update functions $\QQ_t = \{Q^l_t: \R^{t+1} \to \R\}_l$, 
  we define $\calK_{\QQ_t}: \R^{(t+1) \times \NN} \to \R^\NN$,
  \begin{equation}
    \calK_{\QQ_t}(\cchi_0, \ldots, \cchi_t) \defeq 
    \Diag\sum_{l=1}^{L+1}\la{d\hh^{l}}\Qketdbra{d\hh^{l}}{\cchi_{\le t}}{\xx^{l-1}}{\xx^{l-1}}\ra
  \end{equation}
  where $\Qketdbra{d\hh^{l}}{\cchi_{\le t}}{\xx^{l-1}}$ is shorthand for $Q^l_t\left(\ketdbra{d\hh^{l}}{\cchi_0}{\xx^{l-1}}, \ldots, \ketdbra{d\hh^{l}}{\cchi_t}{\xx^{l-1}}\right)$.
\end{defn}

With this in mind, the following theorem yields the NT limit of Adam (\cref{eqn:Adam}) as a corollary.
\begin{thm}[Neural Tangent Limit, Memoryful Nonstationary Case]\label{thm:NT_MLP_memoryful}
  If the update functions $\QQ$ are memoryless but not necessarily stationary, then \cref{thm:memorylessNTK} holds with \cref{eqn:NTK_update_memoryless_stationary} replaced by
  \begin{equation}
    \mathring{\ff}_{t+1} - \mathring{\ff}_{t} =-\eta\calK_{\QQ_t}(\eps_{t}(\mathring{\ff}_{t})),\quad \text{for all $t$}.
  \end{equation}

  For general $\QQ$, not necessarily memoryless, \cref{thm:memorylessNTK} holds with \cref{eqn:NTK_update_memoryless_stationary} replaced by
  \begin{equation}
    \mathring{\ff}_{t+1} - \mathring{\ff}_{t} =-\eta\calK_{\QQ_t}(\eps_0(\mathring \ff_0), \ldots, \eps_{t}(\mathring{\ff}_{t})),\quad \text{for all $t$}.
  \end{equation}
\end{thm}

The proof is a straightforward adaptation of the proof of \cref{thm:memorylessNTK} in \cref{sec:NTKProof}.

\subsection{Lack of Feature Learning}

Just as for the SGD case, the neural tangent limit cannot learn features, for example, in the sense that the feature kernel (the Gram matrix of the input representations) does not evolve during training (\cref{thm:abcdclassification}).
Likewise, pretraining is futile in this limit as finetuning it would be no different than finetuning a randomly initialized network (\cref{rem:pretraining_operator_regime}).
Nevertheless, the NT limit is ``maximal'' among all nondegenerate limits without feature learning in that all other limits are just given by a neural tangent operator that involves a subsum of \cref{eqn:KQ} (\cref{rem:maximality}).

\section{Maximal Update}
\label{sec:maximalupdate}

As for NTP, the ``classical'' maximal update (abc-)parametrization can be generalized easily to an abcd-parametrization using the intuition that the input to any nonlinear update function should be $\Theta(1)$ (c.f.\ \cref{defn:faithful,{lemma:faithfulconditionsinit}}).
After defining this generalization next, we study its limit for shallow MLP.
The deep case will have to wait until we develop the new Tensor Program theory in the next section.
\begin{defn}\label{defn:mup}
\emph{The maximal update abcd-parametrization ($\mu$P) }is defined
(modulo \cref{stmt:SGDsymmetry}) by
\begin{center}
  \begin{tabular}{cccc}
    \toprule
    $l$ & $[2,L]$ & $1$  & $L+1$\tabularnewline
    \midrule
    $a_l$ & 0 & 0  & 1\tabularnewline
    $b_l$ & $\nicefrac{1}{2}$  & $0$ & $0$\tabularnewline
    $c_l$ & 1 & 0  & 0\tabularnewline
    $d_l$ & 1 & 1  & 1\tabularnewline
    \bottomrule
  \end{tabular}
\end{center}

\end{defn}

\paragraph{Recovering ``Classical'' $\mu$P}
If we assume that the Adam update function (\cref{eqn:Adam}) is perfectly scale-invariant, then the $d_l$ row can be dropped, yielding \cite[Table 8]{yang5} regarding Adam LR scaling.

To recover the abc version of $\mu$P in \citep[Table 1]{yang4} for SGD, just apply \cref{stmt:SGDsymmetry} to the $l=1,L+1$ columns with $\theta = \nicefrac {-1} 2$ and then apply \cref{eqn:abc_reduction} (i.e., subtracting the last 2 rows).

\subsection{Shallow Infinite-Width Limit}
\label{sec:shallowmulimit}

We focus on the shallow case first because its $\mu$-limit
is fairly easy to describe. We shall cover the general case after
we describe the outer product tensor program \nexort{} in \cref{sec:nexort}. 
Adopt the following leaner
notation:
\begin{equation}
f(\xi)=(\nicefrac{1}{n}) v^{\trsp}x(\xi),\quad x(\xi)=\phi(h(\xi)),\quad h(\xi)=u\xi,\label{eqn:UV1LP}
\end{equation}
for trainable parameters $u \in \R^{n \times d}, v\in\R^{n\times1}$ with initialization
$u,v\sim\Gaus(0,I)$.\footnote{Again, more generally, we can insert constants in this parametrization,
like $h(\xi)=\frac{1}{\sqrt{d}}u\xi$ or $u_\alpha \sim \Gaus(0, 1/d)$, but we omit them here for simplicity.}

The following general theorem covers the $\mu$-limit for SGD ($Q=\id$) and SignSGD (\cref{eqn:SignSGD}).
\begin{thm}[Shallow $\mu$-Limit, Memoryless Stationary Case]\label{thm:muLimit_1LP}
  Consider any training routine (\cref{assm:trainingroutine}) with memoryless stationary update function $Q$ (\cref{defn:memoryless_stationary}).
Adopt \cref{assm:MLPsmooth}.
As $n\to\infty$, $\ff_{t}$ for the network in \cref{eqn:UV1LP} converges almost surely to some $\mathring{\ff}_{t}$
for every $t$, which is recursively defined from $t=0$ by the following
dynamics:
\begin{enumerate}
  \item (Forward and Backward Propagation)
  \begin{align*}
    \mathring{\cchi}_{t} & =\eps_{t}(\mathring{\ff}_{t})\in\R^{\NN},&
    \mathring{\ff}_{t} & =\braket{v_{t}}{\xx_{t}}\in\R^{\NN},&
    \ket{\xx_{t}} & =\phi(\ket{\hh_{t}})\in\R^{\NN},&
    \ket{\hh_{t}} & =\ket{u_{t}}\xxi\in\R^{\NN}
  \end{align*}
  \item (Parameter Updates)
  \begin{align}
    \ket{u_{t+1}} & =\ket{u_{t}}-\eta \overline{\ket{v_{t}}\phi'(\ket{\hh_{t}})_{\mathring{\cchi}_{t}}\xxi^{\trsp}}\in\R^d \label{eqn:1LP_u_update}
    \\
    \ket{v_{t+1}} & =\ket{v_{t}}-\eta \overline{\ket{\xx_{t}}\cdot\mathring{\cchi}_{t}}\in\R \label{eqn:1LP_v_update}
  \end{align}
  where the bar notation (\cref{eqn:barnotation}) abbreviates application of $Q$ and $\cdot$ denotes dot product.
  \item (Initialization)
  \begin{align*}
    (\ket{v_{0}},\ket{u_{0}}) & = \Gaus(0,I)  
  \end{align*}
\end{enumerate}
\end{thm}
Again, one can note that if $Q$ is identity, then we recover the SGD
equations from \cite[Theorem 6.1]{yang4}.

The proof of \cref{thm:muLimit_1LP} is given in \cref{sec:muLimitProof}.
This can be adapted straightforwardly to cover the nonstationary or memoryful cases:
\begin{thm}[Shallow $\mu$-Limit, Memoryful Nonstationary Case]\label{thm:muLimit_1LP_general}
  If $\QQ$ is memoryless but not stationary, then \cref{thm:muLimit_1LP} holds if the bar in \cref{eqn:1LP_v_update} (resp.\ \cref{eqn:1LP_u_update}) is interpreted as $Q^2_t$ (resp.\ $Q^1_t$).

  If $\QQ$ is not memoryless, then \cref{thm:muLimit_1LP} holds if \cref{eqn:1LP_v_update,eqn:1LP_u_update} are replaced with 
  \begin{align*}
    \ket{u_{t+1}} & =\ket{u_{t}}-\eta \overline{\ket{v_{\le t}}\phi'(\ket{\hh_{ \le t}})_{\mathring{\cchi}_{\le t}}\xxi^{\trsp}}\in\R
    \\
    \ket{v_{t+1}} & =\ket{v_{t}}-\eta \overline{\ket{\xx_{\le t}}\cdot\mathring{\cchi}_{\le t}}\in\R
  \end{align*}
  where the bar notations abbreviate
  \begin{align*}
    \overline{\ket{v_{\le t}}\phi'(\ket{\hh_{ \le t}})_{\mathring{\cchi}_{\le t}}\xxi^{\trsp}}&=Q^1_t \left(\ket{v_{0}}\phi'(\ket{\hh_{0}})_{\mathring{\cchi}_{0}}\xxi^{\trsp}, \ldots, {\ket{v_{t}}\phi'(\ket{\hh_{t}})_{\mathring{\cchi}_{t}}\xxi^{\trsp}}\right)
    \\
    \overline{\ket{\xx_{\le t}}\cdot\mathring{\cchi}_{\le t}} &= Q^2_t\left({\ket{\xx_{0}}\cdot\mathring{\cchi}_{0}}, \ldots, {\ket{\xx_{t}}\cdot\mathring{\cchi}_{t}}\right)
  \end{align*}
\end{thm}

\begin{rem}\label{rem:generalinputdim}

\cref{thm:muLimit_1LP} and \ref{thm:muLimit_1LP_general} also hold if the output dimension is greater than 1, in which case
the equations should be interpreted slightly differently;
see \cref{eqn:ket_outer_product_general}.
\end{rem}

\section{\nexort{}: Tensor Program with Nonlinear Outer Products}
\label{sec:nexort}

As mentioned above, the gradient processing done by entrywise optimizers in general cannot be expressed previously by even the most expressive Tensor Program language.
Here we fix this issue by adding ``nonlinear outer products'' to Tensor Programs.
Like in previous works, we algebraically construct such programs' limit objects (``kets'', in our new notation) and link them to the analytic properties of vectors in the corresponding programs via a \emph{Master Theorem}.
Unlike previous works, we also construct the limits of matrices, which are operators on kets.
This is purely a conceptual change, but this perspective helps express the deep $\mu$-limit much more efficiently than possible before.

\subsection{The \nexort{} Language}

\begin{defn}\label{defn:nexort}
  A \nexort{} program
  generates a sequence $\xx$ of $\R^{n}$-vectors and a sequence $\cc$ of $\R$-scalars inductively defined via one of the following ways from an initial set $\cc^0 \sbe \cc$ of random scalars, an initial set $\xx^0\sbe \xx$ of random $\R^{n}$ vectors, and an initial
   set $\mathcal{W}$ of random $\R^{n\times n}$ matrices.%
  \footnote{which will be sampled with iid Gaussian entries in \cref{setup:nexort} or general non-Gaussian entries in \cref{setup:nexort_nongaussian}.}
  We will think of $\cc$ as a vector and $\xx$ as a matrix with the $\R^n$ vectors as columns; then $\cc^0$ is just a subvector of $\cc$ and $\xx^0$ is a submatrix of $\xx$.
  At each step of the program, one can
  \begin{description}
  \item [\texttt{Avg\label{instr:avg}}] choose a vector $x \in \xx$ (think of $x$ as a column in $\xx \in \R^{n \times |\xx|}$) and append to $\cc$ a scalar%
  \footnote{We replaced the \texttt{Moment} instruction of \netsortp{} with \texttt{Avg} here, but there is no loss of expressivity since \texttt{Moment} is just a composition of \texttt{Avg} and \texttt{Nonlin}$^+$}
  \begin{equation}\frac 1 n \sum_{\alpha=1}^n x_\alpha \in \R \label{eqn:avg}\end{equation}
  \item [\texttt{MatMul\label{instr:matmul}}] choose a matrix $W\in \calW$ and vector $x\in\xx$, and append to $\xx$ the vector 
  \[Wx\in\R^{n} \quad \text{or} \quad W^{\trsp}x\in\R^{n}\]
  \item [\texttt{OuterNonlin\label{instr:outernonlin}}] choose integer $r \ge 0$ and function $\psi: \R^{|\xx|(r+1)+l} \to \R$; append to $\xx$ the vector%
  \footnote{We can equivalently allow $\xx$ in each slot below to be different multi-vectors (which are subsets of $\xx$), since $\psi$ here can just be chosen to ignore the irrelevant subsets of $\xx$. For notational simplicity, we don't do this.}
  \begin{equation}
  y \in \R^n, \quad y_{\alpha} = \frac{1}{n^{r}}\sum_{\beta_{1},...,\beta_{r}=1}^{n}\psi(\xx_{\alpha};\xx_{\beta_{1}};...;\xx_{\beta_{r}};\cc)
  \label{eqn:outernonlin}
  \end{equation}
  where $\xx_\gamma$ is the $\gamma$th row in $\xx$ as a matrix and $|\xx|$ is the number of vectors in $\xx$.
  We call $r+1$ the \emph{order} of $\psi$ in this context.
  \end{description}
\end{defn}

Note that while \cref{defn:nexort} doesn't directly give an instruction to transform scalars into scalars (in contrast to \refMatMul{} and \refOuterNonlin{} that transforms vectors into vectors), this can be done by combining instructions.
\begin{lemma}\label{lemma:transform_scalars}
  If $\psi: \R^{|\cc|} \to \R$ and $\cc$ are the scalars in a \nexort{} program, then $\psi(\cc) \in \R$ can be introduced as a new scalar in the program.
\end{lemma}
\begin{proof}
  Think of $\psi$ as a function that ignores vector arguments and depend only on scalars, and use it in \refOuterNonlin{} to create the vector whose entries are identically equal to $\psi(\cc)$. Applying \refAvg{} to this vector gives the desired result.
\end{proof}

\subsection{Setups}
We are interested in the behavior of \nexort{} programs in two typical settings:
\begin{setup}[Gaussian]\label{setup:nexort}
  Assume%
  \footnote{Compared to \cite{yang3}, we have WLOG simplified the setup by assuming 1) $\sigma_W =1$ for every $W$, 2) $Z^{\xx^0} = \ket{\xx^0}= \Gaus(0, I)$, and 3) $\mathring \theta = 0$ for every $\theta \in \cc^0$. This is WLOG because $\sigma_W$, $\mathring \theta$, and the mean and covariance of $\ket{\xx^0}$ can all be absorbed into \refOuterNonlin{} via the appropriate linear functions.}
  \begin{enumerate}
    \item Every entry of every $W\in\mathcal{W}$ is sampled iid from $\Gaus(0,1/n)$. 
    \item Every entry of every initial vector $x \in \xx^0$ is sampled iid from $\Gaus(0, 1)$.
    \item The initial scalars $\cc^0$ converge almost surely to 0. %
    \item All functions $\psi$ used in \refOuterNonlin{} are pseudo-Lipschitz.

  \end{enumerate}
\end{setup}

\begin{setup}[Non-Gaussian]\label{setup:nexort_nongaussian}
  Assume the same as \cref{setup:nexort} but replace 1) and 4) with
  \begin{enumerate}
    \item[1*.] there exists a sequence $\nu_3, \nu_4, \ldots > 0$ such that all matrices have independent entries%
    \footnote{For all of our results, it does not matter how the matrices for different $n$ are correlated, e.g., whether they are independent or the matrices are all upper left submatrix of fixed infinite iid matrix. This is because our proof only depends on how moments of vectors behave with $n$, which does not care about such inter-$n$ correlations.
    }
    drawn from distributions with zero mean, variance $n^{-1}$, and all higher $k$th moment bounded by $\nu_k  n^{-k/2}$; and%
    \footnote{Initial vectors are still sampled from $\Gaus(0, 1)$, as in \cite{tp3b}.}
    \item[4*.] All functions $\psi$ used in \refOuterNonlin{} are polynomially smooth.%
    \footnote{Recall from \cite{tp3b} that $f:\RR^k \to \RR$ is \emph{polynomially smooth} if it is $C^\infty$ and its partial derivatives of any order are polynomially bounded. See \cite{tp3b}.}
  \end{enumerate}
  We further require initial scalars $\cc^0$ to have moments of all orders bounded in $n$.%
  \footnote{But the moments do not need to be bounded as a function of the order.}
\end{setup}
While \cref{setup:nexort_nongaussian} allows more general distributions for matrix entries, its nonlinearities need to have more smoothness than \cref{setup:nexort}.
See \cite{tp3b} for more discussions on these setups.

\subsection{Limit Objects}
As before, when the width $n$ of the program goes to infinity, one can infer how the program behaves via a calculus of random variables.
We define them below via the new ket notation instead of the earlier $Z$ notation.
\begin{defn}[{Ket Construction}]\label{defn:ket}
We recursively define the random variable $\ket x$ (called a \emph{ket}) for each vector $x$ and deterministic number $\mathring{\theta}$ for each scalar $\theta$ in the program.
For a vector $Wx$ produced by \refMatMul{}, we also define random variables $\hatket{Wx}$ and $\dotket{Wx}$ (called \emph{hat-ket} and \emph{dot-ket} respectively) such that $\ket{Wx} = \hatket{Wx} + \dotket{Wx}$.
Their recursive definitions are given below.
\begin{description}
\item[\texttt{Init}]
$\ket{\xx^0} \defeq \Gaus(0, I) \in \R^{|\xx^0|}$ and $\mathring \cc \defeq 0 \in \R^{|\cc^0|}$.%
\footnote{These \texttt{Init} rules depend on the fact that, in \cref{setup:nexort} and \cref{setup:nexort_nongaussian}, initial vectors are sampled with variance 1 and initial scalars converge to 0.}
\item[\texttt{Avg}] If $\theta$ is generated by \refAvg{} as in \cref{eqn:avg}, then $\mathring{\theta}\defeq\EV\ket x=\braket 1x$.
\item[\texttt{OuterNonlin}] If $x$ is generated by \refOuterNonlin{} as in \cref{eqn:outernonlin}, then 
\begin{center}
$\ket x\defeq f(\ket{\xx})\quad$
where $\quad f: \R^{|\xx|} \to \R, f(\ybf)\defeq\EV\psi(\ybf;\ket{\xx}^{\bx{1}};\cdots;\ket \xx^{\bx{r}};\mathring{\cc})$.%
\footnote{
recall (\cref{sec:notation_iid}) $\ket{\xx}^{\bx{1}},\cdots,\ket \xx^{\bx{r}}$ are iid
  copies of $\ket{\xx}$, which is the tuple $(\ket{x^{1}},\ket{x^{2}},\ldots)$ 
  }
\end{center}
\item[\texttt{Hat}]
All hat-kets
are jointly Gaussian with zero-mean and covariance%
\footnote{In \cref{eqn:hatketcovar}, $\ind(W=U)$ is the deterministic number that is 1 iff $W$ and $U$ are the same matrix (as symbols in the program) and 0 otherwise. This should \emph{not} be interpreted as a random variable that is 1 precisely when $W$ and $U$ take the same values.}
\begin{align}
  \Cov(\hatket{Wx}, \hatket{Uy}) &= \ind(W = U)\braket x y \label{eqn:hatketcovar}
\end{align}

\item[\texttt{Dot}] Every dot-ket is a linear combination of previous kets, expressed by the following equation
\begin{equation}
  \dotket{Wx}\defeq \ket{\xx}\braket{\nabla_{W^{\trsp}\xx}}{x} \label{eqn:dotket}
\end{equation}
\end{description}
\end{defn}
\cref{eqn:dotket} is the same equation \citep[Zdot]{yang4} 
but formulated much more succinctly in the bra-ket notation:
\begin{align*}
  \text{\citep[Zdot]{yang4}},\quad
  \Zdot^{Wx}
  &= \sum_{y \in \xx } Z^{y}\EV \f{\partial Z^x} {\partial \hat Z^{W^\trsp y}},\\
  \text{or, in scalar ket notation,}\quad\dotket{Wx} &= \sum_{y \in \xx} \ket{y}\EV \frac{\partial \ket x}{\partial \hatket{W^\trsp y}}.
\end{align*}
To arrive at the presentation \cref{eqn:dotket}, we think of $\EV \frac{\partial \ket x}{\partial \hatket{W^\trsp y}}$ as $\braket{\partial_{\hatket{W^\trsp y}}}{x}$ for a ``generalized bra'' $\bra{\partial_{\hatket{W^\trsp y}}}$,
and group together 1) all kets $\ket y$ as $\ket \xx$ and 2) all the bras $\bra{\partial_{\hatket{W^\trsp y}}}$ (over all $y \in \xx$) as a multidimensional bra written simply as $\bra{\nabla_{W^\trsp \xx}}$.
Then the sum $\sum_y$ can be straightforwardly rewritten as the ``ket outer product'' described in \cref{sec:ket_outer}.

\newcommand\lhatbra[1]{\mkern+3mu\check{\mkern-3mu\langle}#1 \ob}
\newcommand\lhatbraket[2]{\mkern+3mu\check{\mkern-3mu\langle}#1 \ob #2 \rangle}

\begin{rem}[Alternative Notation]
  The bra $\bra{\nabla_{W^\trsp \xx}}$ is really the ``{dual}'' of $\hatbra{W^\trsp \xx}$ in the sense that%
  \footnote{But note this identity only holds when $\xx$ contains all vectors $z$ where $\ket y$ depends on $\hatket{W^\trsp z}$.}
  \begin{align*}
    \text{for any ket $\ket{y}$,}\quad \braket \xx \xx \braket{\nabla_{W^\trsp \xx}} y &= \hatbraket{W^\trsp \xx}y,\\
    \text{or, in short,}\quad \braket \xx \xx \bra{\nabla_{W^\trsp \xx}} &= \hatbra{{W^\trsp \xx}}
  \end{align*}
  This follows from Stein's lemma.%
  \footnote{
    In the language of Riemannian geometry, if we think of $\braket{\xx} \xx$ as a metric tensor in a Riemannian manifold, then $\bra{\nabla_{W^\trsp \xx}}$ is obtained from $\hatbra{W^\trsp \xx}$ by ``lowering the index.''
  }
  Thus, a more appropriate notation for $\bra{\nabla_{W^\trsp \xx}}$ is perhaps
  \[\lhatbra{ W^\trsp \xx} \defeq \bra{\nabla_{W^\trsp \xx}},\]
  so that
  \begin{align*}
    \braket \xx \xx \lhatbra{ W^\trsp \xx} &= \hatbra{{W^\trsp \xx}}
  \end{align*}
  and \cref{eqn:dotket} reads
  \begin{align*}
    \dotket{Wx}\defeq \ket{\xx}\lhatbraket{{W^{\trsp}\xx}}{x} = \ket{\xx} \braket \xx \xx ^{+} \hatbraket{{W^{\trsp}\xx}}{x}.
  \end{align*}
  However, this ``duality'' is not essential for understanding this paper, so we keep the more intuitive notation $\bra{\nabla_{W^\trsp \xx}}$ instead.
\end{rem}

\begin{defn}\label{defn:oplim}
  Let $W$ be an initial matrix in a \nexort{} program.
  We define $\oplim{W}$ to be the linear operator on kets%
  \footnote{\label{footnote:hilbertspace}
    To be rigorous, we need to specify the ``Hilbert space'' of kets.
    This is somewhat pedantic and not crucial to the key points of this paper, but the Hilbert space can be constructed as follows:
    Let $\sigma(\pi)$ be the $\sigma$-algebra generated by the kets of the program $\pi$.
    Let $\Sigma(\pi) \defeq \bigcup_{\pi' \supseteq \pi} \sigma(\pi)$ be the union (more precisely, the direct limit) of $\sigma(\pi')$ over all programs $\pi'$ extending $\pi$.
    Then the Hilbert space in question is the $L^2$ space of random variables over the $\Sigma$ of our program.
    }
  that acts by 
  \[\opket{W}x \defeq \ket{Wx} = \hatket{Wx} + \dotket{Wx}.\]
  Any linear operator that is equal to $\oplim{W}$ for some initial matrix $W$ is called an \emph{initial operator}.
  A set of initial operators is called \emph{independent} if their corresponding initial matrices are distinct.%
  \footnote{i.e., independently sampled in \cref{setup:nexort} or \cref{setup:nexort_nongaussian}.}
\end{defn}
We have already seen an example of a linear operator on kets: expressions like $\Qketdbra{y}{\chi}z$.
\cref{defn:oplim} puts $\oplim{W}$ in the same space as $\Qketdbra{y}{\chi}z$.
This allows us to add them in the sequel, which simplifies the presentation of the $\mu$-limit.

We can immediately see a few properties of $\oplim{W}$ by considering the counterpart when $n$ is finite.
\begin{prop}
  For any initial matrix $W$, the operator $\oplim{W}$ is bounded.%
  \footnote{i.e., there exists real number $L > 0$ such that for any ket $\ket x$, $\braket{Wx}{Wx} \le L \braket x x$.}
\end{prop}
\begin{proof}
  This follows from the classical operator norm tail bounds on iid random matrices (see, e.g., \citep{tao}), which passes to the limit via \cref{thm:MasterTheorem} below.
\end{proof}
\begin{prop}
  For any initial matrix,
  the operator $\oplim{W^\trsp}$ is the adjoint%
  \footnote{``adjoint'' in the sense of Hilbert space operators; see \cref{footnote:hilbertspace}.}
  of the operator $\oplim{W}$.
\end{prop}
\begin{proof}
  By \cref{thm:MasterTheorem} below, $\frac 1 n \la x, W y \ra \asto \braket{x}{Wy}$ and $\frac 1 n \la W^\trsp x, y\ra \asto \braket{W^\trsp x}y$.
  Since $\la x, W y \ra = \la W^\trsp x, y\ra$ for any $n$, we have
  \[\braket{x}{Wy} = \braket{W^\trsp x}y,\]
  i.e., $\oplim{W^\trsp}$ is the adjoint of $\oplim{W}$.
\end{proof}

\subsection{The Master Theorem}
Our key foundational result is that the Master Theorem of earlier Tensor Programs generalizes to \nexort{} programs.
This underlies all of our theorems about adaptive optimization.
\begin{thm}[\nexort{} Master Theorem]\label{thm:MasterTheorem}
  Consider a \nexort{} program with (Gaussian) \cref{setup:nexort} or (non-Gaussian) \cref{setup:nexort_nongaussian}. 
  Then, as $n\to \infty$, its scalars $\cc$ satisfy
  \begin{equation*}
    \cc \asto \mathring \cc.
  \end{equation*}
  In \cref{setup:nexort_nongaussian}, this convergence also happens in $L^p$ for every $p \in[1,\infty)$.
  In either setup, if the initial scalars are all $\tilde O(n^{-1/2})$ (\cref{defn:tildeO}), then%
  \begin{equation*}
    \cc - \mathring \cc = \tilde O(n^{-1/2})
  \end{equation*}
  as well.
\end{thm}

\begin{rem}
  If the initial scalars are $\tilde O(n^{-1/2})$, then we can \emph{almost} say that the distribution of $\cc$ converge to the delta distribution on $\mathring \cc$ in Wasserstein distance, at a rate of $\tilde O(n^{-1/2})$.
  See \cref{sec:Wasserstein_dequiv}.
  An analogous $\tilde O(n^{-1/2})$-convergence result also holds for vectors converging to their kets (in a suitable sense).
  See \cref{lem:otimes_tp_comparison}.
\end{rem}

\section{Maximal Update for Deep MLP}
\label{sec:deepmaximalupdate}

In this section we describe the infinite-width limit of $\mu$P for arbitrarily deep MLP.
The main difference here compared to the shallow case (\cref{sec:shallowmulimit}) is the presence of $n \times n$ iid matrices in the middle of the network, which behaves like initial operators (\cref{defn:oplim}) in the limit.

\begin{thm}[Deep $\mu$-Limit, Memoryless Stationary Case]\label{thm:mulimit_MLP}
  Consider any training routine (\cref{assm:trainingroutine}) with memoryless stationary update function $Q$ (\cref{defn:memoryless_stationary}).
  Adopt \cref{assm:MLPsmooth}.
  As $n\to\infty$, $\ff_{t}$ converges almost surely to some $\mathring{\ff}_{t}$
  for every $t$, which is recursively defined from $t=0$ by the following
  dynamics:%
  \footnote{
    We remind the reader that, in $\mu$P (\cref{defn:mup}), $w_t^{L+1}$ is the output layer weights normalized so that $w_t^{L+1}$ has $\Theta(1)$-sized entries (whereas $W_t^{L+1} = \frac 1 n w_t^{L+1}$).
    The same point applies to $w_t^1$ (but $W_t^1 = w_t^1$).
    We use lower case $w$ for input and output weights while upper case $W$ for other layers to emphasize that the former are \emph{vector-like} parameters (one dimension going to $\infty$) while others are \emph{matrix-like} (two dimensions going to $\infty$).
    }
  \begin{enumerate}
    \item(Forward and Backward Propagation)
      \begin{align*}
        \mathring \ff_t = \braket{w_t^{L+1}}{\xx^L_t},\quad
        \mathring{\cchi}_{t}=\eps_{t}(\mathring{\ff}_{t})
      \end{align*}
      \begin{align*}
        \ket{\hh_{t}^{1}} & =\ket{w_{t}^{1}}\xxi, & 
          \ket{\hh_{t}^{l}} &= \opket{W^l_t}{\xx^{l-1}_t}, &
          \ket{\xx_{t}^{l}} & =\phi(\ket{\hh_{t}^{l}}), \\
      \ket{d\xx_{t}^{L}} & =\ket{w_{t}^{L+1}}\otimes\boldsymbol{1}_{\NN}, & 
        \ket{d\xx^{l-1}_t} &= \opket{W^{l\trsp}_t}{d\hh^l_t}, &
        \ket{d\hh_{t}^{l}} & =\ket{d\xx_{t}^{l}}\odot\phi'(\ket{\hh_{t}^{l}}).
      \end{align*}
    \item(Parameter Updates)
      \begin{align}
        \ket{w^1_{t+1}}&=\ket{w^1_{t}}-\eta \overline{\ket{d\hh_{t}^1}_{\mathring{\cchi}_{t}}\xxi^{\trsp}}
        \label{eqn:muP_in_update}\\
        \oplim{W^{l}_{t+1}} &= \oplim{W^{l}_{t}} - \eta \overline{\ketdbra{d\hh_t^{l}}{\mathring \cchi_t}{\xx_t^{l-1}}},\ \forall l \in [2,L]
          \label{eqn:muP_op_update}
        \\
        \ket{w_{t+1}^{L+1}}&=\ket{w_{t}^{L+1}}-\eta \overline{\ket{\xx_{t}^{L}} \cdot \mathring{\cchi}_{t}}.
          \label{eqn:muP_out_update}
      \end{align}
      \item(Initialization)
      $\oplim{W^2_0}, \ldots, \oplim{W^L_0}$ are independent initial operators (\cref{defn:oplim}), and
        \begin{align*}
          (\ket{w^1_0}, \ket{w_0^{L+1}}) &= \Gaus(0, I).
        \end{align*}
  \end{enumerate}

\end{thm}

One can check that when $L=1$, we recover \cref{thm:muLimit_1LP}.

Here, \cref{eqn:muP_op_update} uses the operator semantics discussed in and below \cref{defn:oplim} to cleanly express the parameter update.
When unwinded, \cref{eqn:muP_op_update} is equivalent to
\begin{align*}
  \ket{\hh_{t}^{l}} & =\ket{W_{0}^{l}\xx_{t}^{l-1}} - \eta \sum_{s=0}^{t-1} \Qketdbra{d\hh_s^{l}}{\mathring \cchi_s}{\xx_s^{l-1}} {\xx_t^{l-1}}\ra\\
  \ket{d\xx_{t}^{l-1}} & =\ket{W_{0}^{l\trsp}d\hh_{t}^{l}} - \eta \sum_{s=0}^{t-1} \Qketdbra{\xx_s^{l-1}}{\mathring \cchi_s}{d\hh_s^{l}}{d\hh_t^{l}}\ra.
\end{align*}

The proof of \cref{thm:mulimit_MLP} and \cref{thm:muLimit_1LP} can be found in \cref{sec:muLimitProof}.

\begin{rem}[$\mu$P is Most Natural]\label{rem:mup_mostnatural}
Observe that all of the equations above are essentially the ``ket'' versions of what one does in a finite network.
This holds for general architectures: the $\mu$-limit can always be obtained straightforwardly transcribing the tensor operations in a finite network to their counterparts acting on kets in the infinite-width limit.
See \cref{thm:mulimit_anyarch}.
In this sense, $\mu$P is the \emph{most natural parametrization}.
\end{rem}

As before, the general case without stationarity or memorylessness is straightforward given \cref{thm:mulimit_MLP}, albeit with some more notation.
\begin{thm}[Deep $\mu$-Limit, Memoryful Nonstationary Case]\label{thm:mulimit_MLP_general}
  If $\QQ$ is memoryless but not stationary, then \cref{thm:mulimit_MLP} holds if the bars in \cref{eqn:muP_op_update,eqn:muP_in_update,eqn:muP_out_update} are interpreted as $Q^l_t$ (where $l$ is the same as the layer index appearing in $W^l_{t+1}$ or $w^l_{t+1}$ on the LHS).

  If $\QQ$ is not memoryless, then \cref{thm:mulimit_MLP} holds if \cref{eqn:muP_op_update,eqn:muP_in_update,eqn:muP_out_update} are replaced with 
  \begin{align*}
    \ket{w^1_{t+1}}&=\ket{w^1_{t}}-\eta \overline{\ket{d\hh_{\le t}^1}_{\mathring{\cchi}_{\le t}}\xxi^{\trsp}}
    \\
    \oplim{W^{l}_{t+1}} &= \oplim{W^{l}_{t}} - \eta \overline{\ketdbra{d\hh_{\le t}^{l}}{\mathring \cchi_{\le t}}{\xx_{\le t}^{l-1}}},\ \forall l \in [2,L]
    \\
    \ket{w_{t+1}^{L+1}}&=\ket{w_{t}^{L+1}}-\eta \overline{\ket{\xx_{\le t}^{L}} \cdot \mathring{\cchi}_{\le t}}
  \end{align*}
  where the bar notations abbreviate
  \begin{align*}
    \overline{\ket{d\hh_{\le t}^1}_{\mathring{\cchi}_{\le t}}\xxi^{\trsp}}&=Q^1_t \left(\ket{d\hh_{0}^1}_{\mathring{\cchi}_{0}}\xxi^{\trsp}, \ldots, {\ket{d\hh_{t}^1}_{\mathring{\cchi}_{t}}\xxi^{\trsp}}\right)
    \\
    \overline{\ketdbra{d\hh_{\le t}^{l}}{\mathring \cchi_{\le t}}{\xx_{\le t}^{l-1}}}
    &=
      Q^l_t\left(
        \ketdbra{d\hh_{0}^{l}}{\mathring \cchi_{0}}{\xx_{0}^{l-1}}
        ,\ldots,
        \ketdbra{d\hh_{t}^{l}}{\mathring \cchi_{ t}}{\xx_{t}^{l-1}}
        \right)
      \\
    \overline{\ket{\xx^L_{\le t}}\cdot\mathring{\cchi}_{\le t}} &= Q^{L+1}_t\left({\ket{\xx^L_{0}}\cdot\mathring{\cchi}_{0}}, \ldots, {\ket{\xx^L_{t}}\cdot\mathring{\cchi}_{t}}\right)
    .
  \end{align*}
\end{thm}

\begin{rem}
  \cref{thm:muLimit_1LP} and \ref{thm:muLimit_1LP_general} also hold if the output dimension is greater than 1, but the equations need to be interpreted slightly differently.
  See \cref{eqn:ket_outer_product_general}.
\end{rem}

\begin{rem}[What is $\mu$P ``maximal'' in?]\label{rem:mupmaximal}
  For SGD, \cite{yang4} showed $\mu$P is the \emph{unique} stable parametrization where every weight matrix is updated maximally \citep[Defn 5.2]{yang4} and the output weight matrix is also initialized maximally \citep[Defn 5.4]{yang4}.
  These definitions still make sense here and this statement holds as well if ``stable'' is replaced with ``stable and faithful'' (\cref{defn:stability}, \ref{defn:faithful}).
\end{rem}

\section{Dynamical Dichotomy: The Classification of abcd-Parametrizations}
\label{sec:classifyabcd}

In this section, we characterize the infinite-width limits of all possible abcd-parametrizations under reasonable assumptions of the optimizer and network nonlinearity, generalizing the work done for SGD in \cite{yang4}.
First, we filter out the uninteresting limits: the unstable (training blows up), the trivial (training gets stuck at initialization), and the unfaithful (the update functions $\QQ$ and/or nonlinearities $\phi$ are trivialized).
We sort all other limits into a \emph{Dynamical Dichotomy} (\cref{cor:dichotomyMain}) between feature learning and operator regimes (the latter being the nonlinear version of \emph{kernel regime} in the SGD case).
The $\mu$ and NT limits are respectively their archetypes (indeed, the \emph{maximal} parametrizations in these regimes (\cref{{rem:maximality}})).
Like in the SGD case, this dichotomy is not tautological: it implies certain network training dynamics cannot be the infinite-width limit of any abcd-parametrization (\cref{rem:invalidlimits}).
Likewise, pretraining is still futile in the operator regime even with adaptive optimizers (\cref{rem:pretraining_operator_regime}).

Our results here hold for not only memoryless stationary but also memoryful nonstationary updates.

\subsection{Technical Assumptions}

\begin{defn}
  We say a function $F:\R\to\R$ \emph{preserves positivity} if $F(x)>0$
  whenever $x>0$. We say it \emph{preserves sign} if $\sgn F(x)=\sgn(x)$
  for all $x$ (where $\sgn$ takes value in $\{-1,0,1\}$).
\end{defn}
For proving the Dynamical Dichotomy Theorem for entrywise updates, we will make the following technical assumptions.
Roughly speaking, we will focus only on ``relu-like'' nonlinearities and sign-preserving update functions with mild smoothness.
\begin{assm}\label{assm:relulike_phi_Q_preserves_sign}
  Suppose
  \begin{enumerate}
    \item $\phi$ and $\phi'$ are nonnegative and pseudo-Lispchitz.
    \item $\phi$ preserves positivity.
    \item There exists $\delta>0$ such that $t^\delta \phi(x/t)$ converges uniformly (as a function in $x \in \R$) to $\relu(x)^\delta$ as $t\to 0$.
    \item $Q^l_t$ is pseudo-Lipschitz for all $l,t$, and $Q^l_0$ preserves sign for all $l$.%
  \end{enumerate}
\end{assm}

\begin{rem}[Sufficiency]
  As in \cite{yang4}, the pseudo-Lipschitzness of $\phi, \phi', Q^l_t$ are sufficient for letting us use our new Master Theorem (\cref{thm:MasterTheorem}) to take the limits of any parametrization, get the operator limits, and prove their properties.
  Any assumptions beyond such is required only for proving that $r=0$ implies feature learning (\cref{thm:r=0impliesFL}).
  In particular, the reason we only require $Q^l_0$ to preserve signs (instead of for all $t$) is because we will only need to show that features evolve in the first step.
\end{rem}

\begin{rem}[Necessity]
  \cref{assm:relulike_phi_Q_preserves_sign} is satisfied by typical smooth versions of relu like gelu and its powers.
  However, note that these are only a specific set of conditions that allow us to easily prove our desired result and are likely very far from a necessary set of conditions.
  For example, 1) the pseudo-Lipschitz conditions can likely be relaxed to allow $\phi=\relu$ and $Q^l_t=\sgn$, and 2) the uniform convergence in \cref{assm:relulike_phi_Q_preserves_sign} certainly can be relaxed to some measure-theoretic convergence.
  In fact, we expect all theorems below in this section to hold for \emph{generic} activations and update functions.
  We leave this to future work.
\end{rem}

\subsection{Size of Feature Learning}

In \citep{yang4}, the number $r$ of an abc-parametrization measures how much the features change over training.
We adapt its definition to abcd-parametrizations.
\begin{defn}\label{def:r}
  Define
  \[
  r_{l}\defeq\begin{cases}
  c_{l}+a_{l}-1 & \text{if \ensuremath{l>1}}\\
  c_{l}+a_{l} & \text{if \ensuremath{l=1}}
  \end{cases},\quad
  r_{\le l} \defeq \min_{m=1}^l r_l,\quad
  r \defeq r_{\le L}
  \]
\end{defn}

Morally, in any ``reasonable'' abcd-parametrization (in the sense of stable (\cref{defn:stability}) and faithful (\cref{defn:faithful}) discussed below), we have $\Delta W_{t}^{l}\xx_{t}^{l-1}=\Theta(n^{-r_{l}})$
and $\Delta \xx_{t}^{L}=\Theta(n^{-r})$, where $\Delta \bullet_t = \bullet_t - \bullet_0$ is the cumulative change of $\bullet_t$.
Concretely, for NTP and $\mu$P we have, for all $l \in [1, L]$,
\begin{align}
  r=r_l &= \nicefrac{1}{2}\quad \text{in NTP (\cref{defn:NTP})}&
  r=r_l &= 0\quad\text{in $\mu$P (\cref{defn:mup})}
  \label{eqn:concrete_r}
\end{align}
The reader should sanity check that $r_l$ and $r$ are invariant to the symmetry in \cref{stmt:SGDsymmetry}.

\begin{rem}
Ostensibly, \cref{def:r} is different from and much simpler than \cite[Defn 3.2]{yang4} for abc-parametrizations.
But, in fact, they are equivalent for stable and faithful abcd-parametrizations (under the reduction \cref{eqn:abc_reduction} to abc-parametrizations).
The comparative simplicity of \cref{def:r} is also due to the faithfulness.
\end{rem}

\subsection{Stability and Faithfulness}

We will only care about any parametrization satisfying two basic properties:
1) does not blow up during at initialization or during training as
width $\to\infty$ and 2) does not trivialize the update functions
$Q_{t}^{l}$. The former property is known as \emph{stability} and has already
been studied in \cite{yang4} for SGD. The latter is specific to nonlinear entrywise updates, and we will call this the \emph{faithful} property. By ``trivialize,''
we mean that the input to $Q_{t}^{l}$ is either i) too small and
thus linearizing $Q_{t}^{l}$ around the origin or ii) too large and
only depends on $Q_{t}^{l}$'s behavior ``around infinity''. Both
scenarios ignore the bulk of $Q_{t}^{l}$'s values as a function.
If such behaviors are actually desired, then one can change $Q^l_t$ to such effects.
For example, the linearizing behavior in case
i) can be implemented by choosing a linear $Q_{t}^{l}$ and modifying
$c_{l},d_{l}$ appropriately so that the input to $Q_{t}^{l}$ has
constant typical size (wrt width).

Recall from \cref{sec:bigO} the \emph{entry-wise semantics} of big-O notation.
We formalize the \emph{stability} and \emph{faithfulness} properties below.
\begin{defn}[Stability]\label{defn:stability} 
  We say an abcd-parametrization of an $L$-hidden layer MLP is
  \begin{enumerate}
  \item \emph{stable at initialization} if %
  \begin{equation}
  \hh_{0}^{l},\xx_{0}^{l}=\Theta(1),\forall l\in[L],\quad\text{and}\quad \ff_{0}=O(1).\label{eq:initstable}
  \end{equation}
  \item \emph{stable during training} if for any training routine, any time $t\ge0$, $l\in[L]$, we have 
  \[
  \Delta \hh_{t}^{l},\Delta \xx_{t}^{l}=O(1),\forall l\in[L],\quad\text{and}\quad \Delta\ff_{t}=O(1).
  \]
  \end{enumerate}
  We say the parametrization is \emph{stable} if it is stable both at initialization and during training.
\end{defn}

\begin{defn}\label{defn:faithful}
We say an abcd-parametrization is \emph{faithful at time $t$} if the input
to $Q_{_{t}}^{l}$ is $\Theta(1)$ for every $l \in [L]$.
We also say it is \emph{faithful at initialization} if this is true at $t=0$.
\end{defn}

\begin{rem}
  The condition $\hh_{0}^{l},\xx_{0}^{l}=\Theta(1)$ in \cref{eq:initstable} is in truth more of a ``faithfulness to $\phi$'' condition than just stability (which would strictly speaking be more like $O(1)$ than $\Theta(1)$).
  But we never need to distinguish these notions, so following \cite{yang4}, we will keep the definition as is.
\end{rem}

\begin{lemma}%
  \label{lemma:stabilityconditionsinit}
  Adopt \cref{assm:relulike_phi_Q_preserves_sign}. An abcd-parametrization is stable at initialization iff
  \begin{equation}
    a_{1}+b_{1}=0;\quad a_{l}+b_{l}=1/2,\ \forall l\in[2,L];\quad\text{and}\ \ a_{L+1}+b_{L+1}\ge1/2.\label{eq:actlogitinit}
  \end{equation}  
\end{lemma}
This condition is the same as \cite[Thm H.6(1)]{yang4}.
For example, SP, NTP, and $\mu$P are all stable at initialization.
In this situation, some easy calculation shows that, at initialization, the last layer gradients have entry size $\Theta(n^{-a_{L+1}})$ and while all other layers' gradients have entry size $\Theta(n^{-a_l-a_{L+1}-b_{L+1}})$.
Hence,
\begin{lemma}\label{lemma:faithfulconditionsinit}
  In \cref{lemma:stabilityconditionsinit}, the abcd-parametrization is furthermore faithful at initialization iff
  \begin{equation}
    d_{l}= a_l + a_{L+1}+b_{L+1}\text{\ \ for $l\le L$}\quad\text{and}\quad d_{L+1}=a_{L+1}.
    \label{eqn:faithful_init}
  \end{equation}

\end{lemma}

For example, NTP and $\mu$P are faithful at initialization but SP (\cref{exmp:SP}) is not (but if $\QQ$ is scale invariant then SP is equivalent to a faithful parametrization).

\begin{thm}[Stability and Faithfulness Characterization]\label{thm:stablefaithful}
  Adopt \cref{assm:relulike_phi_Q_preserves_sign}. 
  Suppose, at initialization, an abcd-parametrization is both faithful and stable. Then it remains so for all $t\ge 1$ iff 
  \begin{equation}
    r_l\ge0\ \text{ for all } l \in [L+1]\quad\text{and}\quad
    a_{L+1} + b_{L+1} + r \ge 1\quad\text{and}\quad
    b_{L+1}\le c_{L+1}.
    \label{eqn:stable_faithful}
  \end{equation}
\end{thm}

In words: 1) $r_l\ge0$ for all $l \in [L]$ ensures the features do not blow up, while 
2) $r_{L+1} \ge0$ and $a_{L+1} + b_{L+1} + r \ge 1$ resp.\ ensure that $W_{t}^{L+1}\xx_{t}^{L},W_{0}^{L+1}\Delta \xx_{t}^{L}=O(1)$, so $\ff_t$ does not blow up;%
\footnote{Recall $\Delta \bullet_t = \bullet_t - \bullet_0$ is the cumulative change of $\bullet_t$.}
finally, 3) $b_{L+1}\le c_{L+1}$ ensures that $W^{L+1}$ does not change scale after updates, since otherwise the $d_l$ in \cref{eqn:faithful_init} is no longer faithful.

\begin{rem}\label{rem:minormismatch}
One can check that \cref{thm:stablefaithful} reduces to \cite[Thm 3.3]{yang4} (the stability characterization of abc-parametrizations) plus the additional constraint that $W^{L+1}_t = O(W^{L+1}_0)$ for all $t$ (which is imposed by $b_{L+1}\le c_{L+1}$ in \cref{eqn:stable_faithful}).
As remarked above, this constraint is due to the faithfulness requirement.
But in fact, if we allow $d_l$ to vary during training, then this constraint can be removed and the appropriate version of \cref{thm:stablefaithful} would be equivalent to \cite[Thm 3.3]{yang4}.
\end{rem}

\subsection{Nontriviality}

Even among stable and faithful parametrizations, we are only interested in \emph{nontrivial} parametrizations (as in \cite{yang4}),
where in the infinite width limit the network will not be stuck at
initialization during training.
\begin{defn}[Nontriviality]
  We say an abcd-parametrization of an $L$-hidden layer MLP is \emph{trivial} if for every training routine and any time $t\ge1$, $\ff_{t}-\ff_{0}\asto0$ as $n\to\infty$ (i.e., the function does not evolve in the infinite-width limit). We say the parametrization is \emph{nontrivial} otherwise.
\end{defn}

Nontriviality is characterized by a \emph{disjunction} of equations in $a_l,b_l,c_l$, just as for SGD in \cite{yang4}.
\begin{thm}\label{thm:nontrivial}
  Adopt \cref{assm:relulike_phi_Q_preserves_sign}. 
A stable and faithful abcd-parametrization is nontrivial iff 
\begin{equation}
  a_{L+1}+c_{L+1}=1
  \quad \text{or}\quad a_{L+1}+b_{L+1}+r=1.
  \label{eqn:nontrivial}
\end{equation}
\end{thm}
This is essentially equivalent to \cite[Thm 3.4]{yang4}.
For example, NTP and $\mu$P are both nontrivial.

\subsection{Feature Learning and Operator Regimes}

Finally, we ask: what are the different possible behaviors among the nontrivial, stable, and faithful parametrizations?
As in \cite{yang4}, we will see a dichotomy between feature learning and a \emph{nonlinear} version of the kernel regime which we call the \emph{operator regime}.

\begin{defn}[The Operator Regime]\label{defn:operatorregime}
  For memoryless stationary updates, 
  we say an abcd-parametrization of an $L$-hidden layer MLP \emph{is in the operator regime} if there exists a function $\calK: \R^{\NN} \to \R^\NN$, which we call an \emph{operator}, such that for all training routines and every $t \ge 0$, as width $n\to\infty$,
  \begin{equation}
    \ff_{t+1} - \ff_t - \eta \calK(\ff_t) \asto 0.
    \label{eqn:kernelupdate}
  \end{equation}
  For memoryless, nonstationary updates, we allow $\calK$ to depend on $t$.
  For general entrywise updates, we make the same definition if for all $t$,
  \begin{equation}
    \ff_{t+1} - \ff_t - \eta \calK_t(\ff_0, \ldots, \ff_t) \asto 0.
    \label{eqn:memoryfulkernelupdate}
  \end{equation}
  for some $\calK_t: \R^{(t+1) \times \NN} \to \R^\NN$.
\end{defn}

Notice that the operator regime is defined solely in the \emph{function space}, without talking about the internals of the network, in contrast to feature learning.

\begin{defn}[Feature Learning]\label{defn:featurelearning}
  We say an abcd-parametrization of an $L$-hidden layer MLP \emph{admits feature learning in the $l$th layer} if there exists some training routine such that
  \begin{equation}
  \Delta \xx_{t}^{l}=\Omega(1)\label{eqn:xchanges}
  \end{equation}
  for some $t\ge0$.
  We say the parametrization \emph{admits feature learning} if it does so in any layer.
  
  We say the parametrization \emph{fixes the $l$th layer features} if for all training routine, 
  \[
  \|\Delta \xx_{t}^{l}\|^{2}/n\asto0
  \]
  for all $t\ge0$. We say the parametrization \emph{fixes all features} if it does so in every layer.
  
  We make similar definitions as above replacing \emph{feature} with \emph{prefeature} and $\xx^{l}$ with $\hh^{l}$.
\end{defn}

\begin{defn}[Feature Kernel Evolution]\label{defn:featurekernelevolve}
  We say an abcd-parametrization of an $L$-hidden layer MLP \emph{evolves the $l$th layer feature kernel} if there exists some training routine such that
  \[
  \xx_{t}^{l}\xx_{t}^{l\trsp}/n-\xx_{0}^{l\trsp}\xx_{0}^{l}/n=\Omega(1)
  \]
  for some $t\ge0$.
  We say the parametrization \emph{evolves feature kernels} if it does so in any layer.
  
  We say the parametrization \emph{fixes the $l$th layer feature kernel} if for all training routine, 
  \[
  \xx_{t}^{l\trsp}\xx_{t}^{l}/n-\xx_{0}^{l\trsp}\xx_{0}^{l}/n\asto0,\quad\text{as}\quad n\to\infty,
  \]
  for all $t\ge0$. We say the parametrization \emph{fixes all feature kernels} if it does so in every layer.
  
  We make similar definitions as above replacing \emph{feature} with \emph{prefeature} and $x^{l}$ with $h^{l}$.
\end{defn}

\subsection{Classification of abcd-Parametrizations}

\begin{wrapfigure}{r}{0.35\textwidth}
  \begin{center}
      \includegraphics[width=0.35\textwidth]{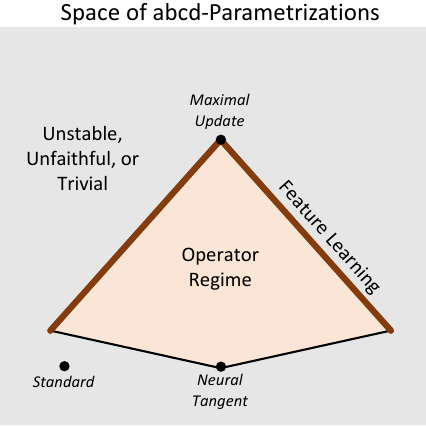}
  \end{center}
  \caption{\textbf{A Caricature of abcd-Parametrizations.}
  The nontrivial stable faithful parametrizations form a high dimensional polyhedron.
  Those on a part of its boundary admit feature learning, while all others are in the operator regime.
  $\mu$P is a vertex in the former, while NTP, latter.
  The overall shape is similar to \cite[Fig.\ 2]{yang4}}
  \label{fig:abcdparamspace}
  \vspace{-80pt}
\end{wrapfigure}

The classification of abcd-parametrizations is similar to that of abc-parametrizations \citep[Thm H.13]{yang4}.
We remind the reader that this holds for not only memoryless stationary but more generally memoryful nonstationary updates.
\begin{restatable}{thm}{main}
  \label{thm:abcdclassification}
  Adopt \cref{assm:relulike_phi_Q_preserves_sign}. Consider a nontrivial, stable, and faithful abcd-parametrization of an $L$-hidden layer MLP. Then
\begin{enumerate}
\item The following are equivalent to $r=0$ \label{enum:r=0}
\begin{enumerate}
\item feature learning
\item feature learning in the $L$th layer
\item feature kernels evolution
\item feature kernel evolution in the $L$th layer
\item prefeature learning
\item prefeature learning in the $L$th layer
\item prefeature kernels evolution
\item prefeature kernel evolution in the $L$th layer
\end{enumerate}
\item The following are equivalent to $r>0$ \label{enum:r>0}
\begin{enumerate}
\item the operator regime
\item fixes all features
\item fixes features in the $L$th layer
\item fixes all feature kernels
\item fixes feature kernel in the $L$th layer
\item fixes all prefeatures
\item fixes prefeatures in the $L$th layer
\item fixes all prefeature kernels
\item fixes prefeature kernel in the $L$th layer
\end{enumerate}
\item If there is feature learning \emph{or} feature kernel evolution \emph{or} prefeature learning \emph{or} prefeature kernel evolution in layer $l$, then there is feature learning \emph{and} feature kernel evolution \emph{and} prefeature learning \emph{and} prefeature kernel evolution in layers $l,\ldots,L$.
\item If $r=0$, then $\ff_{0}\asto0$ and $\ff_t \asto \mathring \ff_t$ for some deterministic $\mathring \ff_t$.
However, the converse is not true.\label{item:FLDeterministic}
\end{enumerate}
\end{restatable}

Consequently, we can generalize Dynamical Dichotomy to (nonlinear) entrywise updates.
\begin{cor}[Dynamical Dichotomy]\label{cor:dichotomyMain}
  A nontrivial, stable, and faithful abcd-parametrization either admits feature learning or is in the operator regime, but not both.
\end{cor}
Of course, the canonical examples here are $\mu$P in the feature learning regime and NTP in the operator regime.
In the SGD case, \cref{{thm:abcdclassification},cor:dichotomyMain} are equivalent to their counterparts \cite[Thm H.13, Cor H.14]{yang4} other than a minor technical difference as discussed in \cref{rem:minormismatch}.

\begin{rem}[Maximality]\label{rem:maximality}
  The $\mu$ and NT limits are resp.\ the ``maximal'' limits in the feature learning and operator regimes, in the sense that all parameter tensors contribute to the function update, and that any other limits in those regimes are just ``downgrades'' of $\mu$ and NT limits by zeroing out the initialization or learning rate of some parameters.
  See also \cref{rem:mupmaximal}.
\end{rem}

\begin{rem}[Pretraining is still futile in the operator regime even with adaptive optimizers]\label{rem:pretraining_operator_regime}
  \cite[Thm H.17]{yang4} holds almost verbatim in our case as well, after replacing ``stable'' with ``stable and faithful'' and ``kernel regime'' with ``operator regime'':
  Finetuning any pretrained network in the operator regime would be equivalent to finetuning a randomly initialized network.
  Thus, pretraining in the operator regime is useless.
\end{rem}

\begin{rem}[Function Space Picture]\label{rem:functionspacepicture}
  In the memoryless stationary case, an operator regime limit resides solely in the \emph{function space picture}, i.e. $\ff_{t+1}$ being solely determined by the function values $\ff_t$ themselves (as opposed to the internal activations of $f$ as well) along with learning rate $\eta$ and error signals $\eps_t$.
  However, as in \cite[Remark 3.11]{yang4}, this is not true of any feature learning limit because one can construct counterexamples where $\ff_t$ are close for two infinite-width limits but $\ff_{t+1}$ are far.
\end{rem}

\begin{rem}[Not All Dynamics are Infinite-Width Limits]
  \label{rem:invalidlimits}
  Compared to the kernel regime in the SGD case, the operator regime now allows nonlinear evolution in the function space picture.
  Nevertheless, in such dynamics, $\ff_{t+1} - \ff_t$ \emph{must be linear in $\eta$} for every $t$.
  Thus, any function space evolution nonlinear in $\eta$ cannot be the infinite-width limit of any entrywise optimizer.%
  \footnote{The same holds for any adaptive optimizer with ingredients discussed under \emph{Optimizer Coverage} of \cref{sec:optimizer}.}
  For example, $\ff_{t+1} - \ff_t = -\eta \ff_t  - \eta^2 \ff_t^2$ is not a valid limit.
\end{rem}

\begin{rem}[Uniform Parametrization]
  \cite[Sec G]{yang4} identified a subclass of abc-parametrizations, called \emph{uniform parametrizations}, where all layers ``learn the same amount of features'' and the output layer is initialized and updated maximally.
  This is used in \cite{yaida_meta-principled_2022} to give an alternative presentation of $\mu$P as well as discussion of joint width-depth limit.
  This notion also makes sense for abcd-parametrizations:
  For every $s \in [0, \nicefrac 1 2]$, there is a unique stable and faithful abcd-parametrization, called \emph{UP$_s$} such that $r_l = s$ for all $l =1,\ldots,L$ and $r_{L+1}=1$ and $a_{L+1} + b_{L+1} = 1 - s$.
  For example, UP$_0$ is $\mu$P and UP$_{\nicefrac 1 2}$ is NTP.
\end{rem}

\section{Infinite-Width Limits for Any Architecture}
\label{sec:anyarch}

Having written down the infinite-width limits of adaptive optimizers for MLPs, we now turn to general architectures.
An astute reader may have already absorbed the key insights from the previous sections and can use them to derive the NT or $\mu$ limits for each new architecture in an \emph{ad hoc} fashion.
In contrast, here we describe the algorithm to do this once and for all, uniformly for all ``reasonable'' architectures.
This uniformity of course requires abstraction, which is not conducive to quick comprehension; on the flip side, this algorithm will always be here for someone to fall back to if the \emph{ad hoc} approach does not work out.

The main work here is not the proofs (which can be easily adapted from the MLP case and so are omitted), but the \emph{definitions}: What is an architecture? What architecture counts as ``reasonable''? How to define abcd-parametrization for any such architecture? What are their infinite-width limits?
Answering these questions in the most generality requires careful thought.

\subsection{Representating Architectures via Tensor Programs}

\begin{defn}[Architecture and Representability]\label{defn:representable_architecture}
  Let $\calT_n = (\R)^j \oplus (\R^n)^k \oplus (\R^{n\times n})^l$.
  We shall call this the \emph{parameter space with $l$ matrix, $k$ vector, and $j$ scalar parameters.}

  A function family $f = \{f_n(-;-): \R^d \times \calT_n \to \R^{\dout}\}_{n=1}^\infty$ (with fixed input and output dimensions $d$ and $\dout$ independent of $n$) is called an \emph{architecture}.
  It has $l$ matrix, $k$ vector, and $j$ scalar parameters.

  We say an architecture $f$ is \emph{representable} if there is a \nexort{} program $\pi$ and vectors $x^1, \ldots, x^{\dout}$ in $\pi$ such that
  1) $\pi$ has $j+d$ initial scalars, $k$ initial vectors, and $l$ initial matrices; and
  2) for any $n$, when they are instantiated with $\xi \in (\R)^d$ and $\Theta \in \calT_n$,%
  \footnote{In particular, the first $j$ initial scalars are set to the $(\R)^j$ part of $\Theta$ and the last $d$ scalars are set to $\xi$.}
  the sum of $x^i$'s entries yields $f_n(\xi; \Theta)_i$ for each $i \in [{\dout}]$.
  In this case, $(\pi, x^1, \ldots, x^{\dout})$ is called a \emph{representation} of $f$, and $\pi$ is called a \emph{representation program} of $f$.
\end{defn}

As demonstrated in \citep{yang,yang2}, \cref{defn:representable_architecture} covers essentially every architecture in practice: RNNs, residual network, transformers, etc.

\begin{rem}
In this definition, only the symbolic structure of the program matters; the random sampling of \cref{setup:nexort} and \cref{setup:nexort_nongaussian} plays no role.
\end{rem}

\begin{rem}
For simplicity, we only considered the case when the notion of ``width'' is the same throughout the network.
Nevertheless, this definition can be easily modified to cover the nonuniform case, but stating it would be much more complex.
\end{rem}

\begin{rem}
In truth, we could have phrased \cref{defn:representable_architecture} using \netsort{} instead of \nexort{}, since we do not know of any neural network in the wild that is not \netsort{}-representable (\nexort{} is only required for expressing adaptive optimization like Adam).
But there is little cost in stating the more general version, which can potentially matter in the future.
\end{rem}

\begin{rem}
  \cite{tp2b} also defined a notion of \emph{representable functions} using \netsort{}.
  In comparison, our definition is much more general: Beyond the superficial difference of \nexort{} here vs \netsort{} there, \cite{tp2b} dealt with input and output layer weights in special ways, whereas here we do not, instead opting to uniformly deal with scalar, vector, and matrix parameters.
  The input and output weights of \cite{tp2b} are vector parameters in this view.
\end{rem}

\begin{exmp}
The $L$-hidden-layer MLP in \cref{eqn:MLP} has $d+1$ vector parameters (where $d$ corresponds to input dimension and $1$ corresponds to output dimension) and $L-1$ matrix parameters.
It is represented by the program that generates 1) $h^1$ using \refOuterNonlin{} and $h^l, l\ge 2,$ using \refMatMul{}; 2) $x^l$ using \refOuterNonlin{}; and 3) generate function output by summing $W^{L+1} \odot x^L$ (so we can take $x^1$ in \cref{defn:representable_architecture} to be $W^{L+1} \odot x^L$).
\end{exmp}

\subsection{abcd-Parametrization for Any Architecture}

\begin{defn}[abcd-Parametrization for Representable Architectures]\label{defn:generalabcd}
Consider a representable architecture with representation program $\pi$.
Fix a set of update functions $\QQ = \{Q^W_t:\R^{t+1}\to\R\}_{t\ge0, W}$ where $W$ ranges over every matrix, vector, and scalar parameters.
An \emph{abcd-parametrization} of this architecture is specified by $a_W, b_W, c_W, d_W$ for each such $W$, along with an additional number $\aout$, such that
\begin{enumerate}[\hspace{30pt} (a)]
  \item We parametrize $W$ as $W=n^{-a_W}w$ for actual
  trainable parameter $w$;
  \item We initialize each entry of $w$ iid from $\Gaus(0,n^{-2b_W})$;
  \item The learning rate is $\eta n^{-c_W}$ for some width-independent
  $\eta$;
  \item The gradients of $w$ are multiplied by $n^{d_W}$ before being
  processed by $Q^W_t$: i.e., the update at time $t$ is
  \[
  w\gets w-\eta n^{-c_W}Q^W_t(n^{d_W}g_{0},\ldots,n^{d_W}g_{t})
  \]
  where $g_{s},s=0,\ldots,t$, are the gradients of $w$ at time
  $s$ and $Q^W_t$ is applied entrywise;
\end{enumerate}
and the function output is multiplied by $n^{-\aout}$.%
\footnote{i.e., $f_n(\xi;\Theta)_i = n^{-\aout}\sum_{\alpha=1}^n x^i_\alpha$ instead of just a plain sum.}

\end{defn}

\begin{rem}
  In the MLP case, the $\aout$ was absorbed into $a_{L+1}$.
  But in the generality of \cref{defn:representable_architecture}, output weights are not singled out,%
  \footnote{In fact, output weights may be ill defined in some architectures, such as if $f(\xi) = \sum_\alpha x^L(\xi)_\alpha$.}
  so we need to separately specify $\aout$.
\end{rem}

\begin{rem}
As always, we are only concerned with scaling with $n$ here, but there can be a tunable constant hyperparameter in front of every power of $n$ in \cref{defn:abcd}.
\end{rem}
\begin{rem}
The random initialization in \cref{defn:generalabcd} is always mean-zero.
For some applications, such as layernorm/batchnorm weights $W$ (that is initialized as all 1s), this may seem insufficient.
However, one can just refactor the parameter:
For this particular example, we can refactor $W = 1 + W'$ where $W'$ is the initial vector of the program $\pi$.
$W'$ can then be initialized as $\Gaus(0, \sigma n^{-2b_{W'}})$ for some tunable constant hyperparameter $\sigma$ (which is set to 0 by practitioners typically).
\end{rem}

NTP and $\mu$P naturally generalize to general representable architectures.

\begin{defn}[NTP for general architecture]\label{defn:NTPgeneral}
  For any representable architecture, its \emph{Neural Tangent Parametrization (NTP)} is defined by the following setting of $a,b,c,d$ for matrix, vector, and scalar parameters as well as the output multiplier $\aout$.
  \begin{center}
    \begin{tabular}{ccccc}
      \toprule
      & matrix & vector & scalar & out\tabularnewline
      \midrule
      $a$ & $\nicefrac{1}{2}$ & $0$ & $0$ & $\nicefrac{1}{2}$\tabularnewline
      $b$ & $0$ & $0$ & $0$ & -\tabularnewline
      $c$ & $1$ & $\nicefrac{1}{2}$ & $0$ & -\tabularnewline
      $d$ & $1$ & $\nicefrac{1}{2}$ & $0$ & -\tabularnewline
      \bottomrule
    \end{tabular}
  \end{center}
\end{defn}

\begin{defn}[$\mu$P for general architecture]\label{defn:muP_general}
  For any representable architecture, its \emph{Maximal Update Parametrization ($\mu$P)} is defined by the following setting of $a,b,c,d$ for matrix, vector, and scalar parameters as well as the output multiplier $\aout$.
  \begin{center}
    \begin{tabular}{ccccc}
      \toprule
      & matrix & vector & scalar & out\tabularnewline
      \midrule
      $a$ & $0$ & $0$ & $0$ & $1$\tabularnewline
      $b$ & $\nicefrac{1}{2}$ & $0$ & $0$ & -\tabularnewline
      $c$ & $1$ & $0$ & $0$ & -\tabularnewline
      $d$ & $1$ & $1$ & $0$ & -\tabularnewline
      \bottomrule
    \end{tabular}
  \end{center}
\end{defn}

\begin{rem}
  In comparison to their counterparts for MLP, the NTP and $\mu$P above consider scalar parameters (which are not present in the MLP in \cref{eqn:MLP}).
  Otherwise, \cref{defn:NTP} can be recovered from \cref{defn:NTPgeneral} by mapping the columns $[2, L]$ to ``matrix,'' $1$ to ``vector,'' and $L+1$ to ``vector'' but with the value of $a$ taken from the ``out'' column.
  \cref{defn:mup} can be recovered likewise from \cref{defn:muP_general}.
\end{rem}

Before we formulate their limits, we need to discuss how to construct the ``backpropagation'' of an arbtrary of \nexort{} program.%
\footnote{This has been constructed previously for \netsort{} programs in \citep{tp3b}.}

\subsection{Interlude: Backpropagation and Total Programs}

Recall that 
when $x \in \R^n$, the notation $\la x_\alpha \ra_\alpha$ denotes its average entry; this applies more generally to tensors, such as $\la x_{\alpha_1\ldots \alpha_r}\ra_{\alpha_1\cdots\alpha_r}$ when $x \in \R^{n\times \cdots \times n} = (\R^n)^{\otimes r}$.

\begin{defn}[Backpropagation Program]\label{defn:backpropagation}
  Consider any \nexort{} program $\pi$ and a vector $x$ in $\pi$.
  Then \emph{$\pi$'s backpropagation program wrt $x$} is an extension of $\pi$ defined by constructing the following objects on top of $\pi$:
  (Intuitively, one should interpret $d^x y = n \pdf{\la x_\alpha \ra_\alpha}{y}$ if $y$ is a vector and $d^x c = \pdf{\la x_\alpha \ra_\alpha}{c}$ if $c$ is a scalar.)
  \begin{itemize}
    \item $d^{x} x := 1_n \in \R^n$
    \item For any \refMatMul{} instruction $z:=Wy$ in $\pi$, we construct $d^{x|z} y := W^\trsp d^x z$ (via another \refMatMul{})
    \item For any \refAvg{} instruction $c := \la z_\alpha \ra_\alpha$ in $\pi$, we construct $d^{x|c} z := (d^x c) {1}_n \in \R^n$ (via \refOuterNonlin{})
    \item 

    Suppose $y = \la \psi(\xx; \xx_{\beta_1}; \cdots; \xx_{\beta_r}; \cc)\ra_{\beta_1\cdots \beta_r}$.
    For each $i = 0, \ldots,r$,
    let 
    \[\gg^i_{\beta_0 \cdots\beta_r} = d^x y_\alpha \ppsi_i(\xx_{\beta_0}; \cdots; \xx_{\beta_r}; \cc) \in \R^{|\xx|},\]
     where $\ppsi_i: \R^{|\xx|(r+1)+l} \to \R^{|\xx|}$ yields the derivative of $\psi$ against $\xx$ in the $i$th slot.
     When $i=r+1$, we make the analogous definition for $\gg^{r+1}_{\beta_0\cdots \beta_r} \in \R^{|\cc|}$.
    We write $\bbeta = (\beta_0, \ldots, \beta_r)$, $\bbeta[i\mapsto \alpha] = (\beta_0, \ldots, \beta_{i-1}, \alpha, \beta_{i+1}, \ldots, \beta_{r})$, and $\bbeta_{-i} = (\beta_0, \ldots, \beta_{i-1}, \beta_{i+1}, \ldots, \beta_r)$.
    Then we construct $d^{x|y} \cc = (d^{x|y} c^1, \ldots, d^{x|y}c^{|\cc|})$ and $d^{x|y} \xx = (d^{x|y} x^1, \ldots, d^{x|y} x^{|\xx|})$ by
    \begin{align*}
      d^{x|y} \cc &:= \la \gg^{r+1}_\bbeta \ra_\bbeta \in \R^{|\cc|} & \text{\refOuterNonlin{} + \refAvg}\\
      d^{x|y} \xx_\alpha 
        &:= \sum_{i=0}^r \la \gg^i_{\bbeta[i \mapsto \alpha]}\ra_{\bbeta_{-i}} \in \R^{|\xx|} & \text{\refOuterNonlin{}}
    \end{align*}
    Explicitly,
    \begin{align*}
      d^{x|y} \xx_\alpha  = \la \gg^0_{\alpha \beta_1 \cdots\beta_r} \ra_{\beta_1 \cdots \beta_r} + \la \gg^1_{\beta_0 \alpha \beta_2\cdots \beta_r} \ra_{\beta_0 \beta_2\cdots\beta_r} + \cdots +
      \la \gg^r_{\beta_0 \cdots \beta_{r-1} \alpha} \ra_{\beta_0 \cdots \beta_{r-1}}
    \end{align*}
    \item Finally, for every vector or scalar $y$ in $\pi$ other than $x$,
    \begin{align*}
      d^x y := \sum_u d^{x|u} y
    \end{align*}
    where $u$ ranges over all vector or scalar in $\pi$ whose construction used $y$.%
    \footnote{If this sum is empty, then the RHS is set to 0.}
  \end{itemize}
  The subprogram%
  \footnote{The notion of \emph{subprogram} is formally defined in \cite[Defn I.1]{yang3}. Roughly it means a contiguous subset of instructions in the program.}
  constructing all of these new objects is denoted $d^x \pi$ (so that the backpropagation program is $\pi | d^x \pi$).
\end{defn}
Recall (\cite[Defn I.1]{yang3}) that ``$|$'' (as in ``$\pi | \pi'$'') signifies the concatenation of programs.
\begin{rem}
  If $r = 0$ in \refOuterNonlin{} (i.e., we just have a \texttt{Nonlin+} instruction), then the formulas simplify to 
  \begin{align*}
    d^{x|y} \cc &:= \la d^x y_\alpha \ppsi_1(\xx_\alpha; \cc) \ra_\alpha\\
    d^{x|y} \xx_\alpha &:= y_\alpha \ppsi_0(\xx_\alpha; \cc).
  \end{align*}
\end{rem}
\begin{defn}[Total Program $\db \ppi$]\label{defn:totalprogram}
  Consider a representable architecture $f$ with representation $(\pi, x^1, \ldots, x^{e})$.
  Gather all of $\pi$'s backpropagation programs wrt $x^i$ into a single (large) program:
  \[\db \pi \defeq \pi | d^{x^1} \pi | \cdots | d^{x^{e}} \pi.\]
  Let $\xi \in \R^d$ be an input to $f$.
  Then $\db \pi(\xi)$ will denote the program $\db \pi$ with $\xi$ inserted as the appropriate initial scalars (c.f.\ \cref{defn:representable_architecture}).
  As in \cref{sec:setup}, we consider a set of $\NN$ inputs $\xxi = \{\xi^1,\ldots, \xi^\NN\} \sbe \R^d$.
  Then we set
  \[\db \ppi \defeq \db \pi(\xi^1) | \cdots | \db \pi(\xi^\NN),\]
  whose initial data are the $j$ scalars, $k$ vectors, and $l$ matrices corresponding to $\calT_n$ in \cref{defn:representable_architecture}; they are shared among all subprograms $\db \pi(\xi^i)$.
  We call $\db \ppi$ the \emph{total program} of $f$.

  For any vector $y$ in $\pi$, we write $\yy = (y(\xi^1), \ldots, y(\xi^\NN)) \in \R^{n \times \NN}$ to denote the multi-vector of its counterparts in $\db \ppi$.
  We also write $\db \yy = (d^{x^j} y(\xi^a))_{j \in [e], a \in [\NN]} \in \R^{n \times e \times \NN}$ for the multi-vector in $\db \ppi$ of its error signals.
\end{defn}

\subsection{Training Setup}

The setup in \cref{sec:setup} is defined with the MLP in mind, but can be generalized easily to deal with general architectures, as we do here for clariy:
1) Because we consider multi-dimensional outputs here, the error signal functions $\eps_t$ obviously need to take the more general signature $\eps_t: \R^{\dout \times \NN} \to \R^{\dout \times \NN}$.
2) The update functions $\QQ$ now contain a $Q^\Theta_t: \R^{t+1} \to \R$ for every matrix, vector, and scalar parameter $\Theta$.
\emph{Memoryless} still means $Q^{\Theta_t}$ only depends on its last argument, and \emph{stationary} still means $Q^\Theta_t$ is the same regardless of $\Theta$ and $t$.
3) The smoothness assumption we require is the natural adaptation of \cref{assm:MLPsmooth}:
\begin{assm}[Smoothness]\label{assm:MLPsmooth_anyarch}
  Assume $\eps_t$ and $Q^\Theta_t$ for all $\Theta,t$ are pseudo-Lipschitz and all nonlinearities used in the representing program has pseudo-Lipschitz derivatives.
\end{assm}

\subsection{Neural Tangent Limit}

\subsubsection{Memoryless Stationary Case}

\begin{defn}[Neural Tangent Operator for General Architecture]\label{defn:GenArchTangentOperator}
  Fix a representable architecture $f$ with representation $(\pi, x^1, \ldots, x^{\dout})$.
  For any function $Q: \R \to \R$, 
  we define
  the \emph{neural tangent $Q$-operator} $\calK_{Q}:\R^{\dout \times \NN}\to\R^{\dout \times \NN}$
  by the following: for any $\cchi\in\R^{\dout \times \NN}$,
  \begin{align}
    \calK_{Q}^W(\cchi) 
        & \defeq \Diag_\NN \sum_{\substack{h=Wz\\
                  g=Wy
                  }
          }
          \la{\db\gb}\Qketdbra{\db\hh}{\cchi}{\zz}{\yy}\ra
          \label{eqn:KQW}
          \\
    \calK_{Q}^v(\cchi)
        & \defeq 
        \la{\db\vv}\overline{\ket{\db\vv}\cdot \cchi}
        \label{eqn:KQv}
        \\
    \calK_{Q}(\cchi) &\defeq \sum_W \calK_{Q}^W(\cchi) + \sum_v \calK_{Q}^v(\cchi)
        \label{eqn:KQgeneral}
  \end{align}
  where the ``bar'' notation abbreviates application of $Q$ as in \cref{eqn:barnotation} and all kets and bras are evaluated in $\db \ppi$ (\cref{defn:totalprogram}) via \cref{defn:ket}.
  To interpret these formulas, we need to tell you two things:
  
  \emph{1) Ranges of arguments.}
  Here, $W$ ranges over all matrix parameters and $v$ over all vector parameters of $f$, and all kets are calculated from $\db \ppi$ by sampling matrix parameters from $\Gaus(0, 1/n)$ and vector parameters from $\Gaus(0, 1)$.%
  \footnote{Again, we can insert hyperparameters like $\sigma_W$ and $\sigma_v$, but for simplicity we omit them here.}
  The sum in \cref{eqn:KQW} sums over all vectors $h,z,g,y$ in $\pi$ satisfying $h=Wz$ and $g=Wy$ (potentially $h=g$ and $z=y$).

  \emph{2) Tensor operations.}
  Here $\ket{\db \gb}, \ket{\db \hh}, \cchi \in \R^{\dout \times \NN}$ while $\ket{\zz}, \ket{\yy} \in \R^{\NN}$.
  First, we take the convention that each bra has the same shape as the corresponding ket:%
  \footnote{Intuitively, $\ket{\db \gb}$ corresponds to $\db \gb \in \R^{n \times \dout \times \NN}$ while $\bra{\db \gb}$ corresponds to $\db \gb$'s ``transpose'' with shape $\R^{\dout \times \NN \times n}$.}
  $\bra{\db \gb}, \bra{\db \hh} \in \R^{\dout \times \NN}$.
  Second, 
  the meaning of $\ketdbra{\db\hh}{\cchi}{\zz}$ is
    \begin{equation}
    \ketdbra{\db\hh}{\cchi}{\zz}=\sum_{a \in [\NN], i \in [\dout]}\chi^{ai}\ket{d^{x^i}h(\xi^a)}\bra{z(\xi^a)}
    \label{eqn:ket_outer_product_general}
    \end{equation}
  (i.e., contraction of all matching indices), which generalizes the case of $\dout = 1$ in \cref{thm:MemorylessGeneralNTKLimit}.
  Likewise, in \cref{eqn:KQv}, $ \ket{\db \vv} \cdot \cchi\in \R$ contracts all indices.
  Finally, $\Diag_\NN$ in \cref{eqn:KQW} takes its argument tensor of shape $\dout \times \NN \times \NN$ and returns a tensor of shape $\dout \times \NN$ by ``taking the diagonal over the $\NN$ dimension.''
\end{defn}

\begin{thm}[Neural Tangent Limit for General Architecture]\label{thm:MemorylessGeneralNTKLimit}
  Consider a representable architecture $f$ with representation $(\pi, x^1, \ldots, x^\dout)$ and any training routine in NTP (\cref{defn:NTPgeneral}) with memoryless stationary update function $Q$.
  Adopt \cref{assm:MLPsmooth_anyarch}.
  Further assume
  \begin{enumerate}
    \item $f$ has no scalar parameters, \label{_assm:noscalar}
          and
    \item $\braket{1}{x^i(\xi)} = 0$ for every input $\xi$ and output index $i \in [\dout]$ at initialization.%
    \footnote{More precisely, the ket $\ket{x^i(\xi)}$ is evaluated in $\db \ppi$ (\cref{defn:totalprogram}) via \cref{defn:ket}.}
    \label{_assm:mean0}
  \end{enumerate}
  Recall $\ff_{t}$ denotes the function after $t$ steps of updates from random initialization.
  Then 
  \[\ff_0 \distto \Gaus(0,\braket{\xx}{\xx}) \] 
  where $\xx$ is the multi-vector consisting of $x^1, \ldots, x^{\dout}$ evaluated on all $\NN$ inputs.
  Additionally, 
  for every $t\ge 0$, $\ff_{t}$ converges
  almost surely to a random function $\mathring{\ff_{t}}\in\R^{\dout \times \NN}$ satisfying
  \begin{align}
    \mathring{\ff}_{s+1} - \mathring{\ff}_{s}&=-\eta\calK_{Q}(\eps_{s}(\mathring{\ff}_{s})),\quad \text{for all $s$}.
    \label{eqn:NTlimit_update_general}
  \end{align}
\end{thm}

\begin{rem}
  Note Simple GIA Check \citep{yang2} may not be satisfied in general architectures, so that we cannot necessarily calculate kets like $\ket{\db \hh}$ by assuming matrix parameters and their transposes are independent (even if no transposes occur in $\pi$), e.g., ignoring $\dotket{\db\hh}$ in our calculations.
  Nevertheless, \cref{thm:MemorylessGeneralNTKLimit} still holds if one calculates $\ket{\db \hh}$ correctly using the rules of \cref{defn:ket}.
\end{rem}

\begin{rem}\label{rem:ntk_anyarch_assumptions}
  Assumption \ref{_assm:mean0} ($\xx$ being mean zero) in \cref{thm:MemorylessGeneralNTKLimit} is obviously necessary for us to arrive at a Gaussian Process limit at initialization.
  Assumption \ref{_assm:noscalar} is also necessary for two reasons: 1) at initialization, random initialization of scalar parameters would make the initial process a mixture of Gaussian processes (or a GP conditioned on the values of scalar parameters).
  Even if the scalar parameters are deterministic, the gradient wrt scalar parameters will also be a random process (possibly correlated to the process of the function) at initialization. 
  2) Function space picture (c.f.\ \cref{rem:functionspacepicture} and \cite[Remark 3.11]{yang4}) would no longer hold: One cannot track the evolution of the $\mathring f_t$ solely by knowing what it is at time $t=0$. Instead, one would need to track the values of the scalar parameters and their gradients as well. 
  So in a sense we will have a ``function space + scalar parameters'' picture.
  The complete evolution of $\mathring f_t$ can then be described by a) the joint process of the function output, scalar parameters, and their gradients at initialization, together with b) an evolution equation involving how they evolve given their previous values in time.
  This is similar to \cite{littwin_bottleneck_2021}.
\end{rem}

\subsubsection{Memoryful Nonstationary Case}

As before, with some more notation we can write down the NT limit for any representable architecture.
\begin{defn}[Neural Tangent Operator, Memoryful Nonstationary Case]
  For memoryless but nonstationary update functions $\QQ_t = \{Q^\theta_t: \R \to \R: \text{parameter }\theta\}$, 
  we define $\calK_{\QQ_t}: \R^{\dout \times \NN} \to \R^{\dout \times \NN}$ with the same equation as  \cref{eqn:KQW,eqn:KQv,{eqn:KQgeneral}}, except the bar notation abbreviates $Q^W_t$ or $Q^v_t$.

  For general update functions $\QQ_t = \{Q^\theta_t: \R^{t+1} \to \R\}_\theta$, 
  we define $\calK_{\QQ_t}: \R^{(t+1) \times \dout \times \NN} \to \R^{\dout \times \NN}$,
  \begin{align}
    \calK_{\QQ_t}^W(\cchi_0, \ldots, \cchi_t) 
        & \defeq \Diag_\NN \sum_{\substack{h=Wz\\
                  g=Wy
                  }
          }
          \la{\db\gb}\Qketdbra{\db\hh}{\cchi_{\le t}}{\zz}{\yy}\ra
          \\
    \calK_{\QQ_t}^v({\cchi_0, \ldots, \cchi_t})
        & \defeq 
        \la{\db\vv}\overline{\ket{\db\vv}\cdot \cchi_{\le t}}
        \\
    \calK_{\QQ_t}(\cchi_0, \ldots, \cchi_t) &\defeq \sum_W \calK_{\QQ_t}^W(\cchi_0, \ldots, \cchi_t) + \sum_v \calK_{\QQ_t}^v(\cchi_0, \ldots, \cchi_t)
  \end{align}
  where $\Qketdbra{\db\hh}{\cchi_{\le t}}{\zz}$ is shorthand for $Q^W_t\left(\ketdbra{\db\hh}{\cchi_0}{\zz}, \ldots, \ketdbra{\db\hh}{\cchi_t}{\zz}\right)$
  and $\overline{\ket{\db\vv}\cdot \cchi_{\le t}}$ is shorthand for $Q^v_t\left(\ket{\db\vv}\cdot \cchi_{0}, \ldots, \ket{\db\vv}\cdot \cchi_{t}\right)$
\end{defn}

With this in mind, the following theorem yields the NT limit of Adam (\cref{eqn:Adam}) as a corollary.
\begin{thm}[Neural Tangent Limit, Memoryful Nonstationary Case]\label{thm:nt_memoryful_anyarch}
  If the update functions $\QQ$ are memoryless but not necessarily stationary, then \cref{thm:MemorylessGeneralNTKLimit} holds with \cref{eqn:NTlimit_update_general} replaced by
  \begin{equation}
    \mathring{\ff}_{t+1} - \mathring{\ff}_{t} =-\eta\calK_{\QQ_t}(\eps_{t}(\mathring{\ff}_{t})),\quad \text{for all $t$}.
  \end{equation}

  For general $\QQ$, not necessarily memoryless, \cref{thm:MemorylessGeneralNTKLimit} holds with \cref{eqn:NTlimit_update_general} replaced by
  \begin{equation}
    \mathring{\ff}_{t+1} - \mathring{\ff}_{t} =-\eta\calK_{\QQ_t}(\eps_0(\mathring \ff_0), \ldots, \eps_{t}(\mathring{\ff}_{t})),\quad \text{for all $t$}.
  \end{equation}
\end{thm}

\subsection{Maximal Update Limit}

Given a representable architecture, the representing program describes the symbolic procedure for computing the output of the network given an \emph{assignment of} the input and parameters to concrete values.
This procedure is well defined for any network width (by construction).
Naturally, it remains well defined as explicitly spelled out below even when we ``pass to the infinite-width limit.''
In short, \cref{defn:assignment} specifies what it means for a program to compute something given \emph{an assignment of} the input and parameters to concrete infinite-width limits (e.g., kets and operators).

\begin{defn} \label{defn:assignment}
  For any program $\pi$ and a vector $y$ in $\pi$, $\ket y$ can be thought of as a function, defined via \cref{defn:ket}, of $\{\oplim{W}\}_W, \{\ket v\}_v, \{\mathring c\}_c$ where $W, v, c$ range over the initial matrices, vectors, and scalars respectively.
  Consider an assignment $\Theta$ that assigns an operator ${\Theta\{W\}}$ to each initial matrix $W$,
  ket ${\Theta\{v\}}$ to each initial vector $v$, and a deterministic number $\Theta\{c\}$ to each initial scalar $c$.
  Then we write 
  \[\Theta\{\ketdbra x {\mathring \cchi} y\},\quad\Theta\{\ket y\},\quad \Theta\{\braket y z\},\quad \Theta\{\mathring \theta\}\]
   for $\ketdbra x {\mathring \cchi} y, \ket y$, $\braket y z$ and $\mathring \theta$ computed using this assignment (i.e., swapping out $\oplim{W}$ for $\Theta\{W\}$, $\ket v$ for $\Theta\{v\}$, and $\mathring c$ for $\Theta\{c\}$).
  
  More precisely, we recursively define
  \begin{align*}
    \ket{\Theta\{x\}} &\defeq \Theta\{\ket x\}\\
    \Theta\{\braket x y\} &\defeq \braket{\Theta\{x\}}{\Theta\{y\}}\\
    \Theta\{\ketdbra x {\mathring \cchi} y\} &\defeq \ketdbra{\Theta\{x\}}{\Theta\{\mathring \cchi\}}{\Theta\{y\}}\\
    \Theta\{\oplim{W}x\ra\}
      &\defeq
        \oplim{\Theta\{W\}} \Theta\{x\} \ra
        \\
    \Theta\left\{\EV_{\bx 1 \cdots \bx r}\psi(\ket{\xx}; \ket{\xx}^{\bx1}; \cdots; \ket{\xx}^{\bx r}; \mathring \cc)\right\}
      &\defeq \EV_{\bx 1 \cdots \bx r}\psi(\ket{\tilde\xx}; \ket{\tilde \xx}^{\bx1}; \cdots; \ket{\tilde \xx}^{\bx r}; \Theta\{\mathring \cc\})
      \quad\text{where $\tilde \xx = \Theta\{\xx\}$}
  \end{align*}
  where we write $\oplim{\Theta\{W\}} = \Theta\{W\}$ for each initial matrix $W$, $\ket{\Theta\{v\}} = \Theta\{v\}$ for each initial vector $v$, and $\Theta\{\mathring c\} = \Theta\{c\}$ for each initial scalar $c$.%
  \footnote{so that $\oplim{\Theta\{\bullet\}}$ and $\ket{\Theta\{\bullet\}}$ are just redundant affirmations of the ``shape'' of $\Theta\{\bullet\}$, rather than saying $\Theta\{\bullet\}$ is some object in some program and $\oplim{\Theta\{\bullet\}}$ or $\ket{\Theta\{\bullet\}}$ are their ``limits.''}

\end{defn}

\subsubsection{Memoryless Stationary Case}
\begin{thm}[$\mu$-Limit for General Architecture]\label{thm:mulimit_anyarch}
  Consider a representable architecture $f$ with representation $(\pi, x^1, \ldots, x^\dout)$ and any training routine in $\mu$P (\cref{defn:muP_general}) with memoryless stationary update function $Q$. 
  Adopt \cref{assm:MLPsmooth_anyarch}.

  Then for each $t\ge0$, $\ff_t$ converges almost surely to $\mathring \ff_t$ computed from the following.
  \begin{enumerate}
    \item (Forward and Backward Propagation) 
    Let $\Theta_t$ be the assignment that assigns $\oplim{W_t}$ to each matrix parameter $W$, $\ket v$ to each vector parameter $v$, and $\mathring c_t$ to each scalar parameter $c$.
      \begin{align}
        \mathring \ff_t &= \Theta_t\{\braket 1 \xx\}  \in \R^{\dout \times \NN}
          & \mathring \cchi_t &= \eps_t(\mathring \ff_t) \in \R^{\dout \times \NN}
          \label{eqn:out_fbprop_mulimit_anyarch}\\
        \ket{\yy_t} &= \Theta_t\{\ket{\yy}\} \in \R^{\NN}
        &\ket{\db \yy_t} &= \Theta_t\{\ket{\db \yy}\}  \in \R^{\dout \times \NN}
        &\text{for any vector $y$ in $\pi$}
        \label{eqn:vector_fbprop_mulimit_anyarch}\\
        \mathring \aa_t &= \Theta_t\{\mathring \aa\} \in \R^\NN
        &\db \mathring \aa_t &= \Theta_t\{\db \mathring \aa\} \in \R^{\dout \times \NN}
        &\text{for any scalar $a$ in $\pi$}
        \label{eqn:scalar_fbprop_mulimit_anyarch}
      \end{align}
      Here $\xx$ is the multi-vector consisting of $x^1, \ldots, x^{\dout}$ evaluated on all $\NN$ inputs, and all kets and limits are calculated in $\db \ppi$ via \cref{defn:ket}.
    \item (Parameter Updates)
      \begin{align}
        \oplim{W_{t+1}} &= \oplim{W_{t}} - \eta \overline{\sum_{h=Wx}\ketdbra{\db \hh_t}{\mathring \cchi_t}{\xx_t}}
        && \text{for every matrix parameter $W$}
        \label{eqn:muP_op_update_general}
        \\
        \ket{v_{t+1}}&=\ket{v_{t}}-\eta \overline{\ket{\db v_{t}} \cdot \mathring{\cchi}_{t}}
        && \text{for every vector parameter $v$}
        \label{eqn:muP_ket_update_general}\\
        \mathring c_{t+1} &= \mathring c_t - \eta \overline{\db \mathring c_t}
        && \text{for every scalar parameter $c$}
        \label{eqn:muP_scalar_update_general}
      \end{align}
      where the tensor operations and summation over $h=Wx$ should be interpreted as in \cref{defn:GenArchTangentOperator}.
    \item (Initialization) $\{\oplim{W_0}: \text{matrix parameter $W$}\}$ is a set of independent initial operators. Additionally,
    \begin{align}
      \ket{v_{0}}&=\Gaus(0, 1)
      && \text{for every vector parameter $v$}\\
      \mathring c_{0} &\sim \Gaus(0, 1)
      && \text{for every scalar parameter $c$}
      \label{eqn:scalar_param_init}
    \end{align}
    all independent from one another.
    \label{item:general_mulimit_initialization}
  \end{enumerate}

\end{thm}
Here, we used the notation $W, v, c$ for matrix, vector, and scalar parameters, in contrast to $y$ and $a$ for vector and scalar generated by the program $\pi$.
The former are exemplified by weights while the latter by (pre)activations.

\begin{exmp}
  In the MLP case with program $\pi$ given in \cref{eqn:MLP}, there are no generated scalars, so we can ignore \cref{eqn:scalar_fbprop_mulimit_anyarch}.
  The generated vectors (in $\db \ppi$) are $\hh^1, \ldots, \hh^L$ and $\xx^1, \ldots, \xx^L$ as well as their error signals $d\hh^l, d\xx^l$.
  So \cref{eqn:vector_fbprop_mulimit_anyarch} reduces to the 2nd and 3rd rows in \cref{thm:mulimit_MLP}(1).
  Finally, the function output is given by averaging $w^{L+1} \odot x^L$, so the $\mathring \ff_t = \Theta_t\{\braket 1 \xx\}$ in \cref{eqn:out_fbprop_mulimit_anyarch} reduces to
  $\mathring \ff_t = \braket{w_t^{L+1}}{\xx^L_t}$ in  \cref{thm:mulimit_MLP}(1).
\end{exmp}
\begin{rem}[Scalar Parameters]
  In contrast to the NT limit (\cref{thm:MemorylessGeneralNTKLimit}), here we allow scalar parameters.
  The sampling of $\mathring c_{0} \sim \Gaus(0, 1)$ in \cref{eqn:scalar_param_init} for every scalar parameter $c$ is just a consequence of the sampling in \cref{defn:generalabcd}.
  But in fact for any deterministic initialization $C\in \R$ of $c$, we can set $\mathring c_0 = C$ in  \cref{eqn:scalar_param_init}, and \cref{thm:mulimit_anyarch} still holds.
  Indeed, for scalar parameters, the most natural initialization is probably either 0 (for additive parameters like bias) or 1 (for multiplicative parameters like layernorm weights).
  When scalar parameters have deterministic initializations, $\mathring \ff_t$ is always deterministic as well;
  however, in \cref{thm:mulimit_anyarch}, as written, $\mathring \ff_t$ is only deterministic conditioned on $\mathring c_0$ for every scalar parameter $c$.
\end{rem}

\begin{rem}[$\mu$P is Most Natural]
  As discussed in \cref{rem:mup_mostnatural}, $\mu$P is in a sense \emph{the most natural parametrization} because its infinite-width limit is just a direct ``ket-translation'' of the finite-width computations, no matter the architecture.
  Compare this with NTP, where the Gaussian process and kernel behaviors can only happen under some relative restrictive conditions, like ``no scalar parameters'' (\cref{thm:MemorylessGeneralNTKLimit} and \cref{rem:ntk_anyarch_assumptions}).
  \end{rem}

\begin{rem}
  In the $\mu$-limit for a general architecture, $\mathring \ff_0$ no longer needs to be 0, unlike \cref{thm:mulimit_MLP}.
  For example, if in an MLP with relu activation, the output is given by the average entry of the last layer activation vector, then obviously $\mathring \ff_0$ will always be positive.
\end{rem}

\subsubsection{Memoryful Nonstationary Case}

\begin{thm}[$\mu$-Limit for General Architecture, Memoryful Nonstationary Updates]\label{thm:mulimit_MLP_general_anyarch}
  If $\QQ$ is memoryless but not stationary, then \cref{thm:mulimit_anyarch} holds if the bars in \cref{{eqn:muP_op_update_general},{eqn:muP_ket_update_general},{eqn:muP_scalar_update_general}} are interpreted as $Q^W_t, Q^v_t,$ and $Q^c_t$.

  If $\QQ$ is not memoryless, then \cref{thm:mulimit_anyarch} holds if \cref{{eqn:muP_op_update_general},{eqn:muP_ket_update_general},{eqn:muP_scalar_update_general}} are replaced with 
  \begin{align*}
    \oplim{W_{t+1}} &= \oplim{W_{t}} - \eta \overline{\sum_{h=Wx}\ketdbra{\db \hh_{\le t}}{\mathring \cchi_{\le t}}{\xx_{\le t}}}
    && \text{for every matrix parameter $W$}
    \\
    \ket{v_{t+1}}&=\ket{v_{t}}-\eta \overline{\ket{\db v_{{\le t}}} \cdot \mathring{\cchi}_{{\le t}}}
    && \text{for every vector parameter $v$}
    \\
    \mathring c_{t+1} &= \mathring c_t - \eta \overline{\db \mathring c_{\le t}}
    && \text{for every scalar parameter $c$}
  \end{align*}
  where the bar notations abbreviate
  \begin{align*}
    \overline{\sum_{h=Wx}\ketdbra{\db \hh_{\le t}}{\mathring \cchi_{\le t}}{\xx_{\le t}}}
    &=
      Q^W_t\left(\sum_{h=Wx} \ketdbra{\db \hh_{0}}{\mathring \cchi_{0}}{\xx_{0}}, \ldots, \sum_{h=Wx} \ketdbra{\db \hh_{ t}}{\mathring \cchi_{ t}}{\xx_{ t}}
        \right)
    \\
    \overline{\ket{\db v_{{\le t}}} \cdot \mathring{\cchi}_{{\le t}}}
    &=
      Q^v_t\left(
        \ket{\db v_{{0}}} \cdot \mathring{\cchi}_{{0}},
        \ldots,
        \ket{\db v_{{ t}}} \cdot \mathring{\cchi}_{{ t}}
      \right)
    \\
    \overline{\db \mathring c_{\le t}}
    &=
      Q^c_t\left(
        \db \mathring c_{0},
        \ldots,
        \db \mathring c_{ t}
      \right)
    .
  \end{align*}
\end{thm}

\section{Extensions}

\label{sec:extensions}

\subsection{Weight Decay}

Consider the update equation (\cref{eqn:abcdupdate})
with \emph{decoupled} weight decay $\lambda$:%
\footnote{We will assume all layers have the same $\lambda$, for simplicity. The generalization to layer-specific $\lambda$ is straightforward.}
\begin{equation}
  w^l_{t} = (1 - \lambda) w^l_{t-1} -\eta n^{-c_l} Q_{t}^l (n^{d_l} g_{0},\ldots, n^{d_l} g_{t}). \label{eq:update_eqn_weight_decay}
\end{equation}
This decoupled weight decay $\lambda$ is equivalent to a traditional weight decay value $\nicefrac \lambda {\eta n^{-c_l}}$.
It's obvious that $\lambda$ should be invariant to the width $n$: if $\lambda \to \infty$ with $n$, then $1 - \lambda <0$ eventually; if $\lambda \to 0$ with $n$, then weight decay has no effect in the limit.

The theory in this section covers AdamW \citep{Loshchilov2017DecoupledWD}, SGD with weight decay, and so on.

\subsubsection{Maximal Update Limit}
\cref{thm:mulimit_MLP} holds with \cref{eqn:muP_in_update,eqn:muP_op_update,eqn:muP_out_update} replaced with
\begin{align}
  \ket{w^1_{t+1}}&=(1 - \lambda)\ket{w^1_{t}}-\eta \overline{\ket{d\hh_{t}^1}_{\mathring{\cchi}_{t}}\xxi^{\trsp}}
  \label{eqn:muP_in_update_wd}\\
  \oplim{W^{l}_{t+1}} &= (1 - \lambda) \oplim{W^{l}_{t}} - \eta \overline{\ketdbra{d\hh_t^{l}}{\mathring \cchi_t}{\xx_t^{l-1}}},\ \forall l \in [2,L]
    \label{eqn:muP_op_update_wd}
  \\
  \ket{w_{t+1}^{L+1}}&=(1 - \lambda) \ket{w_{t}^{L+1}}-\eta \overline{\ket{\xx_{t}^{L}} \cdot \mathring{\cchi}_{t}}.
    \label{eqn:muP_out_update_wd}
\end{align}

All other $\mu$-limit theorems (\cref{thm:muLimit_1LP}, \ref{thm:muLimit_1LP_general}, \ref{thm:mulimit_MLP_general}, and \ref{thm:mulimit_MLP_general_anyarch}) hold with similar modifications.
In particular, if we interpret the bars in \cref{eqn:muP_in_update_wd,eqn:muP_op_update_general,eqn:muP_out_update_wd} as \cref{eqn:Adam} and replace the $\bullet_t$ subscripts under the bars to $\bullet_{\le t}$ (c.f.\ the notation in \cref{thm:mulimit_MLP_general}), then we get the equations for AdamW's $\mu$-limit.

\subsubsection{Neural Tangent}

Fix decoupled weight decay $\lambda$.
\begin{defn}
  For each $l=1,\ldots,L$ and $s = 0, 1, \ldots$, the ket $\ket{\hh^{l}_s}\in\R^{\NN}$ is constructed
  as a mean-zero Gaussian vector with covariance matrix $(1-\lambda)^{2s} \braket{\xx^{l-1}_s}{\xx^{l-1}_s}$
  when $l>1$ or $\xxi^{\trsp}\xxi$ when $l=1$. Simultaneously, $\xx^{l}_s\defeq\phi(\ket{\hh^{l}_s})\in\R^{\NN}$
  for $l=1,\ldots,L$.
  Furthermore, for any $s,t \ge 0$, $\ket{\hh^l_s}$ and $\ket{\hh^l_t}$ are jointly Gaussian with covariance $(1-\lambda)^{s+t}\braket{\xx^{l-1}_s}{\xx^{l-1}_t}$.
\end{defn}

\begin{defn}
  For each $l=L,\ldots,1$ and $s = 0, 1, \ldots$, the ket $\ket{d\xx^{l}_s}\in\R^{\NN}$ is independent
  from $\{\ket{\xx^{l}_t},\ket{\hh^{l}_t}\}_{l=1}^{L}$ for any $t$ and is a mean-zero
  Gaussian vector with covariance matrix $(1-\lambda)^{2s}\braket{d\hh^{l+1}_s}{d\hh^{l+1}_s}$
  when $l<L$ or the all-1s matrix when $l=L$. Simultaneously, $\ket{d\hh^{l}_s}\defeq\phi'(\ket{\hh^{l}_s})\odot\ket{d\xx^{l}_s}\in\R^{\NN}$
  for all $l$.
  Furthermore, for any $s,t \ge 0$, $\ket{d\xx^l_s}$ and $\ket{d\xx^l_t}$ are jointly Gaussian with covariance $(1-\lambda)^{s+t}\braket{d\hh^{l+1}_s}{d\hh^{l+1}_t}$.
\end{defn}

\begin{thm}\label{thm:NTlimitWD}
  With decoupled weight decay $\lambda$, 
  \cref{thm:memorylessNTK} holds with the update equation \cref{eqn:NTK_update_memoryless_stationary} replaced by
\begin{align*}
  \mathring{\ff}_{t+1} - \mathring{\ff}_{0}&=
  -\eta 
  \sum_{s=0}^t (1 - \lambda)^{t-s} \Diag\sum_{l=1}^{L+1}
  \la{d\hh^{l}_{t+1}}
  \Qketdbra{d\hh^{l}_{s}}{\cchi_s}{\xx^{l-1}_{s}}
  {\xx^{l-1}_{t+1}}\ra.
\end{align*}
For memoryless nonstationary updates, interpret the bar as $Q_t^l$.
For memoryful nonstationary updates, replace the nonlinear outer product with 
$\Qketdbra{d\hh^{l}_{\le s}}{\cchi_{\le s}}{\xx^{l-1}_{\le s}}$ (c.f.\ \cref{thm:mulimit_MLP_general} for the notation).
\end{thm}

Note that when $\lambda = 0$, $\hh^l_s$ is invariant to $s$ as are $\xx^l_s, d \hh^l_s, d\xx^l_s$, and this exponentially weighted sum just reduces to a simple sum and thus to 
\cref{eqn:NTK_update_memoryless_stationary}.

In the memoryful case, interpreting the bar in \cref{thm:NTlimitWD} as \cref{eqn:Adam} yields the NT limit of AdamW.
The limit theorems (\cref{thm:MemorylessGeneralNTKLimit} and \cref{thm:nt_memoryful_anyarch}) for general architecture also hold with analogous modifications.

\subsection{Update Clipping and Normalization}

We now add update clipping or normalization to \cref{eq:update_eqn_weight_decay}:%
\footnote{Traditionally, \emph{gradient clipping} for SGD clips the norm of the entire gradient (for the whole network). Naively, one would just do this before applying the update function $Q$. However, in Adam, for example, this would be meaningless because of Adam's normalization. For general $Q$, the $d_l$ terms can be trivially adjusted according to how the global gradient norm scales. So this notion of clipping or normalization (before $Q$) is not very interesting.}
\begin{equation}
  w^l_{t} = (1 - \lambda) w^l_{t-1} -\eta n^{-c_l} \cdot \nu^{-1} Q_{t}^l(n^{d_l} g_{0},\ldots,n^{d_l} g_{t}), \label{eq:update_eqn_grad_clip}
\end{equation}
where 
\begin{itemize}
  \item if we are doing update normalization, then we set $\nu \gets \|Q_{t}^l(n^{d_l} g_{0},\ldots,n^{d_l} g_{t})\|$;
  \item if we are doing update clipping, then calculate additionally $\nu \gets \min(\nu, \theta^l)$ where $\theta^l$ is a threshold hyperparameter for layer $l$.
\end{itemize}

Some recent works \citep{you_large_2017,shazeer_adafactor_2018,you_large_2020,carbonnelle_layer_2019,bernstein_distance_2021,liu_learning_2021} normalize the update by the current weight (Frobenius) norm instead of by the update norm.
We explore its implications in \cref{{sec:normbycurrentweightnorm}}.
For now we focus on the formulation above.

The key questions we explore here are:
\begin{enumerate}
  \item How should the abcd values adjust to update clipping and normalization?
  \item What do the neural tangent and maximal update limits look like?
  \item How should the threshold hyperparameter $\theta^l$ scale with width $n$?
\end{enumerate}

The key intuition is as follows:
In all ``reasonable'' parametrizations (more precisely, faithful ones (\cref{defn:faithful})), $Q_{t}^l(n^{d_l} g_{0},\ldots,n^{d_l} g_{t})$ is entrywise $\Theta(1)$.
Thus, its norm $\nu$ scales like $\sqrt{\#\text{entries}}$, as can be verified via the Master Theorem (\cref{thm:MasterTheorem}).
For example, for the input and output weights of the MLP (\cref{eqn:MLP}), this is $\sqrt{n}$, while for the hidden weights, this is $n$.
Therefore, for update clipping, the threshold should be
\begin{equation}
  \theta^l = \theta^l_0 n^{e_l}\quad
  \text{where}\quad
  e_l =
  \begin{cases}
    \nicefrac 1 2  &\text{if $l = 1, L+1$}\\
    1       &\text{otherwise,}
  \end{cases}
  \label{eqn:gradclipscale}
\end{equation}
for some tunable hyperparameter $\theta^l_0$ independent of width.
Otherwise, either the clipping has no effect (if $\theta^l$ is larger than this) or $\nu$ after update clipping is always equal to the threshold $\theta^l$ (if $\theta^l$ is smaller than this).

At the same time, $c_l$ should be ${e_l}$ smaller than if there is no update normalization or clipping (i.e., the learning rate should be larger).
Thus, for example, 
\begin{defn}\label{defn:NTPgradclip}
  The NTP with update normalization or clipping is
\begin{center}
  \begin{tabular}{cccc}
    \toprule
    $l$ & $[2,L]$ & $1$  & $L+1$\tabularnewline
    \midrule
    $a_l$ & $\nicefrac{1}{2}$ & 0  & $\nicefrac{1}{2}$\tabularnewline
    $b_l$ & $0$  & $0$ & $0$\tabularnewline
    $c_l$ & $0$ & $0$  & $0$\tabularnewline
    $d_l$ & $1$ & $\nicefrac{1}{2}$  & $\nicefrac{1}{2}$\tabularnewline
    \bottomrule
  \end{tabular}
\end{center}
where update clipping thresholds $\theta^l$ scale as in \cref{eqn:gradclipscale}.
\end{defn}

\begin{defn}\label{defn:muPgradclip}
  The $\mu$P with update normalization or clipping is
  \begin{center}
    \begin{tabular}{cccc}
      \toprule
      $l$ & $[2,L]$ & $1$  & $L+1$\tabularnewline
      \midrule
      $a_l$ & $0$ & $0$  & $1$\tabularnewline
      $b_l$ & $\nicefrac{1}{2}$  & $0$ & $0$\tabularnewline
      $c_l$ & $0$ & $-\nicefrac{1}{2}$ & $-\nicefrac{1}{2}$\tabularnewline
      $d_l$ & $1$ & $1$  & $1$\tabularnewline
      \bottomrule
    \end{tabular}
  \end{center}
  where update clipping thresholds $\theta^l$ scale as in \cref{eqn:gradclipscale}.
\end{defn}

Likewise, the classification of parametrizations in \cref{sec:classifyabcd} hold for update normalization or clipping if we replace all mentions of $c_l$ in \cref{sec:classifyabcd} with $c_l + e_l$.
For example, 
\begin{defn}%
  With update clipping or normalization, we redefine
  \[
  r_{l}\defeq\begin{cases}
  c_{l}+e_l+a_{l}-1 & \text{if \ensuremath{l>1}}\\
  c_{l}+e_l+a_{l} & \text{if \ensuremath{l=1}}
  \end{cases},\quad
  r_{\le l} \defeq \min_{m=1}^l r_l,\quad
  r \defeq r_{\le L}
  \]
\end{defn}
Then $\xx^L_t - \xx^L_0 = \Theta(n^{-r})$ still and it remains that $r = \nicefrac{1}{2}$ for NTP (\cref{defn:NTPgradclip}) and $r = 0$ for $\mu$P (\cref{defn:muPgradclip}) with the new definitions above.

\paragraph*{General Architectures}

For general architectures, there is one $e_w$ for every parameter tensor $w$, and should be set to $0, \nicefrac 1 2,$ and $1$ respectively for scalar, vector, and matrix parameters.
Then the obvious generalization of the above discussion holds.

\paragraph*{Maximal Update Limit}
\newcommand{\doverline}[1]{\tilde{\overline{#1}}}

\begin{defn}
  For any random vector $\XX$, let $|\XX| \defeq \sqrt{\EV \|\XX\|^2}$.
  Then let $\tilde \XX \defeq \XX / |\XX|$.
\end{defn}

Then \cref{thm:mulimit_MLP} holds with \cref{eqn:muP_in_update,eqn:muP_op_update,eqn:muP_out_update} replaced with (essentially, compared to \cref{eqn:muP_in_update_wd,eqn:muP_op_update_wd,eqn:muP_out_update_wd} we added a tilde $\tilde{\ }$ to all bars.)%
\footnote{Unfortunately, LaTeX breaks $\mathring \chi_t$ in the subscripts, so here we just write $\chi_t$ instead.}
\begin{align}
  \ket{w^1_{t+1}}&=(1 - \lambda)\ket{w^1_{t}}-\eta \doverline{\ket{d\hh_{t}^1}_{{\cchi}_{t}}\xxi^{\trsp}}
  \label{eqn:muP_in_update_gclip}
  \\
  \oplim{W^{l}_{t+1}} &= (1 - \lambda) \oplim{W^{l}_{t}} - \eta \doverline{\ketdbra{d\hh_t^{l}}{\cchi_t}{\xx_t^{l-1}}},\ \forall l \in [2,L]
    \label{eqn:muP_op_update_gclip}
  \\
  \ket{w_{t+1}^{L+1}}&=(1 - \lambda) \ket{w_{t}^{L+1}}-\eta \doverline{\ket{\xx_{t}^{L}} \cdot {\cchi}_{t}}.
    \label{eqn:muP_out_update_gclip}
\end{align}
Here, we think of $\ketdbra{\hh_t^l}{\mathring {\cchi}_t}{\xx_t^{l-1}}$ as a scalar random variable with the same distribution as $\sum_{i=1}^\NN \mathring \chi_t(\xi_i) \ket{h^l_t(\xi_i)} \ket{x^{l-1}_t(\xi_i)}^{\bx 1} \in \R$.

If doing update clipping, then interpret $\tilde \XX \defeq \XX / \min(\theta_0, |\XX|)$ for $\theta_0$ being the $\theta^l_0$ (\cref{eqn:gradclipscale}) for appropriate layer $l$.

Similar statements hold for other $\mu$-limit theorems, including those in general architectures.

\paragraph*{Neural Tangent Limit}

\cref{thm:NTlimitWD} holds with a tilde added to all bars, where the tilde should be interpreted according to $\tilde \XX \defeq \XX / |\XX|$ if doing update normalization or $\tilde \XX \defeq \XX / \min(\theta_0, |\XX|)$ if doing update clipping.
Similar statements hold for other NT limit theorems, including those in general architectures.

\subsubsection{Normalization by Current Weight Norm}
\label{sec:normbycurrentweightnorm}

\newcommand{\half}{\nicefrac{1}{2}}

Now suppose in \cref{eq:update_eqn_grad_clip}, we normalizd by the current weight norm instead, i.e.,
\[\nu \gets \|w^l_{t-1}\|_F.\]

For simpicity, assume the weight decay $\lambda$ is zero; our conclusion will turn out to hold even when this is not the case.

Let's consider what happens to the hidden weights.
Consider $W^l$ for $l \in [2, L]$. The Frobenius norm of its initialization $W^l_0 \in \R^{n \times n}$ (a random Gaussian matrix with $W_{\alpha\beta} \sim \Gaus(0, 1/n)$) scales like $\Theta(\sqrt{n})$ (in fact, asymptotically $\sqrt{n} + O(1)$).
But, in a stable parametrization (\cref{defn:stability}), the Frobenius norm of $\delta W^l_t = W^l_{t} - W^l_{t-1}$, a nonlinear outer product with $O(1/n)$ entries, scales like $O(1)$.
This means that
\[\frac{\|\delta W^l_t\|_F}{\|W^l_0\|_F} = \frac{\|\delta w^l_t\|_F}{\|w^l_0\|_F} = O(1/\sqrt n).\]
Therefore, 
\[\frac{\|W^l_0 + \delta W^l_1 + \cdots + \delta W^l_t\|_F}{\|W^l_0\|_F} = \frac{\|W^l_t\|_F}{\|W^l_0\|_F} = \frac{\|w^l_t\|_F}{\|w^l_0\|_F} \to 1\]
for any $t$ as $n \to \infty$.%
\footnote{The fact that $\|w_t\|_F$ is essentially $\|w_0\|_F$ really implies that this is \emph{not the right quantity} to normalize or clip the update by.
For example, other norms (like spectral norm) does not have this property and could be better suited (if we ignore the computational issues for the moment).
As another example, if Frobenius norm is desired, we should perhaps want to normalize by $\|\Delta w_t\|_F$ instead of $\|w_t\|_F$.}
So $\nu = \|w^l_{t-1}\|_F = \Theta(\|w^l_{0}\|_F) = \Theta(n^{1-b_l})$ (in fact, $\nu / n^{-b_l} \to 1$).

In contrast, for the input layer weight $W = W^1$, we have $\Delta W_t = O(W_0)$ (entrywise and in Frobenius norm) in a stable parametrization.
If the parametrization is furthermore faithful (\cref{defn:faithful}), then for $W = W^{L+1}$, this is true as well (c.f.\ \cref{rem:minormismatch}).
So $\|w^1_t\|_F / \|w^1_0\|_F$ and $\|w^{L+1}_t\|_F / \|w^{L+1}_0\|_F$ are both $\Theta(1)$ (but $\not\to 1$ generally) and $\nu = \Theta(n^{\half-b_l})$ in both cases.

Therefore, if we subtract $b_l$ from $e_l$ in \cref{eqn:gradclipscale}, then all discussion above applies.
For example, the parametrizations for current weight norm can be obtained by adding the $b_l$ row to the $c_l$ row.
\begin{defn}\label{defn:NTPgradclip2}
  
  The NTP with update normalization or clipping with current weight norm is the same as \cref{defn:NTPgradclip}.
\end{defn}
\begin{defn}\label{defn:muPgradclip2}
  The $\mu$P with update normalization or clipping with current weight norm is
  \begin{center}
    \begin{tabular}{cccc}
      \toprule
      $l$ & $[2,L]$ & $1$  & $L+1$\tabularnewline
      \midrule
      $a_l$ & $0$ & $0$  & $1$\tabularnewline
      $b_l$ & $\nicefrac{1}{2}$  & $0$ & $0$\tabularnewline
      $c_l$ & $\nicefrac{1}{2}$ & $-\nicefrac{1}{2}$ & $-\nicefrac{1}{2}$ \tabularnewline
      $d_l$ & $1$ & $1$  & $1$\tabularnewline
      \bottomrule
    \end{tabular}
  \end{center}
  where the clipping thresholds $\theta^l$ scale as $n^{\half}$.
\end{defn}

The limit equations for $\mu$P in this case is just \cref{eqn:muP_in_update_gclip,eqn:muP_op_update_gclip,eqn:muP_out_update_gclip} but with the tilde interpreted as dividing by $\sqrt{\braket{w^1_{t}}{w^1_{t}}}$, $1$, and $\sqrt{\braket{w^{L+1}_{t}}{w^{L+1}_{t}}}$ respectively.
Similar modifications apply to the NT limits.

Likewise, the classification of parametrizations hold if we replace each $c_l$ with $c_l + e_l- b_l$.

We have assumed that the weight decay $\lambda$ is 0 at the beginning.
In the general case, define $\delta W_t = W_t - (1 - \lambda) W_{t-1}$ instead, but the same calculations will yield the same conclusions.

\subsection{Second Moment Factoring ala Adafactor}

The key observation here is that, even though the factored update function $Q$ is not entrywise anymore, the update itself is still a nonlinear outer product of $d\hh_{\le t}$ and $d\xx_{\le t}$ (with some scalars variables inserted in an appropriate \nexort{} program), %
All of our theory in fact holds as-is for this more general kind of update function, since they all factor through the \refOuterNonlin{} instruction of \nexort{}.

\subsection{Future Optimizers}

As technology advances, we may get better optimizers of very different forms than we discussed here. 
But we can quite reasonably expect them to continue to efficiently utilize GPUs, i.e., perform large matrix-multiplies.
As such, these optimizers can always be analyzed in a Tensor Program.
So the lesson is: as long as one understands Tensor Programs, one can always derive the correct way of scaling an optimizer (with width).

\section{Proof Sketch}\label{sec:proofsketch}

\nexort{} programs equipped with their Master Theorem (\cref{thm:MasterTheorem}) provide the main tool to all of our rigorous results.
For example, to prove the neural tangent or $\mu$ limit equations, one can:
express the optimization dynamics using a \nexort{} program, mechanically compute the kets according to \cref{defn:ket}, and apply \cref{thm:MasterTheorem} to compute the limit (see full proofs in \cref{sec:muLimitProof,sec:NTKProof}).
What remains is to prove \cref{thm:MasterTheorem} using a strategy which we now outline. 

In a program, all vectors can be collected into an $n \times M$ matrix $V$ where $n$ is the dimension of each vector, and $M$ is the total number of vectors.
The Master Theorem can be interpreted as saying that each row of $V$ (i.e., the slice for each $\alpha \in [n]$) is roughly an iid sample from some distribution $\mathcal D$ on $\R^M$ (the distribution of some multivariate ket computed from \cref{defn:ket}).
Specifically, \cref{thm:MasterTheorem} and all previous versions of the Master Theorem formalize this by saying: this matrix $V$ of vectors looks similar to a matrix $V'$ of iid samples from $\mathcal D$, as measured by applying arbitrary pseudo-Lipschitz ``test function $\psi $'' to both sides and taking averages (the scalars in \cref{thm:MasterTheorem} are exactly of this form by \refAvg{}).

\paragraph{Core Insight}
Our core insight here is that $V$ is in fact similar to $V'$ in a stronger sense without needing to refer to any test function $\psi $: There is a ``small'' matrix $\delta V$ of the same shape as $V$ such that $V - \delta V$ is distributed exactly the same as $V'$.
In general, if this happens, then we say $V$ is \emph{distributionally equivalent, or dequivalent, to} $V'$ (\cref{def:distequiv1}).
The definition of ``small'' roughly means that $\delta V$ is $\tilde O(1)$ in $L^2$ norm, so that dequivalence is morally just ``close in Wasserstein distance.''
This ``Wasserstein'' notion is a stronger sense of closeness of $V$ to $V'$ than the ``test function'' notion described in the previous paragraph because the test function $\psi$ is (by assumption) smooth enough that $\delta V$ contributes a vanishing amount to the average (\cref{prop:equiv_pLip_tensor}).

To prove this core insight, there are two parts.

\paragraph{Part 1:}
We show that, in any \netsort{} program (i.e., a program with no scalar variables and no \textsc{Tensor} operation), $V$ is dequivalent to $V'$.
This can be done by re-analyzing the proof of the \netsort{} Master Theorem in \citep{yang3} in a fairly straightforward way.

\paragraph{Part 2:}
For any \nexort{} program $\pi$ (the subject of our work here), we construct a \emph{parallel \netsort{} program} (\cref{defn:integratedprogram}) and show, by induction, that the vectors of the two programs are dequivalent (i.e., distributed exactly the same after subtracting ``small'' vectors).
This parallel program essentially replaces 1) all scalar variables in the original program by their deterministic limits, as computed in \cref{defn:ket}(\texttt{Avg}), and 2) all \refOuterNonlin{} operations by \texttt{Nonlin} operations, as computed in \cref{defn:ket}(\texttt{OuterNonlin}), so that the corresponding vectors of these programs share their kets.
Then, by Part 1 above, we will have proven $V$ dequivalent to $V'$ for the original \nexort{} program $\pi$.

In this induction, we need to prove and use a lemma that, in the simplest case as an illustration, says the following:
For any pseudo-Lipschitz function $\psi : \R^k \to \R$ and random vector $x \in \R^n$ with iid standard Gaussian entries, the following two tensors $T$ and $T'$ are equivalent:
1) the tensor $T$ with entries $T_{\beta_1 \ldots \beta_k} = \psi (x_{\beta_1},\ldots,x_{\beta_k})$, and
2) the tensor $T'$ with entries $T'_{\beta_1 \ldots \beta_k} = \psi (x^1_{\beta_1},\ldots,x^k_{\beta_k})$
where $x^1,\ldots,x^k$ are iid copies of $x$.
The proof of this lemma interestingly requires Baranyai's theorem (\cref{thm:bara}), a classical theorem from the theory of combinatorial design about perfect matching in complete hypergraphs.

\chapter{Proofs of Infinite-Width Limits}

Here we prove the classificatioon of abcd-parametrizations as well as the NT and $\mu$-limit equations, assuming the \nexort{} Master Theorem (\cref{thm:MasterTheorem}).

\section{Proof of Classification of abcd-Parametrizations}

Here we seek to prove the main claims of \cref{sec:classifyabcd}: \cref{thm:abcdclassification}, \ref{thm:nontrivial}, and \ref{thm:stablefaithful}.

\paragraph*{Overview}
The proof is a straightforward modification of the proof of \cite[H.13]{yang4}.
The main subtle point is how $r=0$ implies \cref{enum:r=0} of \cref{thm:abcdclassification} (\cref{item:r=0} in the proof outline below).
The reasoning for this in \cite{yang4} is very specific to SGD and the activation functions tanh and (smoothed) relu.
Here, we adopt a different logic, given in \cref{sec:r=0} and \cref{thm:r=0impliesFL}, based on the ``asymptotic relu-ness'' of the nonlinearity.

On the other hand, the implication of \cref{enum:r>0} in \cref{thm:abcdclassification} by $r>0$ is similar to \cite{yang4}, adapted naturally to general $Q^l_t$ functions, so we omit the full proof here.
However, we provide the derivation of the NT limit in \cref{sec:neuraltangent}, which gives an instructive example of the proof for general operator regime parametrizations.

\paragraph*{Proof Outline}
In the sequel, we provide full details of \cref{item:trainTP,{item:r=0}} of the outline below.
All other steps in this outline are straightforward adaptations of \cite{yang4} or easy in a self-contained way.

\begin{enumerate}
  \item The characterization of stability at initialization (\cref{lemma:stabilityconditionsinit}) is already proven in \cite[Thm H.19]{yang4}. So assume stability at initialization from here on.
  \item Some simple calculations then verifies the characterization of faithfulness at initialization (\cref{lemma:faithfulconditionsinit}). So assume faithfulness at initialization from here on.%
  \footnote{Note this effectively fixes $d_l$ given $a_l, b_l$, so that the abcd-parametrization now has the same degrees of freedom as an abc-parametrization. This is a sanity check for why we can adapt most of the arguments from \cite{yang4}.}
  \item If $b_{L+1} > c_{L+1}$, then it's easy to see that we lose faithfulness after 1 step of update (because the input to $Q^l_t$ is $\omega(1)$).
  \item Otherwise $b_{L+1} \le c_{L+1}$. \label{b_{L+1}<=c_{L+1}}
  \begin{enumerate}
    \item First, assume that $r_l \ge 0$ for all $l \in [L+1]$ and $a_{L+1} + b_{L+1} + r \ge 1$. (i.e., we assume \cref{eqn:stable_faithful})
    \begin{enumerate}
      \item Then we can build a \nexort{} program computing the network evolution (which we do in \cref{sec:Program-Setup,sec:ProgramConstruction}). The Master Theorem yields its infinite-width limit, described in \cref{sec:infwidthlimit}, which will be obviously stable and faithful (\cref{lemma:sufficiency4stabfaith}).%
      \label{item:trainTP}
      \item If $r > 0$, then, by adapting the corresponding reasoning in \cite{yang4}, it's easy to derive all properties in \cref{thm:abcdclassification}(2), as well as the validity of \cref{thm:nontrivial} (assuming $r>0$).
      We derive the NT limit in \cref{sec:NTKProof} as an example.
      \item \label{item:r=0} If $r=0$, then the reasoning in \cite{yang4} cannot be straightforwardly adapted. In \cref{sec:r=0}, using a different method based on ``asymptotic relu-ness'' of the nonlinearity $\phi$ (\cref{assm:relulike_phi_Q_preserves_sign}), we prove all properties in \cref{thm:abcdclassification}(1,3,4), the validity of \cref{thm:nontrivial} (assuming $r=0$), and the fact that we are not in operator regime.
    \end{enumerate}
    \item Suppose $r_l <0$ for some $l \in [L]$, and $l^*$ is the smallest such $l$. Then the infinite-width limit derived above up to time 1, layer $l^*$ shows that $\Delta \xx^l_1 = \omega(1)$ for some choice of $\xxi$ and training routine, so we lose stability.
    \item Otherwise, if $a_{L+1} + b_{L+1} + r < 1$ or $r_{L+1} < 0$, then the infinite-width limit derived above up to time 1 shows that $\Delta \ff_1 = \omega(1)$, so we lose stability as well.
    \item Combining all caseworks, we 1) derive the nontriviality condition \cref{thm:nontrivial}; 2) see \cref{eqn:stable_faithful} is necessary for faithfulness and stability as well, proving \cref{thm:stablefaithful}; and 3) prove \cref{thm:abcdclassification}.
  \end{enumerate}
\end{enumerate}

\subsection{Program Setup}

\label{sec:Program-Setup}

In \cref{sec:Program-Setup,sec:ProgramConstruction,sec:infwidthlimit}, we implement \cref{item:trainTP}: we construct the Tensor Program that encodes the training of an $L$-hidden layer MLP under an abcd-parametrization satisfying \cref{eq:actlogitinit,eqn:faithful_init,{eqn:stable_faithful}} and take its infinite-width limit.
For the most part, they are straightforward adaptations of \cite[Sec.\ H.3-H.5]{yang4}; nevertheless, we show full details to demonstrate the usage of the new \nexort{} language and its Master Theorem (\cref{thm:MasterTheorem}).

In this section, we first describe the initial matrices, vectors, and scalars of the program, along with necessary notations.

For ease of presentation, we assume the input dimension $d=1$.
The general $d$ case is a trivial adaptation.

\subparagraph*{Initial Matrices, Vectors, Scalars}

\begin{enumerate}
\item Initial matrices: $W_{0}^{2},\ldots,W_{0}^{L}$, sampled like $(W_{0}^{l})_{\alpha\beta}\sim\Gaus(0,1/n)$.
\item Initial vectors: input layer matrix $W_{0}^{1}\in\R^{n\times1}$ and \emph{normalized }output layer matrix $\widehat{W}_{0}^{L+1}\defeq W_{0}^{L+1}n^{a_{L+1}+b_{L+1}}\in\R^{n \times 1}$, sampled like $(W_{0}^{1})_{\alpha},(\widehat{W}_{0}^{L+1})_{\alpha}\sim\Gaus(0,1)$.
\item Initial scalars: We define the following scalars (where we explain the intuition in parenthesis). The reader can skip this part on a first read but come back when referred to.
\begin{enumerate}
\item ($n$ times the scale of coordinates of $\Delta W_{t}^{l}$) For $l\ge2$, define 
\[
\theta_{W^{l}}\defeq n^{-r_l} = n^{1-c_l-a_l}
\]
\item (scale of coordinates of $\Delta W_{t}^{1}$ and $\Delta h_{t}^{1}$) Define
\[
\theta_{1}=\theta_{W^{1}}\defeq n^{-r_1} = n^{-c_1-a_1}
\]
\item (scale of coordinates of $\Delta W_{t}^{L+1}$)
\[
\theta_{L+1}=\theta_{W^{L+1}}\defeq n^{-a_{L+1}-c_{L+1}}
\]
\item (scale of $\Delta h_{t}^{l}$ and $\Delta x_{t}^{l}$) For $l\in[L]$, define
\begin{align*}
\theta_{h^{l}}=\theta_{x^{l}}=\theta_{l}&\defeq\max_{m\le l}\theta_{W^{m}}=\max(\theta_{W^{l}},\theta_{l-1})
= n^{-r_{\le l}}\numberthis\label{eq:thetal}
\end{align*}
Note that $\theta_{L}=n^{-r}$ with $r$ defined in \cref{def:r}.
\item (scale of $W_{t}^{L+1}$)
\begin{equation}
  \theta_{f}\defeq n^{-(a_{L+1}+b_{L+1})}  \label{eqn:thetaf}
\end{equation}

\item (convenience scalars)
\begin{align*}
\theta_{x^{l-1}/h^{l}} & =\theta_{x^{l-1}}/\theta_{h^{l}} = n^{r_{\le l} - r_{\le l-1} }\\
\theta_{W^{l}/h^{l}} & =\theta_{W^{l}}/\theta_{h^{l}} = n^{r_{\le l} -r_l}\\
\theta_{W^{l}x^{l-1}/h^{l}} & =\theta_{W^{l}}\theta_{x^{l-1}}/\theta_{h^{l}} = n^{r_{\le l} -r_l - r_{\le l-1}}\\
\theta_{L+1/f} & =\theta_{L+1}/\theta_{f} = n^{b_{L+1} - c_{L+1}}\\
\theta_{L+1}' & =n\theta_{L+1}=n^{1-a_{L+1}-c_{L+1}}\\
\theta_{Lf}' & =n\theta_{L}\theta_{f}=n^{1-(r+a_{L+1}+b_{L+1})}
\end{align*}
\item Depending on the the value of $a_{L+1}+b_{L+1}$, we will also construct the values of $f$ at initialization as initial scalars. See \cref{subsec:First-Forward-Pass} for an explanation.
\end{enumerate}
\end{enumerate}

\cref{eqn:stable_faithful} implies all of these $\theta$s either converge to 0 or stay constant at 1. This means that, assuming appropriate regularity conditions on the nonlinearities and rank stability, we can apply the Master Theorem (if $\theta$ blows up to $\infty$ then we can't do that).

\subparagraph*{Notations}
\newcommand{\fOuterNonlin}{\hyperref[eqn:outernonlin_order2]{\ensuremath{\texttt{Nonlin}^2}}}
\newcommand{\fNonlin}{\hyperref[eqn:outernonlin_order1]{\ensuremath{\texttt{Nonlin}^1}}}

We use $:=$ to more clearly denote assignment happening in the program, as opposed to mathematical equality. To clearly demonstrate the application of \refOuterNonlin{}, we will rewrite expressions in the form
\begin{equation*}
  y = \fOuterNonlin{}(\xx; \zz ; \cc)
\end{equation*}
for vector $y \in \R^n$, multi-vectors $\xx \in \R^{n \times k}, \zz \in \R^{n \times j}$, and multi-scalar $\cc \in \R^{l}$
if there exists a function $\psi: \R^{k + j +l} \to \R$ such that
\begin{equation}
  y_{\alpha} = \frac{1}{n}\sum_{\beta=1}^{n}\psi(\xx_{\alpha};\zz_{\beta};\cc).
  \label{eqn:outernonlin_order2}
\end{equation}
This is the order-2 form of \cref{eqn:outernonlin}.
We also write 
\begin{equation*}
  y = \fNonlin{}(\xx; \cc)
\end{equation*}
if there exists a function $\psi: \R^{k+l} \to \R$ such that 
\begin{equation}
  y_{\alpha} = \psi(\xx_{\alpha};\cc).
  \label{eqn:outernonlin_order1}
\end{equation}
This is the order-1 form of \cref{eqn:outernonlin}.
All usages of \refOuterNonlin{} in the program below will be through \fNonlin{} (majority of cases) and \fOuterNonlin{} (only when weight updates are involved).
This program will not use \refOuterNonlin{} order higher than 2.

\subparagraph*{Preview of Names for Vectors}

In the program, for each $\zz\in\{\xx^{l},\hh^{l}\}_{l}$, we will construct vectors $\del \zz_{t}$ to mathematically represent $\theta_{z}^{-1}(\zz_{t}-\zz_{t-1})$ (intuition: change in $z$ scaled to have $\Theta(1)$ coordinates). Similarly, for $w\in\{W^{L+1},W^{1}\}$, we will construct $\del w_{t}$ to mathematically represent $\theta_{w}^{-1}(w_{t}-w_{t-1})$ (intuition: change in $w$ scaled to have $\Theta(1)$ coordinates). Then, mathematically, $\zz_{t}=\zz_{t-1}+\theta_{z}\del \zz_{t},w_{t}=w_{t-1}+\theta_{w}\del w_{t}$.

We will also construct $d\zz$ to mathematically represent $\theta_{f}^{-1}\nabla_{\zz}\ff$ (intuition: gradient $\nabla_{\zz}\ff$ scaled to have $\Theta(1)$ coordinates). For weight changes, we have the following identity
\begin{align*}
  W_{t}^{l}-W_{t-1}^{l}
  &=-\eta \theta_{W^l} \frac 1 n Q^l_t(n^{d_l-a_l} \cchi_{t-1} \xx_{t-1}^{l-1\trsp})\\
  &=-\eta \theta_{W^l} \frac 1 n Q^l_t(\cchi_{t-1} \xx_{t-1}^{l-1\trsp})
  \numberthis \label{eq:delWL+1}
  \end{align*}
for $l=L+1$,
\begin{align*}
W_{t}^{l}-W_{t-1}^{l}
&=-\eta \theta_{W^l} \frac 1 n Q^l_t(n^{d_l-a_l}\theta_{f}d\hh_{t-1}^{l}\Diag(\cchi_{t-1})\xx_{t-1}^{l-1\trsp})\\
&=-\eta \theta_{W^l} \frac 1 n Q^l_t(d\hh_{t-1}^{l}\Diag(\cchi_{t-1})\xx_{t-1}^{l-1\trsp})
\numberthis \label{eq:delWl}
\end{align*}
for $l \in [2, L]$, and
for $l=1$,
\begin{align*}
W_{t}^{l}-W_{t-1}^{l}
&=-\eta \theta_{W^l} Q^l_t(n^{d_l-a_l} \theta_f d\hh_{t-1}^{l}\Diag(\cchi_{t-1})\xxi_{t-1}^{\trsp})\\
&=-\eta \theta_{W^l} Q^l_t(d\hh_{t-1}^{l}\Diag(\cchi_{t-1})\xxi_{t-1}^{\trsp}).
\numberthis\label{eq:delW1}
\end{align*}

Here, the 2nd equality of each block is due to our assumption of \cref{eq:actlogitinit,eqn:faithful_init,{eqn:stable_faithful}} (see discussion at the start of this section), which means $n^{d_l-a_l} \theta_f = 1$ for all $l \le L$ above and $n^{d_{L+1} - a_{L+1}} = 1$.

\subsection{Program Construction}
\label{sec:ProgramConstruction}

\newcommand\isp{\color{red}{\boldsymbol{\mathcal{X}}}}
\newcommand{\Ee}{\mathcal{E}}

Here we construct the \nexort{} Program encoding the training of an MLP. We separately describe the first forward and backward passes followed by the later forward and backward passes.

\subsubsection{First Forward Pass}

\label{subsec:First-Forward-Pass}

We compute $\hh^1_0 := W^1_0 \xxi$ via \fNonlin{} and then construct the following multi-vectors via \fNonlin{} and \refMatMul{} respectively:
\begin{align}
\xx_{0}^{l}:=\phi(\hh_{0}^{l})\in\R^{n \times \NN},\quad \hh_{0}^{l+1}:=W_{0}^{l+1}\xx_{0}^{l}\in\R^{n \times \NN},\quad\text{for \ensuremath{l=1,\ldots,L-1},}
\end{align}

\subparagraph*{Function Output}

The first output is $\ff_{0}=W_{0}^{L+1\trsp} \xx_{0}^{L} \in \R^{\NN}$, but we will define $\ff_{0}$ in the program slightly differently.

\paragraph{Case when \texorpdfstring{$a_{L+1}+b_{L+1}>1/2$}{a\_\{L+1\}+b\_\{L+1\}>1/2}}

Then $\ff_0\asto0$. In the program, we will construct $\ff_0$ as an \emph{initial (multi-)scalar} mathematically defined by $W_{0}^{L+1\trsp}\xx_{0}^{L}$.%
\footnote{We cannot define it using \refAvg{} + \refOuterNonlin{} because, intuitively, the mechanism of this convergence is through CLT, not Law of Large Numbers.}

\paragraph{Case when \texorpdfstring{$a_{L+1}+b_{L+1}=1/2$}{a\_\{L+1\}+b\_\{L+1\}=1/2}}

\newcommand{\SSigma}{\boldsymbol{\Sigma}}

If $a_{L+1}+b_{L+1}=1/2$, then $\ff_{0}$ converges to a nontrival
Gaussian via CLT \citep{yang}, so we will condition on $\ff_{0}$:
\begin{equation}
  \textit{Given values $\cc\in\R^{\NN}$, we will condition on the event $\Ee$ that $\ff_{0}=\frac{1}{\sqrt{n}}\widehat{W}_{0}^{L+1\trsp}\xx_{0}^{L}$
  equals $\cc$.}
  \label{eqn:conditionf0}
\end{equation}
The distribution of $\widehat{W}_{0}^{L+1}$ conditioned
on $\Ee$ is given by 
\[
\widehat{W}_{0}^{L+1}\disteq_{\Ee}\sqrt{n}\xx^{+\trsp}\cc+\Pi\widetilde{W}_{0}^{L+1}
\]
where $\xx$ is shorthand for $\xx_{0}^{L}$, $\widetilde{W}_{0}^{L+1}$
is an iid copy of $\widehat{W}_{0}^{L+1}$, and $\Pi$ is the orthogonal
projection into the orthogonal complement of the column space of $\xx$
(and $\bullet^{+\trsp}$ denotes the pseudo-inverse transpose as usual).

By standard formulas for pseudo-inverse and orthogonal projection,
we can write 
\[
\xx^{+\trsp}=\frac{1}{n}\xx\SSigma^{+},\quad\Pi=I-\frac{1}{n}\xx\SSigma^{+}\xx^{\trsp},\quad\text{where}\ \ \SSigma\defeq\xx^{\trsp}\xx/n \in \R^{\NN\times \NN}.
\]

If we further define $\ggamma\defeq\xx^{\trsp}\widetilde{W}_{0}^{L+1}/n$,
then 
\[
\Pi\widetilde{W}_{0}^{L+1}=\widetilde{W}_{0}^{L+1}-\xx\SSigma^{+}\ggamma\quad\text{and}\quad\sqrt{n}\xx^{+\trsp}\cc=\frac{1}{\sqrt{n}}\xx\SSigma^{+}\cc.
\]

By the Master Theorem, $\ggamma\asto0$ because $\widetilde{W}_{0}^{L+1}$
is independent from $\xx$, and $\SSigma\asto\mathring{\SSigma}$
for some PSD matrix $\mathring{\SSigma}$. At this point in the program,
all (multi-)scalars we used (like $\cc$) are constant with $n$ and
can be absorbed into nonlinearities. By the rank stability property
of any program without scalars \citep{yang3}, the rank of $\SSigma$
is fixed for large enough $n$, almost surely, so $\SSigma^{+}\asto\mathring{\SSigma}^{+}$
by the continuity of pseudo-inverse on fixed-rank matrices.
Therefore, we have
\[\frac{\cc}{\sqrt{n}},\ggamma\asto0 \implies \SSigma^{+}\frac{\cc}{\sqrt{n}},\SSigma^{+}\ggamma\asto0.\]

Thus, the mathematical conditioning done in \cref{eqn:conditionf0} is achieved programmatically as follows:
\begin{enumerate}
  \item We introduce $\widetilde{W}_{0}^{L+1}$ as an initial vector and $\cc$ (the value of $\ff_0$) as initial scalars
  \item We introduce $\SSigma$ as a multi-scalar via $\SSigma := \xx^{\trsp}\xx/n$ (\fNonlin{} followed by \refAvg{})
  \item We introduce $\ggamma$ as a multi-scalar via $\ggamma := \xx^{\trsp}\widetilde{W}_{0}^{L+1}/n$ (\fNonlin{} followed by \refAvg{})
  \item We replace $\widehat{W}_{0}^{L+1}$ (an initial vector)
  in the program with (the non-initial vector) 
  \[
  \widehat{W}_{\Ee}^{L+1}:=\xx \SSigma^{+}\frac{\cc}{\sqrt{n}}+\widetilde{W}_{0}^{L+1}-\xx \SSigma^{+}\ggamma
  \]
  constructed using \fNonlin{}$(\xx, \widetilde{W}_{0}^{L+1}; \SSigma, \ggamma, \cc, \theta_f=\frac{1}{\sqrt{n}})$.%
  \footnote{recall $\theta_f$ from \cref{eqn:thetaf}.}
\end{enumerate}
Since $\SSigma^{+}\frac{\cc}{\sqrt{n}},\SSigma^{+}\ggamma\asto0$ as above,
we have $\ket{\widehat{W}_{\Ee}^{L+1}}=\ket{\widetilde{W}_{0}^{L+1}}$.
Intuitively, this means that, even after conditioning on $\ff_{0}=\cc$,
the conditional distribution of $\widetilde{W}_{0}^{L+1}$ is practically
the same as the original distribution. We can then proceed exactly
as in the case when $a_{L+1}+b_{L+1}>1/2$, with $\widehat{W}_{\Ee}^{L+1}$
taking the role of $\widehat{W}_{0}^{L+1}$. The program then encodes
the evolution of $\ff$ \emph{conditioned on} $\ff_{0}=\cc$.\footnote{Formally, we can also have $\cc$ as initial scalars, but since they
are fixed with $n$, they can be absorbed into the Nonlin that defines
$\widehat{W}_{\Ee}^{L+1}$.} 

\begin{assm}\label{assm:noNNGP}
For the above reason, we will assume $a_{L+1}+b_{L+1}>1/2$, and remark whenever the case $a_{L+1}+b_{L+1}=1/2$ involves subtleties.
\end{assm}

\subsubsection{First Backward Pass}

Next, we write the backward pass
\begin{align*}
d\xx_{0}^{L} & :=\widehat{W}_{0}^{L+1}\otimes \boldsymbol 1_\NN & \text{\fNonlin{}}\\
d\hh_{0}^{l} & :=d\xx_{0}^{l}\odot\phi'(\hh_{0}^{l}) & \text{\fNonlin{}}\\
d\xx_{0}^{l-1} & :=W_{0}^{l\trsp}d\hh_{0}^{l} & \text{\refMatMul{}}
\end{align*}
where $\boldsymbol 1_\NN$ is the $\NN$-dimensional vector of all 1s, recall, $d\zz$ mathematically equals $\theta_{f}^{-1}\nabla_{\zz}f$.

The error signal at the output is expressed using \fNonlin{} followed by \refAvg{} as in \cref{lemma:transform_scalars}.%
\footnote{Here, $\mathring{\ff}_{0}=0$ if $a_{L+1}+b_{L+1}>1/2$; otherwise,
$\ff_{0}=\mathring{\ff}_{0}$ is the $\cc$ we conditioned on in \cref{eqn:conditionf0}.}
\[
\cchi_{0}:=\eps_{0}(\ff_{0}).
\]

Mirroring \cref{eq:delWL+1}, we also define 
\begin{align*}
  \del W_{1}^{L+1} &:=-\eta Q_{0}^{L+1}(\xx_{0}^{L}\cchi_{0})& \text{\fNonlin{}}  
\end{align*}
to represent the (normalized) change in $W^{L+1}$ due to the first
gradient step. 

\subsubsection{\texorpdfstring{$t$}{t'}th Forward Pass, \texorpdfstring{$t\ge1$}{t
>= 1}}

\subparagraph*{Overview}

We iteratively define $\del\zz_{t}$ to mathematically represent $\theta_{z}^{-1}(\zz_{t}-\zz_{t-1})$,
for $z\in\{x^{l},h^{l}\}_{l}$. Then we eventually set 
\begin{align*}
\zz_{t} & :=\zz_{0}+\theta_{z}\del\zz_{1}+\cdots+\theta_{z}\del\zz_{t}\\
  &= \fNonlin{}(\zz_0, \del \zz_1,\ldots, \del \zz_t; \theta_z)
\end{align*}
Likewise, we will define $\del W_{t}^{L+1}$ so that $W_{t}^{L+1}=\theta_{f}\widehat{W}_{0}^{L+1}+\theta_{L+1}(\del W_{1}^{L+1}+\cdots+\del W_{t}^{L+1})$.
In the program, we will not directly use $W_{t}^{L+1}$ but instead
use its normalized version
\begin{align}
\widehat{W}_{t}^{L+1}
  &:=\widehat{W}_{0}^{L+1}+\theta_{L+1/f}(\del W_{1}^{L+1}+\cdots+\del W_{t}^{L+1})\label{eq:hatWL+1}\\
  &=\fNonlin{}(\widehat{W}_{0}^{L+1}, \del W_{1}^{L+1},\ldots,\del W_{t}^{L+1}; \theta_{L+1/f})
\end{align}
where $\theta_{L+1/f}=\theta_{L+1}/\theta_{f}$. Mathematically, $\widehat{W}_{t}^{L+1}=\theta_{f}^{-1}W_{t}^{L+1}$.

\subparagraph*{The Construction of (Pre)Activations}

We start with $h=h^{1}$: By \cref{eq:delW1}, we have 
\[
\del\hh_{t}:=-\eta Q_{t}^{1}(d\hh_{t-1}\Diag(\cchi_{t-1})\xxi^{\trsp})\xxi=\fNonlin(d\hh_{t-1};\xxi,\eta,\cchi_{t-1}).
\]
For higher layers, if for brevity we write $h=h^{l}$, $x=x^{l-1}$, and $W=W^{l}$, then
$h=Wx$. By \cref{eq:delWl}, we have, mathematically, 
\begin{align*}
\theta_{h}\del\hh_{t} & =\theta_{x}W_{t-1}\del\xx_{t}+(W_{t}-W_{t-1})\xx_{t}\\
 & =\theta_{x}\left(W_{0}\del\xx_{t}+\sum_{s=1}^{t-1}(W_{s}-W_{s-1})\del\xx_{t}\right)+(W_{t}-W_{t-1})\xx_{t}\\
 & =\theta_{x}\left(W_{0}\del\xx_{t}-\eta\theta_{W}\sum_{s=1}^{t-1}Q_{s-1}^{l}(d\hh_{s-1}\Diag(\cchi_{s-1})\xx_{s-1}^{\trsp})\frac{\delta\xx_{t}}{n}\right)\\
 &\quad\quad-\eta\theta_{W}Q_{t-1}^{l}(d\hh_{t-1}\Diag(\cchi_{t-1})\xx_{t-1}^{\trsp})\frac{\xx_{t}}{n}.
\end{align*}
Recall $\theta_{x/h}=\theta_{h}^{-1}\theta_{x},\theta_{W/h}=\theta_{h}^{-1}\theta_{W},\theta_{Wx/h}=\theta_{h}^{-1}\theta_{W}\theta_{x}$.
We construct 
\begin{align*}
\del\hh_{t} & :=\theta_{x/h}W_{0}\del\xx_{t}-\eta\theta_{Wx/h}\sum_{s=1}^{t-1}Q_{s-1}^{l}(d\hh_{s-1}\Diag(\cchi_{s-1})\xx_{s-1}^{\trsp})\frac{\delta\xx_{t}}{n}\\
&\quad\quad -\eta\theta_{W/h}Q_{t-1}^{l}(d\hh_{t-1}\Diag(\cchi_{t-1})\xx_{t-1}^{\trsp})\frac{\xx_{t}}{n}\\
 & =\fOuterNonlin{}(\refMatMul(W_{0},\del\xx_{t}),d\hh_{0},\ldots,d\hh_{t-1};\\
 &\phantom{=\fOuterNonlin{}(}\xx_{0},\ldots,\xx_{t-1},\delta\xx_{t};\\
 &\phantom{=\fOuterNonlin{}(}\eta,\theta_{x/h},\theta_{Wx/h},\theta_{W/h},\cchi_{0},\ldots,\cchi_{t-1})
\end{align*}
If $x=x^{l}$, $h=h^{l}$, then $x=\phi(h)$, and (using $\theta_{x}=\theta_{h}$
(\cref{eq:thetal})), 
\begin{align}
\del\xx_{t} & :=\theta_{h}^{-1}(\phi(\hh_{t-1}+\theta_{h}\del\hh_{t})-\phi(\hh_{t-1}))\nonumber \\
 & = \fNonlin(\hh_{t-1},\del\hh_{t};\theta_{h})\label{eq:delx}
\end{align}
where the function in \fNonlin{} is precisely the difference
quotient for the function $\phi$.\footnote{The pseudo-Lipschitzness of $\phi'$ assumed in \cref{assm:relulike_phi_Q_preserves_sign}
implies that the nonlinearity (the difference quotient function) represented by \fNonlin{} here is pseudo-Lipschitz, so that we can ultimately
apply our Master Theorem.}

\subparagraph*{The Function Outputs}

We do not construct $\ff_{t}$ directly, but rather through scalars
$\del\ff_{t}=\ff_{t}-\ff_{t-1}$, so that
\begin{align*}
  \ff_{t}&:=\ff_{0}+\del\ff_{1}+\cdots+\del\ff_{t}.& \text{\fNonlin{}}  
\end{align*}

Mathematically, $\del\ff_{t}=\theta_{L+1}\del W_{t}^{L+1\trsp}\xx_{t}^{L}+W_{t-1}^{L+1\trsp}\theta_{L}\del\xx_{t}^{L}$,
but we shall write it slightly differently in the program:
\begin{align*}
\del\ff_{t}&:=\theta_{L+1}'\frac{\del W_{t}^{L+1\trsp}\xx_{t}^{L}}{n}+\theta_{Lf}'\frac{\widehat{W}_{t-1}^{L+1\trsp}\del\xx_{t}^{L}}{n}\\
  &= \refAvg(\fNonlin{}(\del W_{t}^{L+1}, \xx_{t}^{L}, \widehat{W}_{t-1}^{L+1}, \del\xx_{t}^{L}; \theta_{L+1}', \theta_{Lf}'))
\end{align*}
where $\theta_{L+1}'=n\theta_{L+1},\theta_{Lf}'=n\theta_{L}\theta_{f}$
and $\widehat{W}_{t-1}^{L+1}$ is constructed in \cref{eq:hatWL+1}.

\subsubsection{\texorpdfstring{$t$}{t'}th Backward Pass, \texorpdfstring{$t\ge1$}{t
>= 1}}

In the last layer, we construct 
\begin{align*}
  d\xx_{t}^{L}&:=\widehat{W}_{t}^{L+1}\otimes \boldsymbol 1_N.  &\text{\fNonlin{}}
\end{align*}
(i.e., outer product between the vector $\widehat{W}_{t}^{L+1}$ and the vector $\boldsymbol 1_N$).

For each $l=L,\ldots,1$ for $dh^{l}$ and $l=L,\ldots,2$ for $dx^{l-1}$,
we also calculate 
\begin{align*}
d\hh_{t}^{l} & :=d\xx_{t}^{l}\odot\phi'(\hh_{t}^{l}) = \fNonlin{}(d\xx_t^l, \hh^l_t)\\
d\xx_{t}^{l-1} & :=W_{0}^{l\trsp}d\hh_{t}^{l}-\eta\theta_{W^{l}}\sum_{s=0}^{t-1}Q_{s}^{l}(\xx_{s}^{l-1}\Diag(\cchi_{s})d\hh_{s}^{l\trsp})d\hh_{t}^{l}/n\\
 & =\fOuterNonlin(\refMatMul(W_{0}^{l\trsp},d\hh_{t}^{l}),\xx_{0}^{l-1},\ldots,\xx_{t-1}^{l-1};\\
 &\phantom{=\fOuterNonlin(}d\hh_{0}^{l},\ldots,d\hh_{t-1}^{l},d\hh_{t}^{l};\\
 &\phantom{=\fOuterNonlin(}\eta, \theta_{W^{l}},\cchi_{0},\ldots,\cchi_{t-1})
\end{align*}
Using \cref{lemma:transform_scalars}, 
we define the error signal
\[
\cchi_{t}:=\eps_{t}(\ff_{t}) = \refAvg(\fNonlin{}(;\ff_t)).
\]
Finally, we compute the (normalized) change in $W^{L+1}$ after this update as in \cref{eq:delWL+1}.
\begin{align*}
  \del W_{t+1}^{L+1}&:=-\eta Q_{t}^{L+1}(\xx_{t}^{L}\cchi_{t}) = \fNonlin{}(\xx^L_t; \cchi_t, \eta)
\end{align*}

\subsection{The Infinite-Width Limit}

\label{sec:infwidthlimit}

In this section, we describe the kets (\cref{defn:ket})
corresponding to the vectors of the program constructed above. According
to the Master Theorem, each such vector $z$ will have roughly iid
coordinates distributed like $\ket z$ in the large $n$ limit.

Let $\mathring{\theta}_{\bullet}$ denote the limit of any $\theta_{\bullet}$
in \cref{sec:Program-Setup}. If pseudostability holds, then $\mathring{\theta}_{\bullet}$
is either 0 or 1, as one can easily verify. We can construct the kets
for each vector in the program, as follows. 
\begin{enumerate}
\item For the first forward and backward passes, we have, 
\begin{align*}
\ket{\hh_{0}^{1}} & =\ket{W_{0}^{1}}\xxi, & \ket{\xx_{0}^{l}} & =\phi(\ket{\hh_{0}^{l}}), & \ket{\hh_{0}^{l+1}} & =\ket{W_{0}^{l+1}\xx_{0}^{l}},\\
\ket{d\xx_{0}^{L}} & =\ket{\widehat{W}_{0}^{L+1}}\otimes\boldsymbol{1}_{\NN}, & \ket{d\hh_{0}^{l}} & =\ket{d\xx_{0}^{l}}\odot\phi'(\ket{\hh_{0}^{l}}), & \ket{d\xx_{0}^{l-1}} & =\ket{W_{0}^{l\trsp}d\hh_{0}^{l}}
\end{align*}
\item For $z\in\{x^{l},h^{l}\}_{l}$, we have 
\begin{equation}
\ket{\zz_{t}}=\ket{\zz_{0}}+\mathring{\theta}_{z}\ket{\del\zz_{1}}+\cdots+\mathring{\theta}_{z}\ket{\del\zz_{t}}\label{eq:Zzt}
\end{equation}
where $\mathring \theta_{x^l} = \mathring \theta_{h^l} = \ind(r_{\le l} = 0)$.
\item For $l\in[L],x=x^{l},h=h^{l}$, we have $\ket {\del\xx_{t}}=\Psi(\ket{\hh_{t-1}},\ket{\del\hh_{t}};\mathring{\theta}_{h})$
where $\Psi$ is the nonlinearity represented by $\fNonlin{}$ in \cref{eq:delx}. If $\mathring{\theta}_{h}=0$
(e.g. if $r_{\le l} >0$), then 
\begin{equation}
\ket{\del\xx_{t}}=\phi'(\ket{\hh_{t-1}})\odot\ket{\del\hh_{t}}.\label{eq:ZdelxWhenr>0}
\end{equation}
Otherwise, $\mathring{\theta}_{h}=1$ (e.g. if $r_{\le l} = 0$), and 
\begin{equation}
\ket{\del\xx_{t}}=\phi(\ket{\hh_{t}})-\phi(\ket{\hh_{t-1}}).\label{eq:ZdelxWhenr=00003D00003D0}
\end{equation}
\item For $h=h^{1}$, we have 
\begin{equation}
\ket{\del\hh_{t}}=-\eta \overline{\ket{d\hh_{t-1}}_{\mathring{\cchi}_{t-1}}\xxi^{\trsp}}  \xxi.
\label{eqn:h1}
\end{equation}
where the bar notation abbreviates $Q^1_{t-1}$.
\item For $l\ge2,h=h^{l},x=x^{l-1},W=W^{l}$, we have 
\begin{align}
\ket{\del\hh_{t}} & =\mathring{\theta}_{x/h}\ket{W_{0}\del \xx_{t}}-\eta\mathring{\theta}_{Wx/h}\sum_{s=0}^{t-1}\Qketdbra{d\hh_{s}}{\mathring{\cchi}_{s}}{\xx_{s}}{\delta\xx_{t}}\ra\nonumber \\
 & \phantomeq\quad-\eta\mathring{\theta}_{W/h}\Qketdbra{d\hh_{t-1}}{\mathring{\cchi}_{t-1}}{\xx_{t-1}}{\xx_{t}}\ra. \numberthis\label{eq:Zdelh}
\end{align}
Here, as is everywhere else, the bar notation abbreviates $Q^l_t$ where $l,t$ are the same as the super/subscripts in $\ket{d\hh^l_t}$ or $\bra{d\hh^l_t}$ contained inside.
Note at least one of $\mathring{\theta}_{x/h}$ and $\mathring{\theta}_{W/h}$
equals 1; the exact formulas are
\begin{align*}
  \mathring{\theta}_{x/h}&=\ind(r_l \ge r_{\le l-1})\\
  \mathring{\theta}_{W/h} &= \ind(r_l \le r_{\le l-1})\\
  \mathring\theta_{Wx/h} &= \ind(r_l=0 \And r_{\le l-1} = 0).
\end{align*}
As usual, we can decompose  $\ket{W_{0}\del\xx_{t}}$ by \cref{defn:ket}. 
\begin{align*}
\ket{W_{0}\del\xx_{t}} & =\hatket{W_{0}\del\xx_{t}}+\dotket{W_{0}\del\xx_{t}}\\
 & =\hatket{W_{0}\del\xx_{t}}+\ket{d\hh_{< t}}\braket{\nabla_{d\hh_{< t}}}{\delta\xx_{t}}.
\end{align*}
where $d\hh_{< t}$ is the multi-vector $(d\hh_0, \ldots, d\hh_{t-1})$.
Here we simplified $\dotket{W_{0}\del\xx_{t}}$ because $\ket{\delta \xx_t}$ only depends on $d\hh_{< t}$ among previous vectors.
\item For last layer weight 
\begin{equation}
\ket{\del W_{t}^{L+1}}=-\eta \overline{\ket{\xx_{t-1}^{L}}\mathring{\cchi}_{t-1}}\label{eq:ZdelWlast}
\end{equation}
(where the bar notation abbreviates $Q_{t-1}^{L+1}$) and 
\begin{equation}
\ket{\widehat{W}_{t}^{L+1}}=\ket{\widehat{W}_{0}^{L+1}}+\mathring{\theta}_{L+1/f}(\ket{\del W_{1}^{L+1}}+\cdots+\ket{\del W_{t}^{L+1}})\label{eq:ZhatW}
\end{equation}
where $\mathring \theta_{L+1/f} = \ind(b_{L+1} = c_{L+1})$.%
\footnote{This equation here assumes $b_{L+1} \le c_{L+1}$ as in \cref{b_{L+1}<=c_{L+1}}.}
\item The output deltas have limits 
\begin{equation}
\del\mathring{\ff}_{t}=\mathring{\theta}_{L+1}'\braket{\del W_{t}^{L+1}}{\xx_{t}^{L}}+\mathring{\theta}_{Lf}'\braket{\widehat{W}_{t-1}^{L+1}}{\del\xx_{t}^{L}}\label{eq:delflimit}
\end{equation}
where $\mathring{\theta}_{L+1}' = \ind(a_{L+1}+c_{L+1}=1)$ and $\mathring \theta_{Lf}' = \ind(r+a_{L+1}+b_{L+1}=1)$,
and 
\[
\mathring{\ff}_{t}=\mathring \ff_0 + \del\mathring{\ff}_{1}+\cdots+\del\mathring{\ff}_{t},
\]
where $\mathring \ff_0 = 0$ if $a_{L+1} + b_{L+1} > 1/2$ (\cref{assm:noNNGP}); otherwise ($a_{L+1} + b_{L+1} = 1/2$), $\mathring \ff_0$ equals the value (from the initial NNGP) we conditioned on (as specified in the discussion above \cref{assm:noNNGP}).
\item For gradients: 
\begin{align*}
\ket{d\xx_{t}^{L}} & =\ket{\widehat{W}_{t}^{L+1}}\otimes\boldsymbol{1}_{\NN}\\
\ket{d\hh_{t}^{l}} & =\ket{d\xx_{t}^{l}}\odot\phi'(\ket{\hh_{t}^{l}})\\
\ket{d\xx_{t}^{l-1}} & =\ket{W_{0}^{l\trsp}d\hh_{t}^{l}}-\eta\mathring{\theta}_{W^{l}}\sum_{s=0}^{t-1}\Qketdbra{\xx_{s}^{l-1}}{\mathring{\cchi}_{s}}{d\hh_{s}^{l}}{d\hh_{t}^{l}}\ra\numberthis \label{eqn:dx}
\end{align*}
where $\mathring \theta_{W^l} = \ind(r_l = 0)$.
\item Error signal
\[
\mathring{\cchi}_{t}=\eps_{t}(\mathring{\ff}_{t}).
\]
\end{enumerate}

From this description of the infinite-width limit, it's clear that the resulting dynamics is both faithful and stable.
\begin{lemma}\label{lemma:sufficiency4stabfaith}
  \cref{eq:actlogitinit,eqn:faithful_init,{eqn:stable_faithful}} imply stability and faithfulness.
\end{lemma}

\subsection{\texorpdfstring{$r=0$}{r=0} Implies Feature Learning}
\label{sec:r=0}

In this section, we implement \cref{item:r=0} of the proof outline.

\begin{thm}\label{thm:r=0impliesFL}
  Adopt \cref{assm:relulike_phi_Q_preserves_sign}.
  Consider a parametrization satisfying \cref{eq:actlogitinit,{eqn:faithful_init},{eqn:stable_faithful},{eqn:nontrivial}}.%
  \footnote{For reader's convenience, these equations are resp.\ the proposed conditions for stability at init, faithful at init, stable and faithful throughout training, and nontriviality.}

  If $r=0$, then the following are true of this parametrization:
  \begin{enumerate}
    \item not in the operator regime
    \item feature learning
    \item feature learning in the $L$th layer
    \item feature kernels evolution
    \item feature kernel evolution in the $L$th layer
    \item prefeature learning
    \item prefeature learning in the $L$th layer
    \item prefeature kernels evolution
    \item prefeature kernel evolution in the $L$th layer
    \item if there is feature learning \emph{or} feature kernel evolution \emph{or} prefeature learning \emph{or} prefeature kernel evolution in layer $l$, then there is feature learning \emph{and} feature kernel evolution \emph{and} prefeature learning \emph{and} prefeature kernel evolution in layers $l,\ldots,L$.
    \end{enumerate}
\end{thm}

\begin{proof}
  WLOG, assume $\NN=1$ where the sole input is nonzero; our construction will work obviously for general $\NN$ by masking the error signal $\eps_t$.
  Correspondingly, we use notation $h^l$ instead of $\hh^l$, etc.

  \paragraph{Outline}
  We will show that as learning rate $\eta \to \infty$, $\braket{h^l_1}{h^l_1}, \braket{x^l_1}{x^l_1} \to \infty$, so that $\braket{h^l_1}{h^l_1}\ne \braket {h^l_0}{h^l_0}$ and $\braket{x^l_1}{x^l_1} \ne \braket {x^l_0}{x^l_0}$ for sufficiently large $\eta$.
  This would imply (pre)feature kernel evolution and (pre)feature learning of the parametrization.
  In addition, we will show that $\mathring f_1$ (output of function after 1 step of update) asympotically grows like $\eta^s$ for some $s>0$.
  If $s\ne 1$, then the dynamics cannot satisfy the the operator equation \cref{eqn:kernelupdate}, which is linear in $\eta$.
  If $s=1$, then we can calculate that $\lim_{\eta \to \infty} \eta^{-1} \mathring f_1 \ne \partial_\eta \mathring f_1 |_{\eta=0}$, so that the update is not perfectly linear in $\eta$, violating \cref{eqn:kernelupdate}, so we are not in the operator regime.
  The other claims will also naturally follow over the course of the proof.

  The general reasoning is that, when $\eta$ is large, by the asymptotic homogeneity of $\phi$ (\cref{assm:relulike_phi_Q_preserves_sign}(3)), $\ket{ h^l_1}$ should roughly be of order $\eta^{1+\delta+\cdots+\delta^j}$ where $j$ is the number of $\ell \le l$ with $r_l = 0$, and the order of $\ket{x^l_1}$ is the $\delta$th power of that.
  This would show $\braket{x^l_1}{x^l_1} \to \infty$.
  Some further calculation shows that $\mathring f_1$ is of the same order as $\ket{x^l_1}$ or $\eta$ more than that.
  This would violate the operator equation \cref{eqn:kernelupdate} either because $\mathring f_1$ does not scale linearly in $\eta$ or because it's not perfectly linear in $\eta$.

  This reasoning should work for ``generic'' activations.
  However, for concreteness, we impose the conditions of \cref{assm:relulike_phi_Q_preserves_sign} to be able to easily prove that certain pathological behavior cannot happen where certain inner products vanish, such as $\braket{x^l_0}{x^l_1}$, allowing this reasoning to become rigorous.

  \paragraph{Proof of theorem}
  
  Let $s_{k}=1+\delta+\cdots+\delta^{k-1}$
  with convention $s_{0}=0$. Let $e_{l}=s_{\sigma},$ where
  $\sigma=\sum_{\ell=1}^{l}\ind(r_{\ell}=0)$ is the partial count of
  how many $r_{l}$s are zero up to layer $l$.
  Because $r=\min_l r_l=0$, $e_l>0$ for some $l$ and is nondecreasing.
  
  Define $\ket{\tilde{h}_{1}^{l}}\defeq\lim_{\eta\to\infty}\ket{x_{1}^{l}}/\eta^{e_{l}}$
  and $\ket{\tilde{x}_{1}^{l}}\defeq\lim_{\eta\to\infty}\ket{x_{1}^{l}}/\eta^{e_{l}\delta}$.
  Design $\eps_{0}$ so that $\mathring{\chi}_{0}=-1$. Then we can
  see that 
  
  \[
  \forall l\in[0,L-1],\ \ket{\tilde{h}_{1}^{l+1}}=\begin{cases}
  \Qketdbra{dh_{0}^{l}}{}{x_{0}^{l}}{\tilde{x}_{1}^{l}}\ra & \text{if \ensuremath{r_{l+1}=0}}\\
  \ket{W_{0}^{l}\tilde{x}_{1}^{l}} & \text{otherwise}
  \end{cases}
  \]
  
  \[
  \forall l\in[1,L],\ \ket{\tilde{x}_{1}^{l}}=\begin{cases}
  \relu^{\delta}(\ket{\tilde{h}_{1}^{l}}) & \text{if \ensuremath{r_{l}=0}}\\
  \phi(\ket{\tilde{h}_{1}^{l}}) & \text{else if \ensuremath{r_{\ell}>0} for all \ensuremath{\ell\le l}}\\
  \relu^{\delta}(\ket{\tilde{h}_{1}^{l}}) & \text{otherwise}
  \end{cases}
  \]

  \paragraph{(Pre)Feature kernel evolution}
  
  Using this formula and the lemmas above, we can perform an easy induction
  to see that 1) $\ket{\tilde{x}_{1}^{l}}\ge0$ and is greater than
  0 with nonzero probability, and 2) $\braket{\tilde{h}_{1}^{l}}{h_{0}^{l}}\ge0$
  for all $l$.
  
  This already implies that $\ket{x_{1}^{L}}$ scales like $\eta^{e_{L}\delta}$
  and $\ket{h_{1}^{L}}$ scales like $\eta^{e_{L}}$ in $\eta$, so
  the (pre)feature kernel evolves.
  
  \paragraph{Output scaling; violation of the operator equation}
  By \cref{eqn:nontrivial}, either 
  $a_{L+1}+c_{L+1}=1$ (last update is maximal) or $a_{L+1}+b_{L+1}+r=1$ (last layer initialization is maximal).

  If the last layer update is maximal, then its contribution
  to the output has the property
  \[
  \lim_{\eta\to\infty}\eta^{-(e_{L}\delta+1)}\braket{\widehat{W}_{1}^{L+1}-\widehat{W}_{0}^{L+1}}{x_{1}^{L}}=\braket{x_{0}^{L}}{\tilde{x}_{1}^{L}}>0
  \]
  by \cref{lem:preserve_positivity} (and the fact that $\relu^{\delta}$
  and $\phi$ both preserve positivity). This means that the contribution
  scales like $\eta^{e_{L}\delta+1}$ in $\eta$.
  
  If the last layer initialization is maximal, then
  its contribution to the output has the property
  \[
  \lim_{\eta\to\infty}\eta^{-e_{L}\delta}\braket{\widehat{W}_{0}^{L+1}}{x_{1}^{L}}=\braket{\widehat{W}_{0}^{L+1}}{\tilde{x}_{1}^{L}}>0
  \]
  by Stein's lemma and \cref{lem:preserve_positivity}. This means the
  contribution scales like $\eta^{e_{L}\delta}$ in $\eta$.%
  \footnote{An important role of \cref{assm:relulike_phi_Q_preserves_sign} is to prevent pathological cases where the above expectations vanish.}
  
  Thus, the output of the function scales either like $\eta^{e_{L}\delta+1}$
  or $\eta^{e_{L}\delta}$.
  This means that, if the last layer update is maximal or $\delta$ is not among the discrete set of values where $e_{L}\delta=1$, then $\mathring f_1$ does not scale linearly in $\eta$.

  If the last layer update is not maximal and $\delta$ takes one of such values (such as $\delta=1$ when $r_l = 0$ for exactly one $l$), then one can calculate
  \begin{equation}
    \braket{\widehat{W}_{0}^{L+1}}{\tilde{x}_{1}^{L}} \ne
    \partial_\eta \braket{\widehat{W}_{0}^{L+1}}{{x}_{1}^{L}}|_{\eta=0}.
  \end{equation}
  which means that $\mathring f_1$ is not perfectly linear in $\eta$.
  
  In either case, \cref{eqn:kernelupdate} cannot be satisfied.

\end{proof}

The above proof (especially the section on ``Output scaling'') in fact also yields the following.
\begin{thm}
  If \cref{eqn:nontrivial} is not satisfied in \cref{thm:r=0impliesFL} (but $r$ is still 0), then the parametrization is trivial.
\end{thm}

\subsubsection{Helper Lemmas}

The significance of the positivity and sign preservation properties in \cref{assm:relulike_phi_Q_preserves_sign} is that they allow us to apply the following lemmas in our reasoning.

\begin{lemma}
  \label{lem:preserve_positivity}If $F,G:\R\to\R$ preserve positivity
  and are nonnegative,%
  \footnote{
    Obviously $F,G$ need to be measurable for the expectation to make sense but let's not get too pedantic;
    rather all functions in this paper are measurable by default.}
  then 
  \[
  \EV F(X)G(Y)>0
  \]
  for any nonzero random variables $X,Y$ with $\EV XY\ge0$.
  \end{lemma}
  \begin{proof}
  Obviously, $\EV XY\ge0$ implies there is nonzero probability that
  both $X$ and $Y$, and thus $F(X)$ and $G(Y)$ are positive. Since
  $F$ and $G$ are nonnegative, the conclusion follows.
  \end{proof}
  \global\long\def\zbf{\mathbf{z}}%
  
  \begin{lemma}
  \label{lem:Qupdate_sign}Let $X,Y\in\R$ be nonnegative random variables.
  Suppose $Q$ preserves sign. Then the function $F:\R\to\R$ defined
  by
  \[
  F(\zbf)\defeq\EV Q(\zbf X)Y
  \]
  is identically 0 if $\EV XY=0$, but otherwise satisfies
  \[
  \sgn F(\zbf)=\sgn\zbf,\ \forall\zbf\in\R.
  \]
  \textup{Again,} here $\sgn$ takes value in $\{-1,0,1\}$.
  \end{lemma}
  In the language of kets, this implies: For nonnegative kets $\ket x,\ket y$
  and any ket $\ket z$,
  \[
  F(\ket z)\defeq Q(\ketdbra z{}x)\ket y
  \]
  is 0 almost surely if $\braket xy=0$ but otherwise satisfies 
  \[
  \sgn F(\ket z)=\sgn\ket z
  \]
  
  \begin{proof}
  If $\EV XY=0$, then $XY$ is almost surely 0 because $X$ and $Y$
  are nonnegative. Then $Q(\zbf X)Y$ is almost surely 0 because $Q$
  preserves sign. So its expectation is 0, hence $F$ is identically
  0.
  
  Otherwise, $\EV XY>0$ and there is nonzero probability that $X$
  and $Y$ are both positive. In this event, $Q(\zbf X)Y$ has the same
  sign as $\zbf$; otherwise $Q(\zbf X)Y=0$. Taking average, we see
  that $\sgn F(\zbf)=\sgn\zbf$ as desired.
  \end{proof}

\section{Proof of Maximal Update Limit}
\label{sec:muLimitProof}

In $\mu$P (\cref{defn:mup}), all of the $\mathring \theta$s in \cref{sec:infwidthlimit} will equal 1.
Some straightforward simplifications then lead to \cref{thm:muLimit_1LP} for the shallow case and \cref{thm:mulimit_MLP} for the deep case.
(We made some simplification in notation as well: $w_t^{L+1}$ in \cref{thm:mulimit_MLP} corresponds to $\widehat W_t^{L+1}$ in \cref{sec:infwidthlimit}  and $w_t^1$  in \cref{thm:mulimit_MLP} corresponds to $W_t^1$ in \cref{sec:infwidthlimit})

\section{Proof of Neural Tangent Limit}
\label{sec:NTKProof}

In this section, we prove \cref{{thm:memorylessNTK}} by specializing the limit formula in \cref{sec:infwidthlimit} above for general abcd-parametrizations.

In NTP (\cref{defn:NTP}), all $r_l$ equal $\nicefrac 1 2$ (\cref{eqn:concrete_r}).
In particular, from \cref{eq:Zzt}, we see that
\[\ket{\xx^l_t} = \ket{\xx^l_0},\quad \ket{\hh^l_t} = \ket{\hh^l_0}\]
for all $l$ and all $t$.
Likewise, from \cref{eq:ZhatW}, we see that
\begin{equation*}
  \ket{\widehat{W}_{t}^{L+1}}=\ket{\widehat{W}_{0}^{L+1}}
\end{equation*}
and from \cref{eqn:dx}, that
\begin{align*}
  \ket{d\xx_{t}^{L}} & =\ket{\widehat{W}_{t}^{L+1}}\otimes\boldsymbol{1}_{\NN}\\
  \ket{d\hh_{t}^{l}} & =\ket{d\xx_{t}^{l}}\odot\phi'(\ket{\hh_{t}^{l}})\\
  \ket{d\xx_{t}^{l-1}} & =\ket{W_{0}^{l\trsp}d\hh_{t}^{l}}
\end{align*}
so that, inductively, 
\[\ket{d\xx^l_t} = \ket{d\xx^l_0},\quad \ket{d\hh^l_t} = \ket{d\hh^l_0}\]
for all $l$ and all $t$.
These consequences justify the notational decision to drop the $t$ subscript in \cref{sec:neuraltangent}.
We now adopt this notation as well in the rest of this section.
However, note that $\del \mathring f_t$, $\del W_{t}^{L+1}$, $\del \hh^l_t$, and $\del \xx^l_t$ still vary with $t$.

In addition, plugging in the NTP values of abcd to \cref{eq:ZdelxWhenr>0,{eq:Zdelh},{eq:delflimit}} gives
\begin{align}
  \ket{\del\xx^l_{t}}&=\phi'(\ket{\hh^l})\odot\ket{\del\hh^l_{t}}.
  \label{eq:NTPZdelxWhenr>0}\\
  \ket{\del\hh^l_{t}} & =\ket{W_{0}\del \xx^{l-1}_{t}}
    -\eta \Qketdbra{d\hh^l}{\mathring{\cchi}_{t-1}}{\xx^{l-1}}{\xx^{l-1}}\ra \label{eq:NTPZdelh}\\
  \del\mathring{\ff}_{t}&=\braket{\del W_{t}^{L+1}}{\xx^{L}}+\braket{\widehat{W}^{L+1}}{\del\xx_{t}^{L}} \label{eq:NTPdelflimit}
\end{align}

To show \cref{thm:memorylessNTK}, we need to calculate $\del \mathring \ff_t$ and show it equals $-\eta \calK_Q(\mathring \cchi_{t-1})$.

There are two contributions to $\del \mathring \ff_t$ from \cref{eq:NTPdelflimit}.
We first calculate the former.
 By \cref{eq:ZdelWlast},
\begin{align*}
  \braket{\del W_{t}^{L+1}}{\xx_{t}^{L}}
    &= -\eta \overline{\mathring{\cchi}_{t-1}^\trsp \bra{\xx_{t-1}^{L}}}{\xx_{t}^{L}}\ra\\
    &= -\eta \overline{\mathring{\cchi}_{t-1}^\trsp \bra{\xx^{L}}}{\xx^{L}}\ra\numberthis \label{eqn:NTKoutputweight}
\end{align*}
(where the bar notation abbreviates $Q_{t-1}^{L+1}$)
which matches with the output weight contribution from \cref{eqn:KQ}.

To calculate the latter term from \cref{eq:NTPdelflimit}, we employ the following lemma.
\begin{lemma}\label{lemma:dh|delh}
  For any $l\in[2,L]$, 
  \[
  \braket{d\hh^{l}}{\del\hh_{t}^{l}}=\braket{d\hh^{L-1}}{\del\hh_{t}^{L-1}}-\eta\la{d\hh^{l}}\Qketdbra{d\hh^{l}}{\mathring{\cchi}_{t-1}}{\xx^{l-1}}{\xx^{l-1}}\ra.
  \]
  For $l=1$, we have
  \[
  \braket{d\hh^{1}}{\del\hh_{t}^{1}}=-\eta\la{d\hh^{1}}\overline{\ket{d\hh^{1}}_{\mathring{\cchi}_{t-1}}\xxi^{\trsp}}\xxi.
  \]
  \end{lemma}
  \begin{proof}
  For $l=1$, this follows trivially from \cref{eqn:h1}. For $l\in[2,L]$,
  by \cref{eq:NTPZdelh}, 
  \[
  \braket{d\hh^{l}}{\del\hh_{t}^{l}}=\braket{d\hh^{l}}{W_{0}^{l}\del x_{t}^{l-1}}-\eta\la{d\hh^{l}}\overline{\ketdbra{d\hh^{l}}{\mathring{\cchi}_{t-1}}{\xx^{l-1}}}{\xx^{l-1}}\ra.
  \]
  So it remains to show $\braket{d\hh^{l}}{W_{0}^{l}\del \xx_{t}^{l-1}}=\braket{d\hh^{l-1}}{\del\hh_{t}^{l-1}}$.
  But applying \cref{lemma:adjunction},
  \begin{align*}
  \braket{d\hh^{l}}{W_{0}^{l}\del \xx_{t}^{l-1}} & =\braket{W_{0}^{l\trsp}d\hh^{l}}{\del \xx_{t}^{l-1}}=\braket{d\xx^{l-1}}{\del\xx_{t}^{l-1}}\\
   & =\braket{d\xx^{l-1}}{\phi'(\hh^{l-1})\odot\del\hh_{t}^{l-1}}\\
   & =\braket{d\xx^{l-1}\odot\phi'(\hh^{l-1})}{\del\hh_{t}^{l-1}}\\
   & =\braket{d\hh^{l-1}}{\del\hh_{t}^{l-1}}
  \end{align*}
  as desired.
  \end{proof}
Then the latter term from \cref{eq:NTPdelflimit} is
\[
\braket{\widehat{W}^{L+1}}{\del\xx_{t}^{L}}=\braket{\widehat{W}^{L+1}}{\phi'(\hh^{L})\odot\del\hh_{t}^{L}}=\braket{\widehat{W}^{L+1}\odot\phi'(\hh^{L})}{\del\hh_{t}^{L}}=\Diag\braket{d\hh^{L}}{\del\hh_{t}^{L}}
\]
which, by \cref{eqn:NTKoutputweight} and a trivial induction with \cref{lemma:dh|delh}, gives
\[
\del\mathring{\ff}_{t}=-\eta \calK_{\QQ}(\mathring{\cchi}_{t-1})
= -\eta \calK_{\QQ}(\eps_{t-1}(\mathring{\ff}_{t-1}))
\]
as desired.

\chapter{Proof of Master Theorem}
\label{chap:masterproof}

To prove our main foundational theorem \cref{thm:MasterTheorem}, we had an editorial choice: we could bash our way to a proof by repeating (over and over again) \emph{analytical} arguments involving the likes of Holder and Cauchy-Schwarz, or we could encapsulate them into neat objects with neat properties that then make the underlying \emph{algebraic} structure more transparent.
Even though the writing is more arduous, we chose the latter way, because of this transparency and because it builds a more extensible foundation for future work.%
\footnote{\emph{extensible} in the sense that properties can be used as black-boxes for more advanced theorems. This is definitely a more \emph{algebraic} style, compared to analysts who tend to open black boxes and constantly tweak the insides.}
The drawback is perhaps that there can be many long sections building up the underlying structures before the payoff --- but that's the choice we made, and the reader is stuck with it, for better (hopefully) or worse.

\section{Basic Tools}

Here we just record some basic lemmas that we would use frequently.
The reader can skip this on first read and come back when necessary.

\begin{lemma}\label{lem:powerbound}
For an integer $m$, and complex numbers $a_i \in \mathbb{C}$, $i \in [k]$,
\[
\left| \sum_{i=1}^k a_i \right|^m
\le 
k^{m-1} \sum_{i=1}^k \left|a_i\right|^m
.
\]
\end{lemma}
\begin{proof}
Expand the power in the LHS using the multinomial theorem, apply AM-GM to each summand, and finally aggregate using triangle inequality.
\end{proof}

\begin{lemma}\label{lemma:momentBoundASConvergence}
Let $\{X_n\}_{n \ge 1}$ be a sequence of random variables with zero mean.
If for some $p \in \N$ and for all $n$, $\EV X_n^{2p} \le c n^{-1-\lambda}$, for some $\lambda > 0$, then $X_n \to 0$ almost surely.
\end{lemma}
\begin{proof}
By Markov's inequality, for any $\epsilon > 0$,
\begin{align*}
    \Pr(|X_n| > \epsilon)
        &=
            \Pr(X_n^{2p} > \epsilon^{2p})
        \le
            \EV X_n^{2p}/\epsilon^{2p}
        \le c n^{-1-\lambda}/\epsilon^{2p}
        \\
    \sum_n \Pr(|X_n| > \epsilon)
        &\le
            \sum_n c n^{-1-\lambda}/\epsilon^{2p}
        <    
            \infty.
\end{align*}
By Borel-Cantelli Lemma, almost surely, $|X_n| \le \epsilon$ for all large $n$.
Then, if we pick a sequence $\{\epsilon_k > 0\}_k$ converging to 0, we have that, almost surely, for each $k$, $|X_n| \le \epsilon_k$ for large enough $n$ --- i.e. almost surely, $X_n \to 0$.
\end{proof}

\begin{lemma}
\label{lem:CLT_moments}Let $x\in\R^{n}$ be a random vector with
iid components such that $\EV x_{1}=0$ and $\EV|x_{1}|^{p}<\infty$
for all integer $p\ge1$. Then there exist constants $C_{p}$, independent
of $n$, for all $p\ge1$ such that $S\defeq\f 1{\sqrt{n}}\sum_{\alpha=1}^{n}x_{\alpha}$
satisfies 
\[
\EV|S|^{p}<C_{p},\quad\forall p\ge1.
\]
Furthermore, $C_{p}$ can be taken to be $P_{p}(\nu_{1},\nu_{2},\ldots,\nu_{p})$
where 
\end{lemma}
\begin{itemize}
\item $\nu_{k}$ is an upper bound on the signed $k$th moment $\EV x_{1}^{k}$
of $x_{1}$
\item $P_{p}$ is a polynomial that depends only on $p$.
\end{itemize}
\begin{proof}
See \cite[prop H.2]{tp3b}.
\end{proof}

\subsection{Review of Moore-Penrose Pseudo-Inverse}
\label{sec:pinv}

We recall Moore-Penrose pseudo-inverse and some properties of it.
\begin{defn}\label{defn:pseuodoinverse}
For $A \in \R^{n \times m}$, a pseudo-inverse of $A$ is defined as a matrix $A^+ \in \R^{m \times n}$ that satisfies all of the following criteria
\begin{align*}
  A A^+ A &= A,&
  A^+ A A^+ &= A^+,&
  (AA^+)^\trsp &= AA^+,&
  (A^+ A)^\trsp &= A^+ A
  .
\end{align*}
\end{defn}

The following facts are standard.
\begin{itemize}
    \item If $A$ has real entries, then so does $A^+$.
    \item The pseudo-inverse always exists and is unique.
    \item When $A$ is invertible, $A^+ = \inv A$.
    \item $(A^\trsp)^+ = (A^+)^\trsp$, which we denote as $A^{+\trsp}$.
    \item $A^+ = (A^\trsp A)^+ A^\trsp = A^\trsp (A A^\trsp)^+$.
    \item $AA^+$ is the orthogonal projector to the column space of $A$;
        $I - A^+ A$ is the orthogonal project to the null space of $A$.
    \item If $A$ has singular value decomposition $A = U\Lambda V$ where $U$ and $V$ are orthogonal and $\Lambda$ has the singular values on its diagonal, then $A^+ = V^\trsp \Lambda^+ U^\trsp$ where $\Lambda^+$ inverts all nonzero entries of $\Lambda$.
    \item For any collection of vectors $\{v_i\}_{i=1}^n$ in a Hilbert space, $w \mapsto \sum_{i,j=1}^n v_i (\Sigma^+)_{ij} \la v_j, w \ra $, where $\Sigma_{ij} = \la v_i, v_j \ra$, is the projection operator to the linear span of $\{v_i\}_{i=1}^n$.
\end{itemize}

\subsection{Baranyai's Theorem}

\begin{figure}[t]
    \centering
    \includegraphics[width=0.45\textwidth]{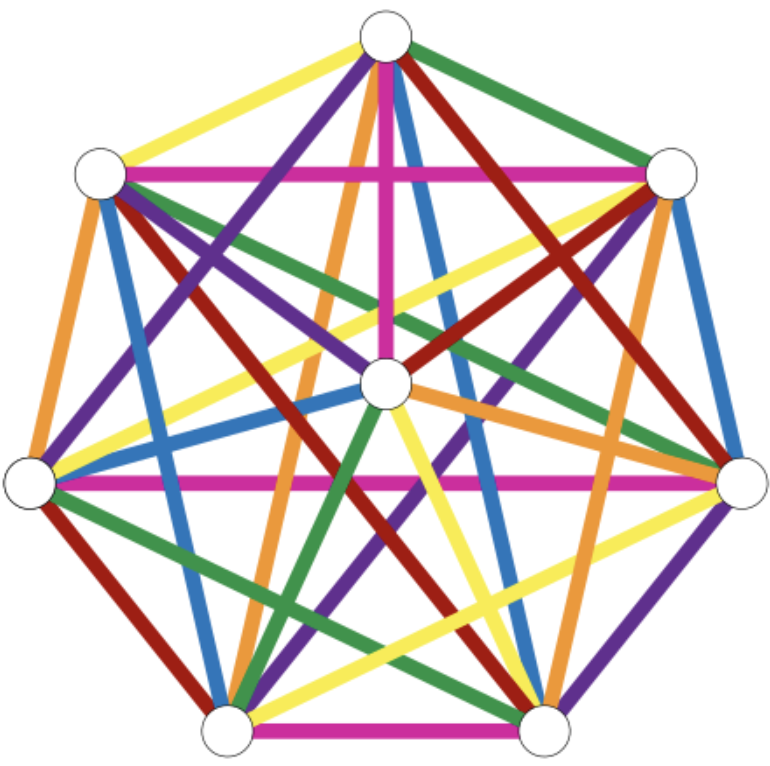}
    \caption{
    \textbf{A graphical illustration of Baranyai's Theorem for $n=8,r=2$}.
    Here $G^8_2$ is just the usual complete graph on 8 vertices.
    A \emph{perfect matching} here reduces to the usual meaning on graphs: a set of 4 edges that covers all 8 vertices.
    Every edge above is colored, and for each color, the edges with that color form a perfect matching.
    \emph{Image source:} \href{https://en.wikipedia.org/wiki/Baranyai\%27s_theorem\#/media/File:Complete-edge-coloring.svg}{\nolinkurl{wikipedia.org}}}
    \label{fig:baranya}
\end{figure}

\emph{The complete hypergraph $G^n_r$} is a hypergraph containing $n$ vertices in which every subset of $r$ vertices forms a hyperedge.
A \emph{perfect matching} of it is a set of $\nicefrac n r$ hyperedges that (thought of as subsets of vertices) partitions the vertices of $G^n_r$.
\begin{thm}[Baranyai's Theorem]\label{thm:bara}
  Suppose integer $r$ divides integer $n$.
  The collection of all $\binom n r$ hyperedges in $G^n_r$ can be partitioned into $\binom n r \frac r n$ perfect matchings. 
\end{thm}
See \cref{fig:baranya} for a graphical illustration.
The hyperedges of $G^n_r$ are just the $r$-element subsets of $[n]$.
A version of Baranyai's Theorem also holds for ordered $r$-element subsets, i.e., length-$r$ sequences of distinct elements of $[n]$.
\begin{thm}[Ordered Baranyai's Theorem]\label{thm:obara}
  Suppose integer $r$ divides integer $n$.
  The collection of all $n(n-1)\cdots (n-r+1)$ length-$r$ sequences of distinct elements of $[n]$ can be partitioned into $r(n-1)\cdots (n-r+1)$ perfect matchings.
\end{thm}
The proof follows from \cref{thm:bara} and symmetrization.

\section{Basic Objects and Operations}

In this section, we define the basic objects and operations on them recurring in our quest to understand Tensor Programs.
In the next, we describe their properties.

\subsection{Fixed-Dimensional Random Variables and Vectors}

\label{sec:fixed_dim_objects}

\paragraph{Notation}

We will predominantly consider sequences of random objects indexed
by an integer $n$ (see \cref{sec:seqspace} below). In this context,
to emphasize objects that do not vary with $n$, we will use uppercase
letters like $X,Y,\ldots$ for scalar random variables and their bolded
versions $\XX,\YY,\ldots$ for random vectors, matrices, or tensors.
The latter will have components denoted by superscripts, e.g., $\XX=(X^{1},\ldots,X^{k})$
if $\XX\in\R^{k}$.

\subsection{Space of Random Sequences}

\label{sec:seqspace}

Prior papers in the Tensor Programs series often talk about objects
like scalars, vectors, and matrices whose size varies with a global
notion of ``width'' denoted $n$. Formally, each such object is
a sequence (of scalars, vectors, or matrices) in $n$, but to be intuitive,
these works downplay this sequence aspect (for example, by suppressing
the dependence on $n$ notationally).

However, here we need to talk about more complex high order tensors
who can contain both dimensions that scale with $n$ and those that
do not. In addition, we will formulate a notion of ``vanishing''
tensors that is really an asymptotic property as $n\to\infty$ rather
than a nonasymptotic one; this notion turns out to interact cleanly
with typical operations like pseudo-Lipschitz nonlinearities that
is involved in Tensor Programs. As such, for these reasons, we will
be explicit that the objects we play with here are sequences in $n$:
\begin{defn}
For any integer $s\ge0$, denote by $\seqspace^{s}$ the set of all
random sequences $x$ of order-$s$ tensors with dimensions $\underbrace{n\times n\times\cdots\times n}_{s}$
(i.e., sequences $x$ with $x(n)\in\otimes^{s}\R^{n}$ for all $n=1,2,\ldots$).
We call the objects in $\seqspace^{0}$, $\seqspace^{1}$, $\seqspace^{2}$
respectively \emph{$n$-scalars}, \emph{$n$-vectors}, and \emph{$n$-matrices}.
For general $s$, we call the objects in $\seqspace^{s}$ \emph{$n$-tensors}.
Often, we will drop the prefix ``$n$-'' when the context is clear.
\end{defn}
So, for example, $\seqspace^{0}$ contains all infinite sequences
of scalar random variables, and $\seqspace^{1}$ contains all infinite
sequences of random vectors of linearly increasing size. 

\paragraph{Notation}

We will use Greek letters $\alpha,\beta,\ldots$ (with values in $[n]$)
to denote indices of an $n$-tensor. For higher order $n$-tensors
in $\seqspace^{s}$, we also use their bolded counterparts $\aalpha,\bbeta,\ldots$
(with values in $[n]^{s}$) to denote multi-indices, where (for example)
$\aalpha$ is understood to have components $\aalpha=(\alpha_{1},\ldots,\alpha_{s})$.
For example, if $x\in\seqspace^{s}$, then $x(n)$ has entries $\{x(n)_{\alpha_{1}\ldots\alpha_{s}}:\alpha_{1},\ldots,\alpha_{s}\in[n]\}=\{x(n)_{\aalpha}:\aalpha\in[n]^{s}\}$.
We can also mix single indices and multi-indices, e.g., $\{x(n)_{\alpha\bbeta}:\alpha\in[n],\bbeta\in[n]^{s-1}\}$.

As in prior works, even though we will work with sequences of tensors
$(x(n))_{n\ge1}$, we will suppress the dependence on $n$ notationally
and talk about $x$ as if it's a fixed tensor. So, for example, for
an $n$-matrix $x$, $x_{\alpha\beta}$ refers to the entry $x(n)_{\alpha\beta}$
where the $n$ is from context.

\subsection{Multi-Tensors}
\label{sec:multitensor}

In many results we shall discuss, we often talk about lists of $n$-tensors,
e.g., $x^{1},\ldots,x^{k}\in\seqspace^{s}$. Going forward, it will
be helpful to think of such lists as (a sequence of) a single tensor
of shape $n\times n\times\cdots\times n\times k$ (or other arrangements
of dimensions, as discussed below) for each $n$. We generalize this
further in the following definition.
\begin{defn}[Multi-Tensors]
\label{def:tensorseq}For any integers $s\ge0$ and $k\ge0$, let
$\seqspace^{s}\otimes\R^{k}$ denote the set of all random sequences (in
$n$) of tensors of shape $\underbrace{n\times\cdots\times n}_{s}\times k$.

In general, suppose $\calV=\R^{k_{1}\times\cdots\times k_{r}}$. Then
$\seqspace^{s}\otimes\calV$ denotes the same but for shape $\underbrace{n\times\cdots\times n}_{s}\times k_{1}\times\cdots\times k_{r}$.
We shall call any element of such a space a \emph{multi-tensor}. When
$s=0,1,2$, we also call such elements \emph{multi-scalar}, \emph{multi-vector},
and \emph{multi-matrix} respectively.
\end{defn}
Intuitively, $\seqspace^{s}$ can be read as $\otimes^{s}\R^{n}$,
so that $\seqspace^{s}\otimes\R^{k_{1}\times\cdots\times k_{r}}$
can be read as $\R^{n}\otimes\cdots\otimes\R^{n}\otimes\R^{k_{1}\times\cdots\times k_{r}}\cong\R^{n}\otimes\cdots\otimes\R^{n}\otimes\R^{k_{1}}\otimes\cdots\otimes\R^{k_{r}}$,
which makes the tensor shapes more apparent. 

More abstractly, we can let $\calV$ be any finite-dimensional Euclidean
space, in which case $\seqspace^{s}\otimes\calV$ is the space of
sequences of tensors with shape $\underbrace{n\times\cdots\times n}_{s}$,
taking values in $\calV$. However, in this work, we will primarly
concern ourselves with the $\calV$ being tensor spaces as in \cref{def:tensorseq}.
\begin{defn}
In this work, by \emph{Euclidean space} we mean any finite-dimensional
real vector space. We will use notation $\calV$ and its cousins to
denote such spaces.
\end{defn}
In this abstraction, we will also talk about nonlinearities that map
between Euclidean spaces (e.g., $\psi:\calV\to\calV'$), generalizing
the scalar functions (e.g., $\psi:\R^{k}\to\R$) typical of prior
works. This generality allows us to compose nonlinearities more naturally,
which simplifies our presentation conceptually and our proofs technically.

\paragraph{Notation}

We will always use lower case letters $x,y,z,\ldots$ to denote $n$-tensors
(i.e., elements of $\seqspace^{\bullet}$). On the other hand, we
will always use their bolded counterparts $\xx,\yy,\zz,\ldots$ to
denote multi-tensors (i.e., elements of $\seq^{\bullet}\otimes\calV$). 

We think of a multi-tensor as a (fixed-size) tuple of $n$-tensors:
for example, $\xx\in\seq^{s}\otimes\R^{k}$ is thought of as a tuple
$(x^{1},\ldots,x^{k})$ of $n$-tensors $x^{i}\in\seqspace^{s}$.
As in this example, we will always use letters $i,j,k,\ldots$ in
superscript for ``constant-sized'' indices (corresponding to $\R^{k}$
in the example) to distinguish from the Greek letters $\alpha,\beta,\ldots$
in subscript for the ``$n$-sized'' indices. This is consistent
with the notation in \cref{sec:fixed_dim_objects} for fixed-dimensional
objects. So $\xx$ here has elements $\{x_{\aalpha}^{i}:\aalpha\in[n]^{s}\text{ and }i\in[k]\}$.
In particular, for any $\aalpha\in[n]^{s}$, $\xx_{\aalpha}=\{x_{\aalpha}^{i}\}_{i}\in\R^{k}$;
for any $i\in[k]$, $x^{i}=\{x_{\aalpha}^{i}\}_{\aalpha}\in\seqspace^{s}$.
Note that when we pick the $i$th component $x^{i}$ of $\xx$, the
letter $x$ is not bolded.

In contrast, we will need to talk about a list of multi-tensors as
well, $\xx^{1},\ldots,\xx^{r}$, where we use the superscript letters
$r,s,t,\ldots$ to denote the indices of this list. Each multi-tensor
has components $\xx^{t}=\{x^{ti}\}_{i}=\{x_{\aalpha}^{ti}\}_{i,\aalpha}=\{\xx_{\aalpha}^{t}\}_{\aalpha}$.
Again, notationally, notice that $\xx^{t}$ means a multi-tensor in
a sequence, but $x^{i}$ means the $i$th component of multi-tensor
$\xx$.

\subsection{Constant Tensors}

\begin{defn}[Notation for constant sequences]
\label{def:constseq}Often we will need to talk about some sequence
(in $n$) that equals a fixed value, say $\vartheta$, for all $n$.
Then we shall denote this sequence by $\vartheta$ as well, which
should not cause confusion in our contexts. 
\end{defn}

\subsection{IID Tensors}
\begin{defn}\label{defn:bxiid}
Let $\XX$ be any random object. Then $\XX^{\bx{1}},\XX^{\bx{2}},\ldots$
denote iid copies of $\XX$: $\XX^{\bx{i}}\disteq\XX$ for all
$i$ and $\XX,\XX^{\bx{1}},\XX^{\bx{2}},\ldots$ are mutually
independent. If $\YY$ is another random object, then $\YY^{\bx{1}},\YY^{\bx{2}},\ldots$
are iid copies of $\YY$ that furthermore satisfy $(\XX^{\bx{i}},\YY^{\bx{i}})\disteq(\XX,\YY)$
for each $i$.
\end{defn}
In other words, the collection of all objects with superscript $\bx{i}$
forms an ``isomorphic world'' to the collection with superscript
$\bx{j}$. We will always use the above notation when we want to
make a ``constant number'' (wrt $n$) of iid copies (except in \cref{defn:iid} below). However, we
will also at times need to make $n$ (or powers of $n$) iid copies,
for which we use the following notation instead.
We first give the general but abstract definition before substantiating it with examples.

\begin{defn}\label{defn:iid}
  For every integer $s\ge 0$, We define the \emph{iid operator}
  \[
  \iid:\seqspace^{s}\otimes\calV\to\seqspace^{s+1}\otimes\calV
  \]
  as follows.
  If $\xx\in\seqspace^{s}\otimes\calV$, then $\iid(\xx)\in\seqspace^{s+1}\otimes\calV$
is the multi-tensor satisfying, for any $\alpha \in [n]$,
\[\iid(\xx)(n)_\alpha = \xx^{\bx{\alpha}}(n).\]

\end{defn}

\begin{exmp}
If $Z\in\R$ is a (scalar) random variable, we can think of $Z$ as the sequence in $\seqspace^0$ identically equal to $Z$ (c.f.\ \cref{def:constseq}).
Then $\iid(Z)$ denotes
the sequence in $\seqspace^{1}$ where, for each $n\ge1$, the $n$th
element $\iid(Z)(n)$ is the $n$-dimensional random vector where
each entry is an iid copy of $Z$.%
\footnote{Furthermore, $\iid(Z)(n)$ is identical to the first $n$ elements of $\iid(Z)(n+1)$. However, we will never use this fact here.}

Similarly, if $\ZZ\in\calV$ for some Euclidean space $\calV$,
then
for each $n$, $\iid(\ZZ)(n)_{\alpha}$ is an iid copy of $\ZZ$ for every
$\alpha\in[n]$.

In general, if $\ZZ\in\R^{k_{1}\times\cdots\times k_{r}}$
is a random tensor, then $\iid(\ZZ)\in\seqspace^{1}\otimes\R^{k_{1}\times\cdots\times k_{r}}$
is the sequence whose $n$th element is the shape-$(n\times k_{1}\times\cdots\times k_{r})$
tensor $\iid(\ZZ)(n)$ where, for each $\alpha\in[n]$, $\iid(\ZZ)(n)_{\alpha}$
is an iid copy of $\ZZ$.
\end{exmp}

\begin{exmp}
  If $Z \in \R$ is a random variable, then $\iid^2(Z) = \iid(\iid(Z))$ is the sequence of $n \times n$ matrices with iid entries drawn from $Z$.
\end{exmp}

When we iterate $\iid$ and want to clarify or emphasize the shape information, we will
also write $\iid_{n}$ for $\iid$, $\iid_{n\times n}$ for $\iid^{2}$,
$\iid_{n\times n\times n}$ for $\iid^{3}$, and so on.
So $\iid_{n \times n}(\Gaus(0, 1))$ is the sequence of $n\times n$ standard Gaussian matrices.

\begin{rem}
  Note that $\iid$ is only well defined up to distributional equality.
  Thus, if for example one wants to prove almost sure equality and the proof involves $\iid$, then one needs to carefully check that the usage of $\iid$ makes sense.
\end{rem}

\subsection{Averaging over \texorpdfstring{$n$}{n}}
\begin{defn}
Let $\calV$ be a finite-dimensional Euclidean space, and let $\xx\in\seqspace^{s+r}\otimes\calV$
be a multi-tensor. For $\aalpha\in[n]^{s},\bbeta\in[n]^{r}$, we write
\[
\left\langle \xx_{\aalpha\bbeta}\right\rangle _{\bbeta}\defeq\frac{1}{n^{r}}\sum_{\bbeta\in[n]^{r}}\xx_{\aalpha\bbeta}\in\seqspace^{s}\otimes\calV
\]
 for averaging over multi-index $\bbeta$ while fixing multi-index
$\aalpha$.
\end{defn}
This notation can be nested, e.g., $\left\langle F\left(\left\langle \xx_{\aalpha\bbeta}\right\rangle _{\bbeta}\right)\right\rangle _{\aalpha}$for
some function $F$. When $F$ is identity, this is obviously just
$\left\langle \xx_{\ggamma}\right\rangle _{\ggamma}$.

\subsection{Implicit Broadcasting of Nonlinearities on Multi-Tensors}

Given a nonlinearity $\phi:\R^{k}\to\R$ and $n$-vectors $x^{1},\ldots,x^{k}$,
the expression $y=\phi(x^{1},\ldots,x^{k})$ denotes another $n$-vector
$y$ with entries $y_{\alpha}=\phi(x_{\alpha}^{1},\ldots,x_{\alpha}^{k})$
(with the multi-tensor notation $\xx=(x^{1},\ldots,x^{k})$, we can
also write $y=\phi(\xx)$). In the terminology of linear-algebra software
(such as numpy), we say ``$\phi$ is implicitly broadcast across
the dimension $n$.'' 

In general, this implicit broadcast rule holds for any $n$-tensors
or multi-tensors.
\begin{defn}
\label{def:broadcast}Let $\calV$ be any finite-dimensional Euclidean
space (e.g., $\calV=\R^{k_{1}\times\cdots\times k_{r}}$). For any
multi-tensor $\xx\in\seqspace^{s}\otimes\calV$ and $\phi:\calV\to\R$,
the expression $\phi(\xx)$ denotes a new $n$-tensor $y\in\seqspace^{s}$
with $y_{\aalpha}=\phi(\xx_{\aalpha})$.\footnote{Unpacking this a bit: for each $\aalpha\in[n]^{s}$, $\xx_{\aalpha}\in\mathcal{V}$,
so $\phi(\xx_{\aalpha})$ is well-defined and yields a value in $\R$}

More generally, if $\calV'$ is another Euclidean space and $\phi:\calV\to\calV'$,
then $\phi(\xx)\in\seqspace^{s}\otimes\calV'$ with $y_{\aalpha}=\phi(\xx_{\aalpha})$.
\end{defn}
In other words, the implicit broadcasting lifts $\phi:\calV\to\calV'$
to a function $\phi:\seqspace^{s}\otimes\calV\to\seqspace^{s}\otimes\calV'$
for any $s\ge0$.

\subsection{Nonlinear Outer Products}

An outer product of two vectors $x,y$ is the matrix $W$ with entries
$W_{\alpha\beta}=x_{\alpha}y_{\beta}$. This can be written trivially
as $W_{\alpha\beta}=\psi(x_{\alpha};y_{\beta})$ for $\psi(-;-)$
being the product function. This is generalized by a notion of \emph{nonlinear
outer product}, where $\psi$ can be any function, not just product.
We make the definition below in full generality, but for first reading,
it may help to mentally set $\calV_{1},\ldots,\calV_{r},\calV'$ to
be $\R$ to understand the basic idea.
\begin{defn}[Nonlinear Outer Product]
\label{def:nonlinouter}Let $\calV_{1},\ldots,\calV_{r},\calV'$
be any finite dimensional Euclidean spaces. Suppose $\psi:\calV_{1}\oplus\cdots\oplus\calV_{r}\to\calV'$,
where we format $\psi$'s arguments in blocks $\psi(-;-;\cdots;-)$,
with the $t$th block ``$-$'' corresponding to $\calV_{t}$. Then
given $\xx^{t}\in\seqspace^{s_{t}}\otimes\calV_{t}$ for $t=1,\ldots,r$,
we write 
\[
\psi(\xx^{1};\cdots;\xx^{r})
\]
for the multi-tensor $\yy$ in $\seqspace^{\sum_{t}s_{t}}\otimes\calV'$
with entries
\[
\yy_{\aalpha_{1}\cdots\aalpha_{r}}=\psi(\xx_{\aalpha_{1}}^{1};\cdots;\xx_{\aalpha_{r}}^{r}),\quad\text{for all }\text{\ensuremath{\aalpha_{t}}\ensuremath{\in[n]^{s_{t}}}}.
\]
We call this the \emph{$\psi$-outer product of $\xx^{1},\ldots,\xx^{r}$}.
\end{defn}
Let's unpack the multi-tensor notation a bit. For example, if $\xx^{t}=(x^{t1},x^{t2})$
for every $t=1,\ldots,r$, then
\[
\psi(\xx^{1};\cdots;\xx^{r})_{\aalpha_{1}\cdots\aalpha_{r}}=\psi(x_{\aalpha_{1}}^{11},x_{\aalpha_{1}}^{12};\cdots;x_{\aalpha_{r}}^{r1},x_{\aalpha_{r}}^{r2}).
\]
In other words, arguments in the same block have matching (multi-)indices,
while arguments in different blocks have freely varying indices. This
in particular generalizes \cref{def:broadcast}, which applies when
there is only one block in $\psi$. Likewise, the implicit broadcasting
in this nonlinear outer product lifts $\psi:\calV_{1}\oplus\cdots\oplus\calV_{r}\to{\cal V}'$
to a function $\psi:\seqspace^{s_{1}}\otimes\calV_{1}\oplus\cdots\oplus\seqspace^{s_{r}}\otimes\calV_{r}\to\seqspace^{\sum_{i}s_{i}}\otimes{\cal V}'$. 

We isolate the case when $\psi:\calV_{1}\oplus\cdots\oplus\calV_{r}\to\calV'$
is the identity function with $\calV'=\calV_{1}\oplus\cdots\oplus\calV_{r}$.
\begin{defn}
$(\xx^{1};\cdots;\xx^{r})$ is called the \emph{semicolon product
of $\xx^{1},\ldots,\xx^{r}$.} It has entries
\[
(\xx^{1};\cdots;\xx^{r})_{\aalpha_{1}\cdots\aalpha_{r}}=(\xx_{\aalpha_{1}}^{1},\ldots,\xx_{\aalpha_{r}}^{r}).
\]
\end{defn}
The importance of this operation is that: for any $\psi$, the $\psi$-outer
product is just the composition of the semicolon product followed
by application of $\psi$ as in \cref{def:broadcast}. 

\begin{equation}
\psi(\xx^{1};\cdots;\xx^{r})=\psi(\xx),\quad\text{where}\quad\xx=(\xx^{1};\cdots;\xx^{r}).\label{eq:psi_o_semicolon}
\end{equation}

This fact simplifies the proofs involving nonlinear outer products.

\paragraph{Consistency with ``Nonlinearities with Parameters''}

The semicolon notation in nonlinear outer product is consistent with
previous papers in the Tensor Programs series. There, the nonlinearities
are all of the form $\psi:\R^{k}\oplus\R^{l}\to\R$, applied to $k$
$n$-vectors $x^{1},\ldots,x^{k}$ and $l$ $n$-scalars $\theta^{1},\ldots,\theta^{l}$
of the program to produce a single vector:
\[
\psi(x^{1},\ldots,x^{k};\theta^{1},\ldots,\theta^{l}).
\]
In light of \cref{def:nonlinouter}, this is now interpreted as a $\psi$-outer
product of the multi-vector $\xx=(x^{1},\ldots,x^{k})$ and the multi-scalar
$\ttheta=(\theta^{1},\ldots,\theta^{l})$.

\section{Vanishing and Bounded Moments}

\subsection{Moment-Bounded Multi-Tensors}

For the following discussion, let $\calV$ be finite-dimensional Eucldiean
spaces and let $\xx\in\seqspace^{s}\otimes\calV$ be a multi-tensor.
For first time reading, one may mentally set $\calV=\R$, so that
$\seqspace^{s}\otimes\calV=\seqspace^{s}$, to get the basic intuitions
more quickly.

In this work, we will especially focus on $\xx$ where each entry
of $\xx(n)$ has ``typical size $O(1)$'' as $n\to\infty$. We formalize
this ``typical size $O(1)$'' criterion as follows.
\begin{defn}
We say a multi-tensor $\xx$ is \emph{entrywise moment-bounded}, or
just\emph{ moment-bounded} for short, if the following holds: for
every integer $p\ge1$
\[\left\langle \|\xx_{\aalpha}\|_{p}^{p}\right\rangle _{\aalpha} = \tilde O(1).\]
with $\tilde O$ from \cref{defn:tildeO}.
We write $\seqb^{s}\otimes\calV$ for the space of such multi-tensors.
\end{defn}
Unpacking the big-O notation, $\xx$ is moment-bounded iff, for
every integer $p\ge1$
and every $\epsilon>0$, we have 
\[
n^{-\epsilon}\cdot\left\langle \|\xx_{\aalpha}\|_{p}^{p}\right\rangle _{\aalpha}\asto0,\quad\text{as}\quad n\to\infty.
\]
Intuitively, if one thinks of the entries $\xx_{\aalpha}$ as samples from a distribution
$\mathcal{D}$, then $\left\langle \|\xx_{\aalpha}\|_{p}^{p}\right\rangle _{\aalpha}$
is the (empirical) $p$th moment of $\mathcal{D}$. Thus, \emph{moment-boundedness
}just means that $\mathcal{D}$ has bounded empirical moments of every order (ignoring logarithmic factors),
i.e., samples from $\mathcal{D}$ has typical size $O(1)$.

\subsubsection{Basic Properties}

First, one can note the following trivial property that is useful
for simplifying proofs.
\begin{prop}
$\xx=(x^{1},\ldots,x^{k})$ is moment-bounded iff each of its components
$x^{1},\ldots,x^{k}$ is moment-bounded.
\end{prop}
A related property holds for ``components along $n$'':
\begin{prop}
  If random $\ZZ \in \R^k$ has all moments, then $\iid^r(\ZZ)$ is moment-bounded for all $r\ge0$.
\end{prop}
Moment-boundedness is closed under applications of polynomially bounded
functions: in short,
\begin{itemize}
\item $\psi(\xx)$ is moment-bounded if $\xx$ is, and 
\item more generally, $\psi(\xx^{1};\cdots;\xx^{r})$ is moment-bounded
if $\xx^{1},\ldots,\xx^{r}$ are.
\end{itemize}
We make this formal in the following
\begin{prop}
\label{prop:Sb_closed_polyb-}Consider moment-bounded multi-tensors
$\xx^{t}\in\seqb^{s_{t}}\otimes\calV_{t}$ for each $t=1,\ldots,r$.
Let $\psi(-;\cdots;-):\calV_{1}\oplus\cdots\oplus\calV_{r}\to\calV'$
be polynomially bounded. Then $\psi(\xx^{1};\cdots;\xx^{r})$ is moment-bounded
as well: $\psi(\xx^{1};\cdots;\xx^{r})\in\seqb^{\sum_{t}s_{t}}\otimes{\cal V}'$.
\end{prop}
For example, for $r=1$, when we take $\psi:\R^{2}\to\R$ as the sum
and product functions, we have that $\seqb^{s}$ is closed under entrywise
sum and product. Similarly, when we take $\psi:\R^{k\times l+l\times m}\to\R^{k\times m}$
to be the matrix multiplication function (multiplying matrices of
sizes $k\times l$ and $l\times m$ to get one of size $k\times m$),
we derive the closure of $\seqb^{s}\otimes\R^{k\times k}$ under matrix
multiplication (after setting $l=m=k$).
\begin{proof}
By polynomially-boundedness of $\psi$,
for any $q>0$, there are $C,p$ such that, for any $\zz^{t}\in\calV_{t},t=1,\ldots,r$,
we have
\[
\|\psi(\zz^{1};\cdots;\zz^{r})\|_{q}^{q}\le C(1+\|\zz^{1}\|_{p}^{p}+\cdots+\|\zz^{r}\|_{p}^{p}).
\]
Therefore, for $\bbeta\in[n]^{s}$,
\begin{align*}
\left\langle \|\psi(\xx_{\aalpha_{1}}^{1};\cdots;\xx_{\aalpha_{r}}^{r})\|_{q}^{q}\right\rangle _{\aalpha_{1}\cdots\aalpha_{r}} & \le C\left\langle (1+\|\xx_{\aalpha_{1}}^{1}\|_{p}^{p}+\cdots+\|\xx_{\aalpha_{r}}^{r}\|_{p}^{p})\right\rangle _{\aalpha_{1}\cdots\aalpha_{r}}\\
 & =C\left(1+\left\langle \|\xx_{\aalpha}^{1}\|_{p}^{p}\right\rangle _{\aalpha}+\cdots+\left\langle \|\xx_{\aalpha}^{r}\|_{p}^{p}\right\rangle _{\aalpha}\right).
\end{align*}
But, for each $t\in[r]$, since $\xx^{t}$ is moment-bounded, we have
$n^{-\epsilon}\left\langle \|\xx_{\aalpha}^{t}\|_{p}^{p}\right\rangle _{\aalpha}\asto0$
for any $\epsilon>0$. Therefore, 
\[
n^{-\epsilon}\left\langle \|\psi(\xx_{\aalpha_{1}}^{1};\cdots;\xx_{\aalpha_{r}}^{r})\|_{q}^{q}\right\rangle _{\aalpha_{1}\cdots\aalpha_{r}}\le n^{-\epsilon}C\left(1+\left\langle \|\xx_{\aalpha}^{1}\|_{p}^{p}\right\rangle _{\aalpha}+\cdots+\left\langle \|\xx_{\aalpha}^{r}\|_{p}^{p}\right\rangle _{\aalpha}\right)\asto0
\]
as well. Since $q$ and $\epsilon$ are arbitrary in this argument,
this shows $\psi(\xx^{1};\cdots;\xx^{r})$ is moment-bounded.
\end{proof}

\subsection{Vanishing Multi-Tensors}

A mentioned above, a notion of ``vanishing'' will play a central
role in what follows:
\begin{defn}\label{defn:vanish}
We say a multi-tensors $\xx$ is \emph{entrywise vanishing}, or just
\emph{vanishing} for short, if
\[\left\langle \|\xx_{\aalpha}\|^{2}\right\rangle _{\aalpha} = \tilde O(n^{-1}),\]
with $\tilde O$ from \cref{defn:tildeO}.
We write $\seqz^{s}\otimes\calV$ for the space of such multi-tensors.\emph{}\footnote{We emphasize that ``vanishing'' is a shorthand for ``entrywise-vanishing''
and \emph{not} interpreted as ``normwise-vanishing''. For example,
an entrywise-vanishing vector can still have $\Omega(1)$ norm, such
as the vector with entries $n^{-1/2}$; a random Gaussian matrix with
iid $\Gaus(0,1/n)$ entries will have $\Theta(1)$ operator norm almost
surely as $n\to\infty$. But since we will never talk about ``normwise-vanishing,''
for our purposes it will not cause confusion to abbreviate ``entrywise-vanishing''
to just ``vanishing.''}
\end{defn}
Unpacking the big-O notation, 
$\xx$ is vanishing iff, for every $\epsilon>0$, 
\[
n^{1-\epsilon}\cdot\left\langle \|\xx_{\aalpha}\|^{2}\right\rangle _{\aalpha}\asto0,\quad\text{as}\quad n\to\infty.
\]
At first, using the same intuition as above, one may think of vanishing
tensors as those whose entries have typical size $O(1/\sqrt{n})$. 

But notice that ``vanishing'' is defined only via $L^{2}$ norm,
while ``moment-bounded'' is defined via every $L^{p}$ norm. This
is an important technical distinction. The primary purpose of this
distinction is that when $x$ is an $n$-vector and $W$ is an iid
matrix with (for example) $\Gaus(0,1/n)$ entries, 
\begin{center}
\textbf{$Wx$ is vanishing if $x$ is; see} \cref{prop:S0_matmul}.
\par\end{center}

This is because of the well-known almost sure operator-norm bounds
\citep{tao,yin_limit_1988} on such iid matrices $W$. If we defined ``vanishing''
based on other $L^{p}$ norm as well, then we cannot make the same
statement as it will be much more difficult to control the $L^{p}$
norm of $Wx$.

As a result of this difference in the definitions, ``vanishing''
implies ``moment-bounded'' \emph{only for scalars and vectors} (i.e.,
$\seqz^{s}\subseteq\seqb^{s}$ only for $s=0,1$); see \cref{lem:S0_Sb_embedding}
below. For example, consider the $n$-matrix $x$ with all-zero entries
except the value $\sqrt{n}$ in its top left corner. It is vanishing
(because $\left\langle |x_{\aalpha}|^{2}\right\rangle _{\aalpha}=\Theta(1/n)$)
but not moment-bounded (e.g., $\left\langle |x_{\aalpha}|^{8}\right\rangle _{\aalpha}=\Theta(n^{2})$).
The takeaway here is that while moment-bounded tensors can be thought
of as ``every entry is $O(1)$'', vanishing tensors can have a small
number of entries that explode with $n$, as long as the quadratic
mean of the entries has size $O(\nicefrac{1}{\sqrt{n}})$.

We summarize these intuitions succinctly as follows (because they
are important to understand sooner than later):
\begin{align*}
\text{moment-bounded} & \approx\text{dense entries all of size \ensuremath{O(1)}}\\
\text{vanishing} & \approx\text{potentially sparse entries, \ensuremath{O(\nicefrac{1}{\sqrt{n}})} in quadratic mean}.
\end{align*}

\subsubsection{Basic Properties}

Like for moment-boundedness, the ``vanishing'' property is reducible
to components, which makes proofs a bit simpler.
\begin{prop}
$\xx=(x^{1},\ldots,x^{k})$ is vanishing iff each of its components
$x^{1},\ldots,x^{k}$ is vanishing.
\end{prop}
As mentioned above, $\seqz^{s}\subseteq\seqb^{s}$ only for $s=0,1$.
\begin{lemma}
\label{lem:S0_Sb_embedding}Let $\xx\in\seqz^{s}\otimes\calV$ for
Euclidean space $\calV$. Whenever $s=0,1$ or else $p<\frac{2s}{s-1}$,
we have
\[
\left\langle \|\xx_{\aalpha}\|_{p}^{p}\right\rangle _{\aalpha}\asto0.
\]

Consequently, $\seqz^{s}\otimes\calV\sbe\seqb^{s}\otimes\calV$ for
$s=0,1$. i.e., all vanishing multi-scalars and multi-vectors are
moment-bounded. But this is false for $s\ge2$.
\end{lemma}
\begin{proof}
For clarity, we prove the claim for $x\in\seqz^{s}$; the generalization
to $\xx\in\seqz^{s}\otimes\calV$ follows from this case componentwise.

By $L^{p}$ norm inequalities, for all $p\ge2$,
\[
\|x(n)\|_{p}\le\|x(n)\|_{2}
\]
so 
\[
\|x(n)\|_{p}^{p}/n^{s}\le\|x(n)\|_{2}^{p}/n^{s}=K{}^{p/2}
\]
where
\[
K\defeq\|x(n)\|_{2}^{2}/n^{s\cdot\frac{2}{p}}.
\]

When $s=0$ or 1 or when $p<\frac{2s}{s-1}$, we have $s\cdot\frac{2}{p}>s-1$,
so that, for any sufficiently small $\delta>0$, 
\[
K<n^{1-\delta}\cdot\|x(n)\|_{2}^{2}/n^{s}\quad\text{for all \ensuremath{n}}.
\]
But since $x$ is vanishing, the RHS goes to 0 almost surely, and
therefore so does $K$. We thus have
\[
\|x(n)\|_{p}^{p}/n^{s}=\left\langle \|\xx_{\aalpha}\|_{p}^{p}\right\rangle _{\aalpha}=K^{p/2}\asto0
\]
as desired.
\end{proof}
The following yields a sufficient condition for vanishing that is
often easier to show than directly showing vanishing itself. Roughly
speaking, it says an $n$-tensor is vanishing if every entry looks
like $O(1/\sqrt{n})$ as measured by every power mean expectation.
\begin{prop}
\label{prop:powermean_implies_vanishing}

Suppose $v\in\seqb^{s}$. If there exist constants $C_{p,\varepsilon}$
for all $\varepsilon>0$ and all integers $p\ge1$ such that we have
the following inequality
\[
n^{-s}\EV\|v\|_{2p}^{2p}=\EV\left\langle v_{\alpha}^{2p}\right\rangle _{\alpha}\le C_{p,\varepsilon}n^{-p+\varepsilon}
\]
 for all $p,\varepsilon$ and $n$, then $v$ is vanishing.
\end{prop}
\begin{proof}
WLOG we assume $s=1$, since the general case follows by unrolling
$v$ into a (giant) vector.

To show $v$ is vanshing, we need to show $n^{-\epsilon}\|v\|^{2}\asto0$
for any $\epsilon>0$. By \cref{{lemma:momentBoundASConvergence}}, we just need to
show that for some $q,\gamma>0$, we have $\EV(n^{-\epsilon}\|v\|^{2})^{q}=O(n^{-1-\gamma})$.
But $\|v\|\le n^{1/2-1/2q}\|v\|_{2q}$ for all $q\ge1$. Thus, for
any $q\ge1$,
\[
\EV(n^{-\epsilon}\|v\|^{2})^{q}\le n^{-\epsilon q}\EV n^{q-1}\|v\|_{2q}^{2q}\le n^{q-\epsilon q}C_{q,\varepsilon}n^{-q+\varepsilon}=n^{\varepsilon-\epsilon q}C_{q,\varepsilon}.
\]
Whence we can take any small $\varepsilon>0$ and $q>(1+\varepsilon)/\epsilon$
to have $\EV(n^{-\epsilon}\|v\|^{2})^{q}=O(n^{-1-\gamma})$, as desired.
\end{proof}

\subsection{Equivalence Modulo Vanishing Multi-Tensors}
\label{sec:equiv}
\begin{defn}
Let $\xx,\yy$ be multi-tensors of the same shape. We say \emph{$\xx$
is equivalent to $\yy$, written} $\xx\equiv\yy$, if $\xx-\yy$ is
vanishing. 
\end{defn}
Since our writing convention suppresses $n$, this notation may be
ambiguous: To disambiguate, the equivalence $\xx\equiv\yy$ is a notion
between multi-tensors as sequences, i.e., it should be read as $\{\xx(n)\}_{n}\equiv\{\yy(n)\}_{n}$,
NOT as a sequence of equivalences $\xx(n)\equiv\yy(n)$, one for each
$n$.

The $\xx,\yy$ here will all have $\Theta(1)$-sized entries in our
applications. By the discussion above regarding vanishing multi-tensors,
$\xx\equiv\yy$ just means that $\xx$ and $\yy$ have roughly the
same entries. But note that, as vectors, matrices, or tensors, $\xx$
can definitely differ from $\yy$ nontrivially in norm, because e.g.,
a vanishing vector can have $\Theta(1)$ norm.

We first note a trivial but useful property.
\begin{prop}
$(x^{1},\ldots,x^{k})\equiv(y^{1},\ldots,y^{k})$ iff $x^{1}\equiv y^{1},\ldots$,
and $x^{k}\equiv y^{k}$.
\end{prop}
Equivalence is preserved under most operations, as summarized below:
\begin{enumerate}
\item ``smooth'' mapping $\phi$: (\cref{prop:equiv_pLip_tensor} and \cref{prop:localLip_equiv})
\begin{enumerate}
\item For example, if $\xx\equiv\yy$ are both moment-bounded, then $\phi(\xx)\equiv\phi(\yy)$.
This holds more generally for ``smooth'' nonlinear outer products.
\end{enumerate}
\item averaging over $[n]$: (\cref{lem:avg_equiv})
\begin{enumerate}
\item If $\xx\equiv\yy$, then $\la\xx_{\aalpha\bbeta}\ra_{\bbeta}\equiv\la\yy_{\aalpha\bbeta}\ra_{\bbeta}$
\end{enumerate}
\item multiplication by operator-norm-bounded matrices: (\cref{prop:S0_matmul})
\begin{enumerate}
\item If $x\equiv y$ are $n$-vectors and $W$ is an $n$-matrix that almost
surely has bounded operator norm, then $Wx\equiv Wy$. This holds
in particular for $W$ having iid, zero-mean entries of size $\Theta(1/\sqrt{n})$.
\end{enumerate}
\end{enumerate}

\subsubsection{Preservation under Nonlinear Outer Products}

\begin{prop}
\label{prop:equiv_pLip_tensor}
Consider multi-tensors $\xx^{t},\yy^{t}\in\seqb^{s_{t}}\otimes\calV_{t}$
for each $t=1,\ldots,r$ and a function $\psi(-;\cdots;-):\calV_{1}\oplus\cdots\oplus\calV_{t}\to\calV'$.
If $\psi$ is pseudo-Lipschitz and $\xx^{t}\equiv\yy^{t}$ for all
$t$, then 
\[
\psi(\xx^{1};\cdots;\xx^{r})\equiv\psi(\yy^{1};\cdots;\yy^{r}).
\]
\end{prop}
\begin{proof}
Note that this holds for the semicolon product (i.e., when $\psi$
is identity and $\calV'=\calV_{1}\oplus\cdots\oplus\calV_{t}$): If
we shorthand $\DDelta=(\xx^{1};\cdots;\xx^{r})-(\yy^{1};\cdots;\yy^{r})$,
then

\[
\DDelta_{\bbeta_{1}\cdots\bbeta_{r}}=(\xx_{\bbeta_{1}}^{1}-\yy_{\bbeta_{1}}^{1},\ldots,\xx_{\bbeta_{r}}^{r}-\yy_{\bbeta_{r}}^{r})\in\calV'
\]
so that, with $\aalpha=(\bbeta_{1},\ldots,\bbeta_{r})$,

\begin{align*}
\left\langle \|\DDelta_{\aalpha}\|^{2}\right\rangle _{\aalpha} & =\left\langle \|\xx_{\bbeta_{1}}^{1}-\yy_{\bbeta_{1}}^{1}\|^{2}+\cdots+\|\xx_{\bbeta_{r}}^{r}-\yy_{\bbeta_{r}}^{r}\|^{2}\right\rangle _{\bbeta_{1}\cdots\bbeta_{r}}\\
 & =\left\langle \|\xx_{\bbeta}^{1}-\yy_{\bbeta}^{1}\|^{2}\right\rangle _{\bbeta}+\cdots+\left\langle \|\xx_{\bbeta}^{r}-\yy_{\bbeta}^{r}\|^{2}\right\rangle _{\bbeta}.
\end{align*}
Then since $\xx^{1}-\yy^{1},\ldots,\xx^{r}-\yy^{r}$ are all vanishing,
so is $\DDelta$. Thus, we get $(\xx^{1};\cdots;\xx^{r})\equiv(\yy^{1};\cdots;\yy^{r})$.

By \cref{eq:psi_o_semicolon}, this means, for proving \cref{prop:equiv_pLip_tensor},
it suffices to show $\psi(\xx)\equiv\psi(\yy)$ when $\psi:\calV\to\calV'$
and $\xx,\yy\in\seqb^{s}\otimes\calV$ with $\xx\equiv\yy$, where
$\calV,\calV'$ are arbitrary Euclidean spaces. WLOG, assume $\calV=\R^{k}$
for some integer $k\ge0$. We can further assume $\calV'$ to be $\R$,
because the general case follows componentwise from this. We proceed
as follows.

By the definition of pseudo-Lipschitz, we have
\begin{align*}
|\psi(\xx_{\bbeta})-\psi(\yy_{\bbeta})| & \le C\|\xx_{\bbeta}-\yy_{\bbeta}\|u_{\bbeta}\\
\text{where}\quad u_{\bbeta} & \defeq1+\|\xx_{\bbeta}\|_{p}^{p}+\|\yy_{\bbeta}\|_{p}^{p}
\end{align*}
for some constants $C,p>0$. By \cref{prop:Sb_closed_polyb-}, $(u_{\bbeta})_{\bbeta}$
form a moment-bounded tensor $u\in\seqb^{s}$. Thus
\[
\left\langle |\psi(\xx_{\bbeta})-\psi(\yy_{\bbeta})|^{2}\right\rangle _{\bbeta}\le C\left\langle \|\xx_{\bbeta}-\yy_{\bbeta}\|^{2}u_{\bbeta}^{2}\right\rangle _{\bbeta}=C\sum_{i=1}^{k}\left\langle (x_{\bbeta}^{i}-y_{\bbeta}^{i})^{2}u_{\bbeta}^{2}\right\rangle _{\bbeta}.
\]

Now fix $i$ and let $v=x^{i}-y^{i}$. By assumption, $v$ is vanishing.
We will show $n^{1-\epsilon}\left\langle v_{\bbeta}^{2}u_{\bbeta}^{2}\right\rangle _{\bbeta}\asto0$
for any $\epsilon>0$, for any $i$. By the above inequality, this
will prove the desired result. Below, we shall abbreivate $\la\bullet_{\bbeta}\rangle_{\bbeta}$
as just $\la\bullet\ra$.

By Holder's inequality, for any $q,r>2$ such that $\frac{2}{q}+\frac{2}{r}=1$,
we have
\[
\left\langle v^{2}u^{2}\right\rangle \le\left\langle v^{q}\right\rangle ^{2/q}\left\langle u^{r}\right\rangle ^{2/r}.
\]
For any $\epsilon>0$, we shall choose $q$ (and consequently $r$
by the relation $\frac{2}{q}+\frac{2}{r}=1$) barely larger than 2
such that 
\begin{align}
n^{1-\epsilon/2}\left\langle v^{q}\right\rangle ^{2/q} & \asto0\label{eq:_equiv_pLip1}\\
n^{-\epsilon/2}\left\langle u^{r}\right\rangle ^{2/r} & \asto0\label{eq:_equiv_pLip2}
\end{align}
from which follows $n^{1-\epsilon}\left\langle v^{2}u^{2}\right\rangle $
as we wanted. Now, \cref{eq:_equiv_pLip2} holds for any $\epsilon>0$
and $r>0$ because $u$ is moment-bounded.

For \cref{eq:_equiv_pLip1}, notice \footnote{Here, the definition of ``vanishing'' in terms of L2 norm becomes
critical, and if we were to define ``vanishing'' using general Lp
norm, then the proof wouldn't go through.}
\[
\left\langle v^{q}\right\rangle ^{2/q}=n^{-2s/q}\|v\|_{q}^{2}\le n^{-2s/q}\|v\|_{2}^{2}=n^{s(1-2/q)}\left\langle v^{2}\right\rangle 
\]
by standard $L^{p}$ norm inequality. We shall choose $q$ just slightly
above $2$ so that $s(1-2/q)<\epsilon/2$. Then $1-\epsilon/2+s(1-2/q)<1-\delta$
for some positive $\delta>0$ and hence
\[
n^{1-\epsilon/2}\left\langle v^{q}\right\rangle ^{2/q}\le n^{1-\epsilon/2+s(1-2/q)}\left\langle v^{2}\right\rangle <n^{1-\delta}\left\langle v^{2}\right\rangle \asto0
\]
where the convergence to 0 is because $v$ is vanishing.
\end{proof}
While most nonlinearities we encounter will be pseudo-Lipschitz (globally),
occasionally we need to work with the functions that are only locally
Lipschitz around some point (for example, the matrix inverse function
is locally Lipschitz around a nonsingular matrix). Equivalence to
deterministic constant is preserved under locally Lipschitz mapping:
\begin{prop}
\label{prop:localLip_equiv}Consider moment-bounded multi-scalar $\ttheta\in\seqb^{0}\otimes\calV$
such that $\ttheta\equiv\mathring{\ttheta}$ for some deterministic
$\mathring{\ttheta}\in\calV$.\footnote{Recall this means the random sequence $\ttheta$ is equivalent to
the sequence that equals $\mathring{\ttheta}$ identically; c.f. \cref{def:constseq}.} Let $\psi:\calV\to\calV'$ be locally Lipschitz at $\mathring{\ttheta}$.
Then 
\[
\psi(\ttheta)\equiv\psi(\mathring{\ttheta}).
\]
\end{prop}
\begin{proof}
WLOG assume $\calV=\R^{k}$. Let $U$ be the open neighborhood of
$\mathring{\ttheta}$ in $\R^{k}$ such that $\psi$ is Lipschitz
on $U$ with Lipschitz constant $L$. Almost surely, $\ttheta\in U$
as $n\to\infty$ because $\ttheta\equiv\mathring{\ttheta}$. Therefore,
almost surely, $\|\psi(\ttheta)-\psi(\mathring{\ttheta})\|\le L\|\ttheta-\mathring{\ttheta}\|$
as $n\to\infty$. Since $\ttheta-\mathring{\ttheta}$ is vanishing,
we have $\psi(\ttheta)-\psi(\mathring{\ttheta})$ is vanishing as
well.
\end{proof}

\subsubsection{Preservation under Averaging}

\begin{lemma}
\label{lem:avg_equiv}Suppose $\xx,\yy\in\seqb^{s}\otimes\calV$ are
equivalent. With $\bbeta$ denoting any subset of indices, let $\left\langle \xx_{\bullet\bbeta}\right\rangle _{\bbeta}$
be the multi-tensor $\xx'$ with entries $\xx'_{\aalpha}=\left\langle \xx_{\aalpha\bbeta}\right\rangle _{\bbeta}$.
Then

1) $\left\langle \xx_{\bullet\bbeta}\right\rangle _{\bbeta},\left\langle \yy_{\bullet\bbeta}\right\rangle _{\bbeta}$
are both moment-bounded;

2) we have
\[
\left\langle \xx_{\bullet\bbeta}\right\rangle _{\bbeta}\equiv\left\langle \yy_{\bullet\bbeta}\right\rangle _{\bbeta}.
\]
\end{lemma}
\begin{proof}
Note 1) follows easily from Jensen's inequality. So we shall focus
on 2) in the remainder.

We first make a few simplifications: 1) We assume $\bbeta$ is just
the single last index and use the unbolded font $\beta$ instead.
The general case follows from the obvious induction. 2) WLOG we can
assume $\calV=\R$ because the general case follows componentwise
from this. We then use unbolded font $x,y$ instead of $\xx,\yy$.
3) By linearity of $\equiv$, it suffices to prove this for $y\equiv0$
(i.e., $x$ and $y$ are both vanishing). 

Let $\bbeta'=(\beta_{1},\ldots,\beta_{s-1})\in[n]^{s-1}$ denote the
first $s-1$ indices. Then, by power-mean inequality, for every $\bbeta'$,
\[
\left\langle x_{\bbeta'\beta}\right\rangle _{\beta}^{2}\le\left\langle x_{\bbeta'\beta}^{2}\right\rangle _{\beta}.
\]
Therefore,
\[
\left\Vert \left\langle x_{\bullet\beta}\right\rangle _{\beta}\right\Vert ^{2}=\sum_{\bbeta'}\left\langle x_{\bbeta'\beta}\right\rangle _{\beta}^{2}\le n^{s-1}\left\langle x_{\bbeta}^{2}\right\rangle _{\bbeta}
\]
where $\bbeta$ ranges over all $[n]^{s}$. Then 
\[
n^{1-\epsilon}\cdot\left\Vert \left\langle x_{\bullet\beta}\right\rangle _{\beta}\right\Vert ^{2}/n^{s-1}\le n^{1-\epsilon}\left\langle x_{\bbeta}^{2}\right\rangle _{\bbeta}\asto0
\]
where the almost sure convergence is because $x$ is vanishing. Thus
we have $\left\langle x_{\bullet\beta}\right\rangle _{\beta}$ is
vanishing, as well.
\end{proof}

\subsubsection{Preservation under Matrix Multiplication}
\begin{prop}
\label{prop:S0_matmul}

Given $n$-vectors $x,y\in\seqb^{1}$ with $x\equiv y$ and an $n$-matrix
$W\in\seqb^{2}$ with almost surely bounded operator norm, we have
\[
Wx\equiv Wy,
\]
no matter how $W$ is correlated with $x$ and $y$.

In particular, this applies when $W(n)=\frac{1}{\sqrt{n}}\iid_{n\times n}(Z)$
where $Z$ is a sub-Gaussian random variable with zero mean and unit variance.%
\footnote{We expect this operator norm bound to hold more generally for $Z$ that is zero-mean, unit variance, and has finite fourth moment, due to \cite{yin_limit_1988}.
But this result assumes the matrices $W(n), n \ge1$, are upper left blocks of an infinite iid matrix.
It should be possible to relax this assumption in \cite{yin_limit_1988}, but since this is not crucial here, we will just leave this task to future works for which this is more critical.}
\end{prop}
\begin{proof}
The first statement is trivial since the definition of ``vanishing''
only depends on $L^{2}$ norm. The second statement follows from well-known operator norm bounds on iid matrices with sub-Gaussian entries.
\end{proof}

\subsection{Distributional Equivalence (aka Dequivalence)}

We will also need the distributional version of equivalence. In the
simplest case, we have the following definition.
\begin{defn}
\label{def:distequiv1}For any multi-tensor $\xx,\yy$ of the same
shape, we say $\xx$ and $\yy$ are \emph{distributionally equivalent, or dequivalent}, written
\[\xx\distequiv\yy,\] if there exist multi-tensors $\hat{\xx},\hat{\yy}$
with $\hat{\xx}\equiv\xx$ and $\hat{\yy}\equiv\yy$ such that $\hat{\xx}\disteq\hat{\yy}$.
\end{defn}
Like for equivalence, $\xx\distequiv\yy$ is a notion between multi-tensors
as sequences, i.e., it should be read as $\{\xx(n)\}_{n}\distequiv\{\yy(n)\}_{n}$,
NOT as a sequence of equivalences $\xx(n)\distequiv\yy(n)$, one for
each $n$.

However, note that, contrary to equality, distributional equality
in general is not reducible componentwise: $(x^{1},x^{2})\disteq(y^{1},y^{2})$
is usually stronger than the conjunction of $x^{1}\disteq y^{1}$
and $x^{2}\disteq y^{2}$. Thus, we need to define a more general
notion of distributional equivalence of list of objects.
\begin{defn}
\label{def:distequiv2}Suppose we have two lists of multi-tensors
$\xx^{1},\ldots,\xx^{r}$ and $\yy^{1},\ldots,\yy^{r}$, such that
$x^{t}$ and $\yy^{t}$ have the same shape for every $t$, but their
shapes can vary with $t$. We say $\left(\xx^{1},\ldots,\xx^{r}\right)$
and $\left(\yy^{1},\ldots,\yy^{r}\right)$ are \emph{distributionally
equivalent,} written 
\[
\left(\xx^{1},\ldots,\xx^{r}\right)\distequiv\left(\yy^{1},\ldots,\yy^{r}\right),
\]
if there exist $\hat{\xx}^{t}\equiv\xx^{t}$, $\hat{\yy}^{t}\equiv\yy^{t}$
for all $t\in[r]$ such that $(\hat{\xx}^{1},\ldots,\hat{\xx}^{r})\disteq(\hat{\yy}^{1},\ldots,\hat{\yy}^{r})$.
\end{defn}
This definition is especially necessary for reasoning about the Gaussian
conditioning trick inside the induction proof of the Master Theorem.

Finally, we note that, while \cref{def:distequiv2} seems to be strictly
more general than \cref{def:distequiv1}, in fact they are logically
equivalent: 
\begin{equation}
\left(\xx^{1},\ldots,\xx^{r}\right)\distequiv\left(\yy^{1},\ldots,\yy^{r}\right)\iff\left(\xx^{1};\cdots;\xx^{r}\right)\distequiv\left(\yy^{1};\cdots;\yy^{r}\right)\label{eq:distequiv_12equiv}
\end{equation}
where the RHS uses \cref{def:distequiv1} on semicolon products.\footnote{When $\xx^{t}$ all have the same number of $n$-dimensions, then
$\left(\xx^{1},\ldots,\xx^{r}\right)\distequiv\left(\yy^{1},\ldots,\yy^{r}\right)$
can be interpreted as distributional equivalence ala \cref{def:distequiv1}
of two multi-tensors fromed from concatenation. But in general $\xx^{t}$
have different number of $n$-dimensions for different $t$, where
the semicolon product still applies.} Nevertheless, we explicitly formulated \cref{def:distequiv2} as it
is more transparent.

As discussed above, $(\zz,\xx)\distequiv(\zz',\yy)$ is strictly stronger
than the conjunction of $\zz\distequiv\zz'$ and $\xx\distequiv y$.
However, they are logically equivalent if we replace $\xx\distequiv y$
with the \emph{conditional dequivalence $\xx\distequiv_{R}\yy$, }as
defined below, where $R$ is the relation $\zz\distequiv\zz'$.
\begin{defn}[Conditional distributional equivalence]
The expression $R:\zz\distequiv\zz'$ means ``we give the relation
$\zz\distequiv\zz'$ the name $R$.''
In this case, we define $\xx\distequiv_{R}\yy$ to mean $(\zz,\xx)\distequiv(\zz',\yy)$.
This is called a \emph{conditional dequivalence}.

More generally, if $R:\zz\distequiv_{R'}\zz'$ is itself a conditional
equivalence, we recursively define $\xx\distequiv_{R}\yy$ to mean
$(\zz,\xx)\distequiv_{R'}(\zz',\yy)$.
Finally, when $R=\emptyset$ is the empty relation (equivalence of
nothing to nothing), we define $\xx\distequiv_{R}\yy$ as just the
unconditional dequivalence $\xx\distequiv\yy$.
\end{defn}

For example, we may recursively write $R_0: \xx_0 \distequiv \yy_0$ and $R_i: \xx_i \distequiv_{R_{i-1}} \yy_i$.
Then the conditional dequivalence $R_i$ can be unpacked into $(\xx_0, \ldots, \xx_i) \distequiv (\yy_0, \ldots, \yy_i)$.

\subsubsection{Wasserstein Distance Interpretation}
\label{sec:Wasserstein_dequiv}

For two multi-tensors $\xx, \yy$ of the same shape, thought of as sequences $\{\xx(n)\}_n, \{\yy(n)\}_n$, we may form a new sequence $W_2(\xx, \yy)$ of their Wasserstein distances $\{W_2(\calD_{\xx(n)}, \calD_{\yy(n)})\}_n$, where $W_2$'s underlying metric for each $n$ is the scaled Euclidean distance $d(\xbf, \ybf) = \sqrt{\la \|\xbf_\aalpha - \ybf_\aalpha\|^2\ra_\aalpha},$ and $\calD_{\xx(n)}, \calD_{\yy(n)}$ denote the measures of $\xx(n), \yy(n)$.
Then,
\[
  \text{morally,}\quad \xx \distequiv \yy \iff W_2(\xx, \yy) = \tilde O(n^{-1/2}).\]
This is not exactly rigorous because Wasserstein distance is defined using \emph{expected movement} while equivalence is defined using an \emph{almost sure} notion of ``vanishing.''
This is a technical issue that can be sidestepped if one were to define $\tilde O$ using convergence of expectations rather than almost sure.
Since we will not actually use this Wasserstein distance interpretation in this work, we will not attempt to make this more formal.
This connection is purely to aid the reader's intuitive understanding.

\subsubsection{Basic Properties}

Conditional dequivalence naturally inherits the basic
properties of equivalence.
Consider any conditional dequivalence named $R$. In summary, conditional
dequivalence is preserved under:
\begin{enumerate}
\item ``smooth'' mapping $\phi$: (\cref{prop:distequiv_pLip_tensor})
\begin{enumerate}
\item For example, if $R':\xx\distequiv_{R}\yy$ are all moment-bounded,
then $\phi(\xx)\distequiv_{R'}\phi(\yy)$. This holds more generally
for ``smooth'' nonlinear outer products.
\end{enumerate}
\item averaging over $[n]$: (\cref{lem:avg_distequiv})
\begin{enumerate}
\item If $R':\xx\distequiv_{R}\yy$, then $\la\xx_{\bullet\bbeta}\ra_{\bbeta}\distequiv_{R'}\la\yy_{\bullet\bbeta}\ra_{\bbeta}$.
\end{enumerate}
\item multiplication by operator-normed-bounded matrices: (\cref{prop:S0_matmul_dequiv})
\begin{enumerate}
\item If $x,y$ are $n$-vectors with $R':x\distequiv_{R}y$ and $W$ is
an $n$-matrix that almost surely has bounded operator norm, then
$Wx\equiv_{R'}Wy$. This holds in particular for $W$ having iid,
zero-mean entries of size $\Theta(1/\sqrt{n})$.
\end{enumerate}
\end{enumerate}
The formal statements are as follows. The proofs are all trivial given
the corresponding statements in \cref{sec:equiv}.
\begin{prop}
\label{prop:distequiv_pLip_tensor} Consider multi-tensors $\xx^{t},\yy^{t}\in\seqb^{s_{t}}\otimes\calV_{t}$
for each $t=1,\ldots,r$ and a function $\psi(-;\cdots;-):\calV_{1}\oplus\cdots\oplus\calV_{t}\to\calV'$.
Let $R$ be any dequivalence. If $\psi$ is pseudo-Lipschitz
and $R':(\xx^{1},\ldots,\xx^{r})\distequiv_{R}(\yy^{1},\ldots,\yy^{r})$,
then 
\[
\psi(\xx^{1};\cdots;\xx^{r})\distequiv_{R'}\psi(\yy^{1};\cdots;\yy^{r}).
\]
\end{prop}
\begin{lemma}
\label{lem:avg_distequiv}Suppose $\xx,\yy\in\seqb^{s}\otimes\calV$
satisfy $R':\xx\distequiv_{R}\yy$ for some dequivalence
$R$. Then, with $\bbeta$ denoting any subset of indices,
\[
\left\langle \xx_{\bullet\bbeta}\right\rangle _{\bbeta}\distequiv_{R'}\left\langle \yy_{\bullet\bbeta}\right\rangle _{\bbeta},
\]
where $\left\langle \xx_{\bullet\bbeta}\right\rangle _{\bbeta}$ is
the multi-tensor $\xx'$ with entries $\xx'_{\aalpha}=\left\langle \xx_{\aalpha\bbeta}\right\rangle _{\bbeta}$,
and likewise for $\left\langle \yy_{\bullet\bbeta}\right\rangle _{\bbeta}$.
\end{lemma}
\begin{prop}
  \label{prop:S0_matmul_dequiv}
  
  Given an $n$-matrix
  $W\in\seqb^{2}$ with almost surely bounded operator norm and $n$-vectors $x,y\in\seqb^{1}$ with $R: x\distequiv_W y$,%
  \footnote{$x\distequiv_W y$ means dequivalence conditional on the trivial equivalence $W \distequiv W$.} we have
  \[
  Wx\distequiv_R Wy.
  \]

  In particular, this applies when $W(n)=\frac{1}{\sqrt{n}}\iid_{n\times n}(Z)$
  where $Z$ is a random variable with zero mean, unit variance, and
  finite fourth moment.
\end{prop}

\section{Getting Equivalence From Distributional Equivalence}

Here we discuss several situations, in order of easy to hard, that allow us to ``upgrade'' dequivalence to equivalence.
The most important result (\cref{lem:decorrelate_avg}) in this section reads roughly as follows:
If $\xx$ is
iid along the $n$-dimensions, then $\la\phi(\xx;\xx_{\beta_{1}};\cdots;\xx_{\beta_{r}})\ra_{\bbeta}\equiv\la\phi(\xx;\xx_{\beta_{1}}^{\bx{1}};\cdots;\xx_{\beta_{r}}^{\bx{r}})\ra_{\bbeta}$
(where, recall, for each $i$, $\xx^{\bx{i}}$ is an iid copy of
$\xx$, such that $\xx,\xx^{\bx{1}},\ldots,\xx^{\bx{r}}$ are
mutually independent).

\paragraph{Trivial Conversion}
Sometimes, we can directly convert $\distequiv $ to $\equiv$:
\begin{prop}
  \label{prop:dequiv_is_equiv_for_deterministic}If $\xx\distequiv\yy$
  and, for each $n$, $\yy(n)$ is deterministic, then $\xx\equiv\yy$.
\end{prop}
\begin{proof}
This just follows from that fact that, for any random vectors $\XX$
and $\YY$, if $\XX\disteq\YY$ and $\YY$ is deterministic, then
$\XX\overset{\text{a.s.}}{=}\YY$.
\end{proof}

\paragraph{Easy Conversion}
In other times, two objects formed from dequivalent ingredients are actually (absolutely) equivalent.
This happens typically when we are averaging over many iid things, essentially because of the Law of Large Numbers.
The most basic example is:
\begin{lemma}
\label{lem:LLNequiv}Let $Z\in\R$ be a random variable for whom moments
of every order exists. Then, as elements of $\seqb^{0}$,
\[
\left\langle \iid(Z)_{\alpha}\right\rangle _{\alpha}\equiv\EV Z.
\]
\end{lemma}
Here $\iid(Z)$ is only well-defined up to distributional equality, so what we mean here is: whenever $z \distequiv \iid(Z)$, we have $\left\langle z_{\alpha}\right\rangle _{\alpha}\equiv\EV Z.$
In other words, if $z \distequiv z' \distequiv \iid(Z)$, then we in fact have $\la z_\alpha \ra_\alpha \equiv \la z'_\alpha\ra_\alpha$, not just $\distequiv$ as from \cref{lem:avg_distequiv}.
\begin{proof}
WLOG assume $\EV Z=0$. Then because $Z$ has moments of every order,
by \cref{lem:CLT_moments}, $Y\defeq\frac{1}{\sqrt{n}}\sum_{\alpha}\iid(Z)_{\alpha}$ has moments
of every order $p$ bounded by some constant $C_{p}$ independent
of $n$. To show $\left\langle \iid(Z)_{\alpha}\right\rangle _{\alpha}=\frac{1}{\sqrt{n}}Y$
is vanishing, we need to show that for every $\epsilon>0$, $n^{1-\epsilon}\left\langle \iid(Z)_{\alpha}\right\rangle _{\alpha}^{2}=n^{-\epsilon}Y^{2}\asto0$.
But $\EV|n^{-\epsilon}Y^{2}|^{p}\le C_{2p}n^{-\epsilon p}$ for every
$p>0$. In particular, taking $p>1/\epsilon$, we can apply \cref{lemma:momentBoundASConvergence}
to get the desired result.
\end{proof}
For example, we will frequently use the following corollary of this:
\begin{lemma}
\label{lem:inner_product_det_equiv}If $\xx\in\seqb^{1}\otimes\R^{k}$
and $\yy\in\seqb^{1}\otimes\R^{l}$ satisfy $(\xx,\yy)\distequiv\iid(\XX,\YY)$
for random vectors $\XX\in\R^{k},\YY\in\R^{l}$ with finite joint
moments of every order, then 
\[
\frac{1}{n}\xx^{\trsp}\yy\equiv\EV\XX\YY^{\trsp}\in\R^{k\times l}.
\]
In other words, the RHS has entries $\EV X^{i}Y^{j}$ for each $i\in[k],j\in[l]$.
\end{lemma}
\begin{proof}
The $(i,j)$th entry of $\frac{1}{n}\xx^{\trsp}\yy$ is just $\langle x_{\alpha}^{i}y_{\alpha}^{j}\rangle_{\alpha}\distequiv\langle\iid(X^{i}Y^{j})_{\alpha}\rangle_{\alpha}$
which is equivalent to $\EV X^{i}Y^{j}$ by \cref{lem:LLNequiv}. Then
we can upgrade this to equivalence by \cref{prop:dequiv_is_equiv_for_deterministic}:
$(\frac{1}{n}\xx^{\trsp}\yy)^{ij}\equiv\EV X^{i}Y^{j}$, as desired.
\end{proof}
\paragraph{Medium Conversion}
For our main theorem involving nonlinear outer products, we need the
following (much stronger) generalization of \cref{lem:LLNequiv}, which  roughly says $\la \psi(\xx; \iid(\ZZ)_\alpha)\ra_\alpha \equiv \EV_{\ZZ} \psi(\xx; \ZZ)$ in the simplest case. 
\begin{lemma}
\label{lem:LLNequiv_outerprod}Let $\xx\in\seqb^{s_{0}}\otimes\calV_{0}$
be moment-bounded. For each $t=1,\ldots,r$, let $\ZZ^{t}\in\calV_{t}$
be a (fixed-dimensional) random vector whose moments exist for all
orders. Assume $\ZZ^{1},\ldots,\ZZ^{r}$ are all mutually independent.
For each $t$, define $s_{t}\ge1$ and $\zz^{t}\in\seqspace^{s_{t}}\otimes\calV_{t}$
such that $\zz^{t}\disteq\iid^{s_{t}}(\ZZ^{t})$. Assume $\xx,\zz^{1},\ldots,\zz^{r}$
are mutually independent. Let $\psi:\calV_{0}\oplus\cdots\oplus\calV_{r}\to\calV'$
be a pseudo-Lipschitz function. Then, with $\bbeta_{t}\in[n]^{s_{t}}$
and $\bbeta=\bbeta_{1}\cdots\bbeta_{r}$,
\[
\left\langle \psi(\xx;\zz_{\bbeta_{1}}^{1};\cdots;\zz_{\bbeta_{r}}^{r})\right\rangle _{\bbeta}\equiv\Psi(\xx)\in\seqb^{s_{0}}\otimes\calV'
\]
where
\[
\Psi:\calV_{0}\to\calV',\quad\Psi(\xi)\defeq\EV\psi(\xi;\ZZ^{1};\cdots;\ZZ^{r}),
\]
where the expectation is taken over $\ZZ^{1},\ldots,\ZZ^{r}$.
\end{lemma}
As discussed above, this implies
\begin{equation}
  \left\langle \psi(\xx;\zz_{\bbeta_{1}}^{1};\cdots;\zz_{\bbeta_{r}}^{r})\right\rangle _{\bbeta} \equiv \left\langle \psi(\xx;\zz_{\bbeta_{1}}^{1\bx1};\cdots;\zz_{\bbeta_{r}}^{r\bx1})\right\rangle _{\bbeta},
  \label{eqn:iidswap}
\end{equation}
which is an upgrade over the $\distequiv$ from \cref{lem:avg_distequiv}.
\begin{proof}
WLOG assume $\calV'=\R$ since the general case follows from this
componentwise. Furthermore, we will assume $r=1$ because the general
case follows from this by induction. Let 
\[
\bar{\psi}(\xi;\zeta^{1})\defeq\psi(\xi;\zeta^{1})-\Psi(\xi).
\]
Then for any $\xi$, $\bar{\psi}(\xi;\ZZ^{1})$ has mean zero (over
the randomness of $\ZZ^{1}$). Let 
\[
y\defeq\left\langle \bar{\psi}(\xx;\zz_{\bbeta_{1}}^{1})\right\rangle _{\bbeta_{1}}\in\seqb^{s_{0}}.
\]
We want to show $y$ is vanishing. 

Notationally, we have usually written $\xx$ when we really mean $\xx(n)$,
the $n$th element of $\xx$ as sequence of tensors. This is not ambiguous
typically, but here we do want to talk about both semantics. Therefore,
in this proof, we will be explicit: for emphasis, we write $\xx[1,\infty)=\{x(n)\}_{n=1}^{\infty}$
for the sequence interpretation and still write just ``$\xx$''
for the $n$th element of it (with dependence on $n$ suppressed as
usual).

Now, we will prove $y$ is vanishing by showing that, with probability
1 on the distribution of $\xx[1,\infty)$, $y$ is vanishing conditioned
on $\xx[1,\infty)$.

By \cref{lem:CLT_moments} applied to $\bar{\psi}(\xx_{\aalpha};\zz_{\bbeta_{1}}^{1})$
with fixed $\aalpha$ (using the independence of $\zz^{1}$ from $\xx$),
we have 
\[
\EV[y_{\aalpha}^{2p}|\xx_{\aalpha}]\le D_{p}(\xx_{\aalpha})n^{-ps_{1}}\le D_{p}(\xx_{\aalpha})n^{-p}
\]
for some polynomially-bounded function $D_{p}:\calV_{0}\to\R$ independent
of $\aalpha$. (More specifically, $D_{p}(\xx_{\aalpha})$ is a polynomial
in the moments of the random variable $\psi(\xx_{\aalpha};\ZZ^{1})$
conditioned on $\xx_{\aalpha}$, which are themselves polynomially
bounded functions of $\xx_{\aalpha}$). Here we used the assumption
that $s_{1}\ge1$ in the 2nd inequality. Then
\[
\EV\left[\left\langle y_{\aalpha}^{2p}\right\rangle _{\aalpha}\mid\xx\right]=\left\langle \EV[y_{\aalpha}^{2p}|\xx_{\aalpha}]\right\rangle _{\aalpha}\le n^{-p}\left\langle D_{p}(\xx_{\aalpha})\right\rangle _{\aalpha}.
\]

Because $\xx$ is moment-bounded, so is $\left\langle D_{p}(\xx_{\aalpha})\right\rangle _{\aalpha}\in\seqspace^{0}$,
so that, for every $\varepsilon>0$,
\[
n^{-\varepsilon}\left\langle D_{p}(\xx_{\aalpha})\right\rangle _{\aalpha}\asto0.
\]
Thus, for almost every sequence $\xx[1,\infty)$, 
\[
n^{-\varepsilon}\left\langle D_{p}(\xx_{\aalpha})\right\rangle _{\aalpha}\le C_{p,\varepsilon}(\xx[1,\infty))\quad\text{for all \ensuremath{n}}
\]
for some constant $C_{p,\varepsilon}(\xx[1,\infty))$ dependent on
the whole sequence $\xx[1,\infty)$. Then
\[
n^{-\varepsilon}\EV\left[\left\langle y_{\aalpha}^{2p}\right\rangle _{\aalpha}\mid\xx[1,\infty)\right]\le C_{p,\varepsilon}(\xx[1,\infty))n^{-p}\quad\text{for all \ensuremath{n}}
\]
satisfying \cref{prop:powermean_implies_vanishing}, implying $y$
is vanishing conditioned on $\xx[1,\infty)$. Since this argument
holds for every $\xx[1,\infty)$ and every $\varepsilon>0$, we have
the desired result.
\end{proof}

\paragraph{Advanced Conversion}
In the most advanced case, we need \cref{eqn:iidswap} to hold even when the $\xx$ and $\zz^i$ are strongly correlated.
This will turn out to be the most important case, since it forms the induction base for our main theorem.
\begin{lemma}[IID Equivalence]
\label{lem:decorrelate_avg}Let $\XX\in\R^{k}$ be a random vector
whose moments exist for all orders. Let $\xx\in\seqb^{1}\otimes\R^{k}$
satisfy $\xx\distequiv\iid(\XX)$. Let $\cc\in\seqb^{0}\otimes\R^{l}$
be such that $\cc\equiv\mathring{\cc}$ for some deterministic $\mathring{\cc}\in\R^{l}$.

Let $\psi:(\oplus^{s+r}\R^{k})\oplus\R^{l}\to\R$ be pseudo-Lipschitz.
Then, with $\bbeta=(\beta_{1},\ldots,\beta_{r})$,%
\footnote{The iid copies $\xx^{\bx 1}, \ldots, \xx^{\bx r}$ are only well-defined up to distributional equality, so \cref{eq:ind_nonlin_tensor} should be read as: for \emph{any} instantiation of $\xx^{\bx 1}, \ldots, \xx^{\bx r}$, the equivalence holds.}
\begin{equation}
\langle\psi(\xx;\cdots;\xx;\xx_{\beta_{1}};\cdots;\xx_{\beta_{r}};\cc)\rangle_{\bbeta}\equiv\langle\psi(\xx;\cdots;\xx;\xx_{\beta_{1}}^{\bx{1}};\cdots;\xx_{\beta_{r}}^{\bx{r}};\mathring{\cc})\rangle_{\bbeta}\label{eq:ind_nonlin_tensor}
\end{equation}
as elements of $\seqb^{s}$ (where, recalling \cref{defn:bxiid}, $\xx^{\bx{1}},\ldots,\xx^{\bx{r}}$
are iid copies of $\iid(\XX)$, independent from $\xx$).
\end{lemma}
We in fact don't need to assume $\mathring{c}$ is deterministic,
but doing so simplifies the proof.

We remind the reader the notations in \cref{eq:ind_nonlin_tensor}.
To be very explicit, if $s=1$ and $T=\langle\psi(\xx;\xx_{\beta_{1}};\cdots;\xx_{\beta_{r}};\cc)\rangle_{\bbeta}\in\seqb^{1}$,
then each entry $T_{\alpha}$ is
\begin{equation}
T_{\alpha}=\frac{1}{n^{r}}\sum_{\beta_{1},\ldots,\beta_{r}}\psi(x_{\alpha}^{1},\ldots,x_{\alpha}^{k}\sep x_{\beta_{1}}^{1},\ldots,x_{\beta_{1}}^{k}\sep\cdots\sep x_{\beta_{r}}^{1},\ldots,x_{\beta_{r}}^{k}\sep c^{1},\ldots,c^{l}).\label{eq:expand_out_avg}
\end{equation}
Likewise for $\langle\psi(\xx;\xx_{\beta_{1}}^{\bx{1}};\cdots;\xx_{\beta_{r}}^{\bx{r}};\mathring{\cc})\rangle_{\bbeta}$.

Intuitively, when all the $\beta_{1},\ldots,\beta_{r}$ are distinct,
then $\xx_{\beta_{1}},\ldots,\xx_{\beta_{r}}$ are roughly independent
as well, but this is obviously not the case when $\beta_{1},\ldots,\beta_{r}$
are not distinct. However, among all possible values of the tuple
$(\beta_{1},\ldots,\beta_{r})$, the nondistinct ones constitute a
minority, vanishing with $n$. Therefore, we hope to say that they
contribute vanishingly to the sum \cref{eq:expand_out_avg} so as to
establish \cref{eq:ind_nonlin_tensor}.

By combining with \cref{lem:LLNequiv_outerprod}, we get
\begin{lemma}
\label{lem:outer_power_avg_equiv_nonlin}In the same scenario as \cref{lem:decorrelate_avg},
if $s=1$, then
\[
\langle\psi(\xx;\xx_{\beta_{1}};\cdots;\xx_{\beta_{r}};\cc)\rangle_{\bbeta}\equiv\Psi(\xx)
\]
where
\[
\Psi:\R^{k}\to\R,\quad\Psi(\xbf)\defeq\EV_{\bx{1},\ldots,\bx{r}}\psi(\xbf;\XX^{\bx{1}};\cdots;\XX^{\bx{r}};\mathring{\cc}),\forall\xbf\in\R^{k},
\]
(with the expectation taken over $\XX^{\bx{1}};\cdots;\XX^{\bx{r}}$).
\end{lemma}

\subsection{Proof of \texorpdfstring{\cref{lem:decorrelate_avg}}{IID Equivalence}}

\subsubsection{Preliminary Reduction and Setup}

We assume in the proof that $s=1$, as the general case is a straightforward
modification.

WLOG, we can assume $\xx$ is an iid copy of $\iid(\XX)$ by \cref{prop:equiv_pLip_tensor}.
Applying \cref{prop:equiv_pLip_tensor} again to $\cc$, we see 
\[
\psi(\xx;\xx;\cdots;\xx;\cc)\equiv\psi(\xx;\xx;\cdots;\xx;\mathring{\cc}).
\]
So we can assume WLOG $\cc=\mathring{c}$ and furthermore just absorb
it into $\psi$ since $\mathring{\cc}$ is deterministic. It remains
to show
\[
\langle\psi(\xx;\xx_{\beta_{1}};\cdots;\xx_{\beta_{r}})\rangle_{\bbeta}\equiv\langle\psi(\xx;\xx_{\beta_{1}}^{\bx{1}};\cdots;\xx_{\beta_{r}}^{\bx{r}})\rangle_{\bbeta}.
\]

Let $\Delta\in\seqb^{r+1}$ be the tensor 
\[
\Delta_{\alpha\bbeta}=\Delta_{\alpha\beta_{1}\cdots\beta_{r}}=\psi(\xx_{\alpha};\xx_{\beta_{1}};\cdots;\xx_{\beta_{r}})-\psi(\xx_{\alpha};\xx_{\beta_{1}}^{\bx{1}};\cdots;\xx_{\beta_{r}}^{\bx{r}}).
\]
Note the following properties of $\Delta$: 
\begin{prop}
\label{prop:_Uproperty}1) If $\alpha$ does not appear in $\bbeta$
and all indices in $\bbeta$ are distinct, then $\EV\Delta_{\alpha\bbeta}=0$.
2) If $\bbeta'$ is another multi-index with the same property and
that furthermore does not intersect $\bbeta$, then $\Delta_{\alpha\bbeta}$
and $\Delta_{\alpha\bbeta'}$ are independent conditioned on $\xx_{\alpha}$.
\end{prop}
Then our goal is to show $\left\langle \Delta_{\bullet\bbeta}\right\rangle _{\bbeta}\in\seqb^{1}$
is vanishing, i.e., for any $\epsilon>0$,
\[
n^{1-\epsilon}\left\langle \left\langle \Delta_{\alpha\bbeta}\right\rangle _{\bbeta}^{2}\right\rangle _{\alpha}\asto0.
\]

By \cref{prop:powermean_implies_vanishing}, we just need to show there
exist constants $C_{p}$ for all integers $p\ge1$ such that 

\begin{equation}
\EV\left\langle \left\langle \Delta_{\alpha\bbeta}\right\rangle _{\bbeta}^{2p}\right\rangle _{\alpha}\le C_{p}n^{-p}\label{eq:moment_goal}
\end{equation}
for all $n$ and $p$. 

Fix $n$ and $p$. Below, we shall bound each $\left\langle \Delta_{\alpha\bbeta}\right\rangle _{\bbeta}^{2p}$
individually. By symmetry in $\alpha$, we WLOG fix $\alpha=n$. Let
$n'$ be the largest multiple of $r$ smaller than $n$. 

\subsubsection{The Core Insight}

We succinctly overview the core insight here, with more explanation
below. There are 3 steps of reasoning:

1. For large $n$, $\sum_{\bbeta}\Delta_{\alpha\bbeta}$ is dominated
by the subset sum $\sum_{\bbeta\in\Gamma}\Delta_{\alpha\bbeta}$ where
$\Gamma\sbe[n']^{r}$ consists of all $\bbeta$ that do not contain
repeating indices (i.e., $\beta_{1},\ldots,\beta_{r}$ are all distinct).
(Note such $\bbeta$ does not contain $\alpha$ either since $\alpha=n>n'$
by assumption).

2. But by Ordered Baranyai's Theorem (\cref{thm:obara}), there is a partition
of $\Gamma$ into perfect matchings $\Gamma_{1},\Gamma_{2},\ldots$ i.e., such that each $\Gamma_{i}$
consists of multi-indices $\Gamma_{i}=\{\bbeta^{1},\ldots,\bbeta^{n'/r}\}$
and $\bbeta^{1},\ldots,\bbeta^{n'/r}$ partition $[n']$ (i.e.,
each $\beta\in[n']$ appear in exactly one of $\bbeta^{j}$). 
For example, if $n' = 4$ and $r = 2$, then $\Gamma = \{(1,2), (1,3), (1,4), (2,3), (2,4), (3,4)\} \cup \{\text{mirror image}\}$.
This is partitioned into the perfect matchings $\{(1,2), (3,4)\} \sqcup \{(1,3), (2,4)\} \sqcup \{(1,4), (2,3)\}$ along with their mirror images.
Back to the general case, we thus have
\begin{equation}
\sum_{\bbeta\in\Gamma}\Delta_{\alpha\bbeta}=\sum_{i}\sum_{\bbeta\in\Gamma_{i}}\Delta_{\alpha\bbeta}.\label{eq:applyBaranyai}
\end{equation}

3. But, conditioned on $\xx_{\alpha}$, each $\sum_{\bbeta\in\Gamma_{i}}\Delta_{\alpha\bbeta}$
is a sum of independent, mean-zero random variables by \cref{prop:_Uproperty}.
Therefore it has typical size $O(\sqrt{n'/r})$ which is just $O(\sqrt{n})$
because $r$ is constant. As a result, the whole sum \cref{eq:applyBaranyai}
is of order $n^{r-1/2}$ so that
\[
n^{-r}\sum_{\bbeta\in\Gamma}\Delta_{\alpha\bbeta}=O(n^{-1/2}).
\]
This will be formalized as moment bounds 
\[
\EV\left\langle \Delta_{\alpha\bbeta}\right\rangle _{\bbeta}^{2p}\le C_{p}n^{-p}
\]
from which \cref{eq:moment_goal} would follow.

Of the 3 steps above, only 1 and 3 need further explanation. We do
so below.

\subsubsection{Filling in the Details}

\paragraph{1. Sum Decomposition}

Let $R_{1}$ be the set of $\bbeta$ that contains at least one of
$n'+1,\ldots,n$. Let $R_{2}$ be the set of $\bbeta$ that contains
repeating indices. Then the complement of $\Gamma$ in $[n]^{r}$
is a subset of $R_{1}\cup R_{2}$.

Now note that $|R_{1}\cup R_{2}|=\Theta(n^{r-1})$ (where the hidden
constant can depend on $r$ since $r=\Theta(1)$). Furthermore, because
$\XX$ (the distribution of $\xx$ and $\yy$'s $\alpha$-slices)
has all moments and $\psi$ is pseudo-Lipschitz and thus polynomially
bounded, there exist constants $D_{q}$ for every $q\ge1$ such that
for all $\alpha,\bbeta$, we have $\EV|\Delta_{\alpha\bbeta}|^{q}\le D_{q}=O(1)$.
Therefore, we have

\[
\EV(\sum_{\bbeta\not\in\Gamma}\Delta_{\alpha\bbeta})^{2p}\le|R_{1}\cup R_{2}|^{2p-1}\EV\sum_{\bbeta\not\in\Gamma}\Delta_{\alpha\bbeta}^{2p}=O(n^{(r-1)(2p-1)+(r-1)})=O(n^{(r-1)2p}).
\]
Thus, after multiplying by $(n^{-r})^{2p}$, we get
\[
\EV\left\langle \Delta_{\alpha\bbeta}\ind_{\bbeta\not\in\Gamma}\right\rangle _{\bbeta}^{2p}=O(n^{-2p}).
\]

Consequently,
\begin{align*}
\EV\left\langle \Delta_{\alpha\bbeta}\right\rangle _{\bbeta}^{2p} & =\EV\left\langle \Delta_{\alpha\bbeta}\ind_{\bbeta\in\Gamma}+\Delta_{\alpha\bbeta}\ind_{\bbeta\not\in\Gamma}\right\rangle _{\bbeta}^{2p}\\
 & \le C\EV\left\langle \Delta_{\alpha\bbeta}\ind_{\bbeta\in\Gamma}\right\rangle _{\bbeta}^{2p}+C\EV\left\langle \Delta_{\alpha\bbeta}\ind_{\bbeta\not\in\Gamma}\right\rangle _{\bbeta}^{2p}\\
 & \le C\EV\left\langle \Delta_{\alpha\bbeta}\ind_{\bbeta\in\Gamma}\right\rangle _{\bbeta}^{2p}+O(n^{-2p})
\end{align*}
where $C=2^{2p-1}$ (by \cref{lem:powerbound}). Thus, to show \cref{eq:moment_goal},
it suffices to show 
\[
\EV\left\langle \Delta_{\alpha\bbeta}\ind_{\bbeta\in\Gamma}\right\rangle _{\bbeta}^{2p}\le C_{p}n^{-p}
\]
(for a different set of constants $C_{p}$).

\paragraph{3. Bounding Moments}

Let $\left\langle -\right\rangle _{\bbeta\in\Gamma}$ denote average
over $\bbeta\in\Gamma$. Recall that, by Ordered Baranyai's Theorem, we have
partitioned $\Gamma$ into $\Gamma_{1},\Gamma_{2},\ldots$ where each
$\Gamma_{i}$ consists of multi-indices $\Gamma_{i}=\{\bbeta^{1},\ldots,\bbeta^{n'/r}\}$
such that $\bbeta^{1},\ldots,\bbeta^{n'/r}$ partition $[n']$. Consequently,
for each $i$, $\{\Delta_{\alpha\bbeta}:\bbeta\in\Gamma_{i}\}$ is
mutually independent conditioned on $\xx_{\alpha}$. We first deduce

\[
\EV\left\langle \Delta_{\alpha\bbeta}\ind_{\bbeta\in\Gamma}\right\rangle _{\bbeta}^{2p}\le\EV\left\langle \Delta_{\alpha\bbeta}\right\rangle _{\bbeta\in\Gamma}^{2p}=\EV\left\langle \left\langle \Delta_{\alpha\bbeta}\right\rangle _{\bbeta\in\Gamma_{i}}\right\rangle _{i}^{2p}\le\left\langle \EV\left\langle \Delta_{\alpha\bbeta}\right\rangle _{\bbeta\in\Gamma_{i}}^{2p}\right\rangle _{i}
\]
where the equality follows from \cref{eq:applyBaranyai} and the last
inequality follows from power-mean inequality. Again, conditioned
on $\xx_{\alpha}$, each $\left\langle \Delta_{\alpha\bbeta}\right\rangle _{\bbeta\in\Gamma_{i}}$
is a mean of independent, mean-zero random variables by \cref{prop:_Uproperty}.
But by \cref{lem:CLT_moments}, there are real polynomially bounded
functions $F_{p}$ of $\xx_{\alpha}$ for each integer $p$ such that
$\EV\left[\left\langle \Delta_{\alpha\bbeta}\right\rangle _{\bbeta\in\Gamma_{i}}^{2p}\mid\xx_{\aalpha}\right]<F_{p}(\xx_{\alpha})(n'/r)^{-p}$
simultaneously for every $i$. \footnote{Specifically, if $H_{\alpha}$ is the random variable distributed
identically as $\Delta_{\alpha\bbeta},\bbeta\in\Gamma$ (which have
the same distribution for any $\bbeta\in\Gamma$), then by \cref{lem:CLT_moments},
$F_{p}$ is a polynomial in the moments of $H_{\alpha}$ conditioned
on $\xx_{\alpha},$and these moments are themself polynomially bounded
in $\xx_{\alpha}$ because $\psi$ is polynomially bounded.} Therefore, altogether,
\[
\EV\left\langle \Delta_{\alpha\bbeta}\ind_{\bbeta\in\Gamma}\right\rangle _{\bbeta}^{2p}\le C_{p}n^{-p}
\]
where $C_{p}=2^{p}r^{p}\EV F_{p}(\xx_{\alpha})=2^{p}r^{p}\EV F_{p}(\XX)$
(which is finite because $\XX$ has all moments, and where $2^{p}\ge(n/n')^{p}$
is there as a loose conversion factor from $n'$ to $n$). 

\section{Matrix Pseudo-Inverse}
\begin{lemma}
\label{lem:matinv}Suppose\footnote{i.e., $\Lambda(n)$ is a random $k\times k$ matrix for each $n$;
c.f. \cref{def:tensorseq}} $\Lambda\in\seqb^{0}\otimes\R^{k\times k}$ such that $\Lambda\equiv\mathring{\Lambda}$
for a deterministic $\mathring{\Lambda}\in\R^{k\times k}$.\footnote{Recall this means the random sequence $\Lambda$ is equivalent to
the sequence that equals $\mathring{\Lambda}$ identically; c.f. \cref{def:constseq}.} If $\mathring{\Lambda}$ is full rank, then 
\[
\Lambda^{+}\equiv\mathring{\Lambda}^{-1}
\]
 as well.
\end{lemma}
\begin{proof}
Apply \cref{prop:localLip_equiv} to $\psi$ being the matrix pseudo-inverse
function, which is locally Lipschitz at any full rank matrix.
\end{proof}

Earlier we defined spaces like $\seqspace^{1}\otimes\R^{k}$ which
roughly contains sequences (in $n$) of $n\times k$ matrices. Likewise,
we can define $\R^{k}\otimes\seqspace^{1}$ which contain the transposes
of $\seqspace^{1}\otimes\R^{k}$, i.e., $k\times n$ matrices.
\begin{lemma}
\label{lem:_pinv_equiv}Suppose $\XX\in\R^{k}$ is a random vector
and $\xx\distequiv\iid(\XX)\in\seqb^{1}\otimes\R^{k}$.\footnote{i.e., $X$ is a sequence of random $\R^{n\times k}$ matrices, distributionally
equivalent to $\iid(\XX)$; c.f. \cref{def:tensorseq}} If $\XX$ has moments of all orders and its covariance $\Lambda\defeq\EV\XX\XX^{\trsp}\in\R^{k\times k}$
(which is deterministic, independent of $n$) is full rank, then as
elements of $\seqb^{1}\otimes\R^{k}$,\footnote{Recall, as discussed in \cref{def:tensorseq}, $\seqb^{1}\otimes\R^{k}$
means sequence in $n$ of moment-bounded $n\times k$ matrices.}
\begin{equation}
n\xx^{+\trsp}\equiv\xx\Lambda^{-1}\distequiv n\ \iid(\XX)^{+\trsp}\distequiv\iid(\XX)\Lambda^{-1}.\label{eq:pinv_equiv}
\end{equation}
Here $n\xx^{+\trsp}$ is the sequence $n\mapsto n\xx(n)^{+\trsp}$,
and likewise for $n\ \iid(\XX)^{+\trsp}$; $\xx\Lambda^{-1}$ at index
$n$ is the matrix multiplication between $\xx(n)$ (shape $n\times k$)
and $\Lambda^{-1}$ (shape $k\times k$).
\end{lemma}
Note that $n\xx^{+\trsp}$ is the correct scaling with $n$ so that
this matrix has $\Theta(1)$ sized entries (it is elucidating to consider
the example when $k=1$ and $\xx$ is just an $n$-vector).
\begin{proof}
Let $\tilde{\xx}$ be an instance of $\iid(\XX)$ such that $\xx\equiv\tilde{\xx}$.

By standard pseudo-inverse facts (\cref{sec:pinv}), $n\xx^{+}=(\frac{1}{n}\xx^{\trsp}\xx)^{+}\xx^{\trsp}.$
By \cref{lem:avg_equiv}, $\frac{1}{n}\xx^{\trsp}\xx\equiv\frac{1}{n}\tilde{\xx}^{\trsp}\tilde{\xx}$.
By \cref{lem:LLNequiv}, $\frac{1}{n}\tilde{\xx}^{\trsp}\tilde{\xx}\equiv\EV\XX\XX^{\trsp}=\Lambda$.
By \cref{lem:matinv} and the nonsingularity of $\Lambda$, we have
\begin{equation}
\left(\frac{1}{n}\xx^{\trsp}\xx\right)^{+}\equiv\Lambda^{-1}.\label{eq:XTX_equiv_Lambdainv}
\end{equation}
Therefore, since equivalence is preserved under matrix multiplication
(which is certainly pseudo-Lipschitz; see \cref{prop:equiv_pLip_tensor}),
we have
\[
n\xx^{+\trsp}=\xx(\frac{1}{n}\xx^{\trsp}\xx)^{+}\equiv\xx\Lambda^{-1}\distequiv\tilde{\xx}\Lambda^{-1}\disteq\iid(\XX)\Lambda^{-1}
\]
as desired.\footnote{Technically, we need to check that $\Lambda^{-1}$ and $\tilde{\xx}^{\trsp}$
are both moment-bounded for \cref{prop:equiv_pLip_tensor} to apply.
But this check passes trivially.}
\end{proof}
\begin{lemma}
\label{lem:pinv_equiv}In the setting of \cref{lem:_pinv_equiv}, suppose
there is additionally $\yy\in\seqb^{1}\otimes\R^{l}$ such that $(\xx,\yy)\distequiv\iid(\XX,\YY)$
(as sequences in $\seqb^{1}\otimes(\R^{k}\oplus\R^{l})$) for some
(fixed wrt $n$) random vector $\YY\in\R^{l}$. Then, as elements
of $\seqb^{0}\otimes\R^{k\times l}$, \footnote{Recall $\seqb^{0}\otimes\R^{k\times l}$ means sequence in $n$ of
moment-bounded $k\times l$ matrices; c.f. \cref{def:tensorseq}}
\begin{equation}
\xx^{+}\yy\equiv\Lambda^{-1}\Gamma\label{eq:XpinvY}
\end{equation}
where $\Gamma\in\R^{k\times l}$ is the deterministic matrix with
entries $\Gamma^{ij}\defeq\EV X^{i}Y^{j}$.\footnote{Note the RHS is just a deterministic $\R^{k\times l}$ matrix, representing
the constant sequence as in \cref{def:constseq}.}
\end{lemma}
\begin{proof}
let $(\tilde{\xx},\tilde{\yy})$ be an instance of $\iid(\XX,\YY)$
such that $(\xx,\yy)\equiv(\tilde{\xx},\tilde{\yy})$. Then by standard pseudo-inverse facts (\cref{sec:pinv}),
\[
\xx^{+}\yy=(\frac{1}{n}\xx^{\trsp}\xx)^{-1}(\frac{1}{n}\xx^{\trsp}\yy)\equiv\Lambda^{-1}(\frac{1}{n}\tilde{\xx}^{\trsp}\tilde{\yy})\equiv\Lambda^{-1}\Gamma
\]
where the first equivalence follows from \cref{eq:XTX_equiv_Lambdainv}
and \cref{prop:equiv_pLip_tensor}, and the last equivalence follows
from \cref{lem:LLNequiv}.
\end{proof}

\section{Uniformized Tensor Programs}

Here, we ``uniformize'' different variations of Tensor Programs, in the sense that we squash all the different instructions into a single step.
This allows us to streamline the induction in our proofs.

\subsection{Uniformized \netsort{}}
\label{sec:unetsort}

\newcommand\na{L}%

Given initial matrices $A^{1},\ldots,A^{\na}\in\RR^{n\times n}$ and initial
vectors $g^{1},\ldots,g^{M_{0}}\in\RR^{n}$, consider the following
iteration for $i=M_{0}+1,\ldots,M$ that generates new vectors $g^{M_{0}+1},\ldots,g^{M}\in\RR^{n}$:
\begin{equation}
g_{\alpha}^{i}\gets\sum_{\beta=1}^{n}W_{\alpha\beta}^{i}x_{\beta}^{i},\quad\text{where }x_{\alpha}^{i}=\phi^{i}(g_{\alpha}^{1},\ldots,g_{\alpha}^{i-1}).\label{eq:tp_iteration}
\end{equation}
Here each $\phi^{i}$ is a chosen scalar function with $(i-1)$ arguments
and $W^{i}$ is an $n\times n$ matrix. Each matrix $W^{i}$ equals
to either some matrix $A^{j}$ of the program or its transpose $A^{j\trsp}$.
The matrices $W^{i}$ for different $i$ can possibly be the same.
Thus each program is entirely determined by the data $\{A^{j}\}_{j=1}^{\na}\cup\{g^{i}\}_{i=1}^{M_{0}}\cup\{\phi^{i}\}_{i=M_{0}+1}^{M}$
along with the correspondence between $W^{i}$ and $A^{j}$ or $A^{j\trsp}$.

This type of program is obviously a subset of \netsort{} (\cite{yang3}) but it's also easy to see that they have the same expressive power.

\subsection{Uniformized \nexort{}}

Now extend the previous formulation to include additionally initial
scalars $c^{1},\ldots,c^{M_0}\in\R$ and extend \cref{eq:tp_iteration}
to produce both a new vector $g^{i}$ and a new scalar $c^{i}$ at
each step $i=M_{0}+1,\ldots,M$:

\begin{align}
g_{\alpha}^{i} & \gets\sum_{\beta=1}^{n}W_{\alpha\beta}^{i}x_{\beta}^{i},\quad c^{i}\gets\frac{1}{n}\sum_{\beta=1}^{n}x_{\beta}^{i},\label{eq:tensor_tp_iteration}\\
\text{where }x_{\alpha}^{i} & =\frac{1}{n^{i}}\sum_{\bbeta}\phi^{i}(\gb_{\alpha};\gb_{\beta};\cdots;\gb_{\beta_{i}};\cc)\\
 & =\langle\phi^{i}(\gb_{\alpha};\gb_{\beta_{1}};\cdots;\gb_{\beta_{i}};\cc)\rangle_{\bbeta} \label{eqn:xi_unexort}
\end{align}
\\
where $\phi^{i}:\underbrace{\R^{i-1}\oplus\cdots\oplus\R^{i-1}}_{i+2}\to\R$,
each $\bbeta=(\beta_{1},\ldots,\beta_{i})$ ranges over all tuples
in $[n]^{i}$, and $\gb=(g^{1},\ldots,g^{i-1}),\cc=(\cc^{1},\ldots,\cc^{i-1})$
at iteration $i$.

Here we are tying together the tensor order of $\phi^{i}$ with the
iteration index $i$, to simplify the formulation. But note that $\phi^{i}$
can always ignore all but the first block of inputs, for example,
to replicate \cref{eq:tp_iteration}. By the same reasoning as \cref{sec:unetsort},
this formulation of \nexort{} is equivalent to \cref{defn:nexort}.

\subsection{Setup and Constructions}

Below, we typically consider the following setup for both uniformized \netsort{} and uniformized \nexort{}.
It is essentially the same as \cref{setup:nexort} except that the initial scalars need to converge at the rate of $\tilde O(n^{-1/2})$.
This is needed to prove the vectorwise convergence results below.
\begin{setup}\label{setup:uniformized_tp}
  Assume
  \begin{enumerate}
    \item Every entry of every $A^l$ is sampled iid from $\Gaus(0,1/n)$.
    \item Every entry of every initial vector $g^i$ is sampled iid from $\Gaus(0, 1)$.
    \item Each initial scalar $c^i$ is $\tilde O(n^{-1/2})$. %
    \item All nonlinearities $\phi^i$ are pseudo-Lipschitz.
  \end{enumerate}
\end{setup}

In addition,
since both uniformized \netsort{} and uniformized \nexort{} are subsets of the \nexort{} language, the construction of kets in \cref{defn:ket} make sense for both of them.
Finally, in the context of $n\to\infty$, we think of the program's objects as sequences in $n$, i.e., $A^{j}\in\seqspace^{2}$, $g^{i}\in\seqspace^{1}$, $c^i \in \seqspace^0$.

\section{\netsort{} Master Theorem, Vectorwise Convergence}
\begin{defn}
Consider a Tensor Program (of any variation). Let $x^{1},\ldots,x^{k}$ be the set of all vectors
in the program. Then we define $(\mathring{x}^{1},\ldots,\mathring{x}^{k})$
as $\iid(\ket{x^{1}},\ldots,\ket{x^{k}})$, i.e., each $\alpha$-slice
$(\mathring{x}_{\alpha}^{1},\ldots,\mathring{x}_{\alpha}^{k})$ is
an iid sample from $(\ket{x^{1}},\ldots,\ket{x^{k}})$.%
\footnote{Note that
$(\mathring{x}^{1},\ldots,\mathring{x}^{k})$ is only well defined
up to distributional equality, and thus it only makes sense to talk
about them up to distributional equality or coarser notions like (conditional)
dequivalence, but not up to equivalence.}
\end{defn}

Both $\mathring x$ and $\ket x$ are ``limits'' of a vector $x$ in some sense, but $\mathring x$ always has the same shape as $x$ (same for $\mathring \theta$ for scalar $\theta$) while $\ket x$ collapses the $n$-dimension.
In this and the next section, we aim to show that $x \distequiv \mathring x$ in different kinds of Tensor Programs.
This essentially means that $x$ and $\mathring x$, random vectors with larger and larger dimensions, become closer and closer in distribution as $n\to\infty$ (see \cref{sec:Wasserstein_dequiv}) ---
hence the \emph{vectorwise convergence} in the section title.

\begin{thm}
\label{thm:scalarless_vMT}
Consider a \netsort{} program in \cref{setup:uniformized_tp}.
Consider any collection of vectors
$y^{1},\ldots,y^{k}$ in the program. Then they are moment-bounded and
\[
(y^{1},\ldots y^{k})\distequiv(\mathring{y}^{1},\ldots\mathring{y}^{k}).
\]
\end{thm}
In other words, $(y^{1},\ldots y^{k})$ is distributed like $n$ iid
copies of $(\ket{y^{1}},\ldots\ket{y^{k}})$, modulo vanishing vectors.
This is a stronger result than \cite{yang3}, at the cost of assuming faster convergence of initial scalars in \cref{setup:uniformized_tp}.

\begin{rem}
  Let us comment that, in this proof, we do not use the property that
  equivalence is preserved under iid matrix multiplication \cref{prop:S0_matmul} (this property is a crucial reason underlying the definition of ``vanishing''), because of \netsort{}'s rank stability property.
  This however will be crucial in \cref{sec:nexort_vectorwise_convergence} to prove the analogous result for \nexort{}.
\end{rem}

We recall some terminology from prior works.
\begin{defn}
  A vector is called a G-var if it is an initial vector or generated by \refMatmul{}.
\end{defn}

\subsection{Proof Setup}

\begin{table}
  \centering
\begin{tabular}{cc}
  \toprule 
  \cite{yang3} & Here\tabularnewline
  \midrule 
  $x^{i}$ & $g^{i}$\tabularnewline
  
  $y^{i}$ & $x^{i}$\tabularnewline
  
  $u^{j}$ & $h^{j}$\tabularnewline
  
  $v^{j}$ & $y^{j}$\tabularnewline
  
  $h$ & $x$\tabularnewline
  
  $g$ & $g$\tabularnewline
  
  $\Upsilon$ & $\xx^{\trsp}\xx/n$\tabularnewline
  
  $\mathring{\Upsilon}$ & $\braket{\xx}{\xx}$\tabularnewline
  
  $\Lambda$ & $\yy^{\trsp}\yy/n$\tabularnewline
  
  $\mathring{\Lambda}$ & $\braket{\yy}{\yy}$\tabularnewline
  
  $\Gamma$ & $\hh^{\trsp}\xx/n$\tabularnewline
  
  $\mathring{\Gamma}$ & $\braket{\hh}{\xx}$\tabularnewline
  
  $\gamma$ & $\xx^{\trsp}x/n$\tabularnewline
  
  $\mathring{\gamma}$ & $\braket{\xx}x$\tabularnewline
  
  $\delta$ & $\hh^{\trsp}x/n$\tabularnewline
  
  $\mathring{\delta}$ & $\braket{\hh}x$\tabularnewline
  \bottomrule 
\end{tabular}
\caption{Notation mapping between \cite{yang3} and here.
We improved the notation to be more intuitive, leveraging our new bra-ket notation.}
\label{tab:notationTP3vsHere}
\end{table}

We will nontrivially leverage parts of the Master Theorem proof in \cite{yang3}.
Our notation differs from there (hopefully improved for first time readers), but a mapping is provided in \cref{tab:notationTP3vsHere}.
By \cref{prop:distequiv_pLip_tensor}, it suffices to show this for all G-vars of the program, since other vectors are pseudo-Lipschitz images of G-vars. We furthermore WLOG assume the formulation in \cref{eq:tp_iteration}
and proceed to show this for vectors $g^{1},\ldots,g^{M}$ (all G-vars
in the program).

By the \netsort{} Master Theorem \cite{yang3}, we know that the matrices
$(\frac{1}{n}g^{i\trsp}g^{j})_{i,j=1}^{M}$ and $(\frac{1}{n}x^{i\trsp}x^{j})_{i,j=1}^{M}$
converge almost surely to deterministic matrices $\Omega\in\R^{M\times M}$
and $\Xi\in\R^{M\times M}$. By rank stability \citep[sec L.5]{yang3},
WLOG, we can assume 
\begin{assm}
\label{assm:full_rank} WLOG, we assume $\Omega$ and $\Xi$
are both full rank.
\end{assm}
Indeed, \cite[sec L.5]{yang3} says that $\Xi$ (resp.
$\Omega$) is singular iff there are linear dependencies (with constant
coefficients) between $x^{i}$s (resp. $g^{i}$s), which we can get
rid of by rewriting the program in the obvious way.%
\footnote{Rank stability is a highly nontrivial but technical result of \cite{yang3}. While seemingly small, it allows us to drastically simplify the proof because we do not have to think about ``corner cases'' where the rank of a matrix drops suddenly in the limit, which can lead to all kinds of pathological behaviors.
Readers interested in full rigor should consult \cite{yang3} for the proof of this result.}

We induct on $i$, starting from the base case $i=M_{0}$.

\paragraph{Base Case: $i=M_{0}.$ }

This holds by \cref{setup:uniformized_tp} that the initial vectors are sampled iid.

\paragraph{Inductive Step}

Here we assume the dequivalence
\begin{equation}
R:(g^{1},\ldots,g^{i})\distequiv(\mathring{g}^{1},\ldots,\mathring{g}^{i})\label{eq:scalarless_IH}
\end{equation}
as well as their moment-boundedness
and we shall show 
\[
g^{i+1}\distequiv_{R}\mathring{g}^{i+1},
\]
and their moment-boundedness,
where the conditional dequivalence can be unpacked into
\begin{equation}
(g^{1},\ldots,g^{i},g^{i+1})\distequiv(\mathring{g}^{1},\ldots,\mathring{g}^{i},\mathring{g}^{i+1}).\label{eq:scalarless_goal}
\end{equation}
For brevity, write $g\defeq g^{i+1},x\defeq x^{i+1},W\defeq W^{i+1}$
so that in this notation, by \cref{eq:tp_iteration}, $g=Wx$.

\newcommand\UU{\boldsymbol{U}}%

\newcommand\VV{\boldsymbol{V}}%

We apply Gaussian conditioning trick to obtain
\begin{align}
(g^{1},\ldots,g^{i},g) & \disteq(g^{1},\ldots,g^{i},\omega+\sigma\Pi^{\perp}z)\label{eq:gaussian_conditioning}
\end{align}
where $\omega\in\seqspace^{1},\Pi^{\perp}\in\seqspace^{2}$ and $\sigma\in\seqspace^{0}$
are functions of $g^{1},\ldots,g^{i}$, and $z\in\seqspace^{1},z\disteq\iid(\Gaus(0,1))$
is a fresh iid standard Gaussian vector, independent from everything
else.

\subsection{Proof Plan}

Below, we will give the exact formulas for $\omega,\Pi^{\perp},\sigma$
and we will show that
\begin{enumerate}
\item there is a pseudo-Lipschitz function $\psi:\R^{i}\to\R$ such that
$\omega\equiv\psi(g^{1},\ldots,g^{i})$,\footnote{with $\psi$ applied coordinatewise as usual} 
\item there is a deterministic $\mathring{\sigma}\in\R$ such that $\sigma\equiv\mathring{\sigma}$,
and 
\item $\Pi^{\perp}z\equiv z$. 
\end{enumerate}
These claims already imply that $g$ is moment-bounded.
Furthermore, applying \cref{eq:scalarless_IH} and \cref{prop:distequiv_pLip_tensor}
to \cref{eq:gaussian_conditioning}, the above claims imply
\[
(g^{1},\ldots,g^{i},g)\distequiv(\mathring{g}^{1},\ldots,\mathring{g}^{i},\psi(\mathring{g}^{1},\ldots,\mathring{g}^{i})+\mathring{\sigma}z)
\]
Using the exact formulas below, by the same calculations of \cite{yang3},
the RHS can be shown to be $\disteq(\mathring{g}^{1},\ldots,\mathring{g}^{i},\mathring{g}^{i+1})$,
yielding \cref{eq:scalarless_goal} as desired.\footnote{In particular, we do not change any of the formulas for the large
$n$ limit in our induction; we only ``upgrade the mode of convergence''
to (distributional) equivalence. So it is easy to verify this claim.}

\subsection{Exact Formulas}

\paragraph{Preliminary Definitions}

We define the following objects.
\begin{itemize}
\item Let $J$ (resp. $J'$) be the set of indices $j\le i$ such that $W^{j}=W^{i+1}$
(resp. $W^{j}=W^{i+1\trsp}$), so that $g^{j}=Wx^{j}$ (resp. $g^{j}=W^{\trsp}x^{j}$)
by construction. 
\item Let $\gb\in\seqspace^{1}\otimes\R^{J}$ (resp. $\xx\in\seqspace^{1}\otimes\R^{J}$)
be the matrix with column vectors $g^{j}$ (resp. $x^{j}$) for all
$j\in J$, so that $\gb=W\xx$. (Note $\gb$ is a distinct object
from and does not contain $g=g^{i+1}$; likewise for $\xx$ and $x=x^{i+1}$)
\item Likewise, let $\hh\in\seqspace^{1}\otimes\R^{J'}$ (resp. $\yy\in\seqspace^{1}\otimes\R^{J'}$)
be the matrix with column vectors $g^{j}$ (resp. $x^{j}$) for all
$j\in J'$, so that $\hh=W^{\trsp}\yy$.
\end{itemize}

\paragraph{Exact Formulas}

Then, by the same calculations as in \cite{yang3}, $\omega,\Pi^{\perp},\sigma$
have the following exact formulas: 
\begin{align}
\omega & =(\gb\xx^{+}+\yy^{+\trsp}\hh^{\trsp}-\yy^{+\trsp}\hh^{\trsp}\xx\xx^{+})x\label{eq:omega_exact_formula}
\end{align}
and 
\begin{align}
\sigma & =\sqrt{\frac{\|\Pi_{\xx}^{\perp}x\|^{2}}{n}}\label{eq:sigma_exact_formula}\\
\text{with }\Pi_{\xx}^{\perp} & \defeq I-\xx\xx^{+}\nonumber 
\end{align}
and
\begin{equation}
\Pi^{\perp}\defeq I-\yy\yy^{+}.\label{eq:Pi_perp_exact_formula}
\end{equation}

\paragraph{Further Constructions}

Now we further define $\mathring{\gb}$ to be the matrix with column
vectors $\mathring{g}^{j}$ for all $j\in J$; likewise we define
$\mathring{\xx},\mathring{\hh},\mathring{\yy}$. By induction hypothesis
(\cref{eq:scalarless_IH}), we have
\begin{equation}
(\gb,\xx,\hh,\yy)\distequiv_{R}(\mathring{\gb},\mathring{\xx},\mathring{\hh},\mathring{\yy})\label{eq:distequiv_gxhy}
\end{equation}
because of \cref{prop:distequiv_pLip_tensor} (as this relation is
a pseudo-Lipschitz image of \cref{eq:scalarless_IH}).

Now note
\[
\seqspace^{0}\otimes\R^{J\times J}\ni\frac{1}{n}\xx^{\trsp}\xx\equiv\braket{\xx}{\xx}\in\R^{J\times J}
\]
by \cref{eq:distequiv_gxhy} and \cref{lem:inner_product_det_equiv},
where the RHS is a deterministic object independent of $n$. Likewise,
\[
\seqspace^{0}\otimes\R^{J'\times J'}\ni\frac{1}{n}\yy^{\trsp}\yy\equiv\braket{\yy}{\yy}\in\R^{J'\times J'}.
\]
But both $\braket{\xx}{\xx}$ and $\braket{\yy}{\yy}$ are principal
submatrices of $\Xi$ by construction, so by Sylvester's Criterion
and \cref{assm:full_rank},
\begin{equation}
\braket{\xx}{\xx}\text{ and }\braket{\yy}{\yy}\text{ are nonsingular.}\label{eq:Lambda_nonsingular}
\end{equation}
 The significance of this is that now we can use \cref{lem:pinv_equiv}
on $\xx$ and $\yy$.

\subsection{Showing 3)}

Using \cref{eq:Pi_perp_exact_formula} and \cref{lem:pinv_equiv}, we
have

\[
\Pi^{\perp}z=z-\yy\yy^{+}z\equiv z-\yy\braket{\yy}{\yy}^{-1}\braket{\yy}z
\]
But $\ket z$ is independent from $\ket{\yy}$ (because $z$ is independent
from $\yy$ by construction) and is distributed as $\Gaus(0,1)$,
so the $\braket{\yy}z$ vanishes. In conclusion,
\[
\Pi^{\perp}z\equiv z
\]
as desired. This shows 3).

\subsection{Showing 2)}

Writing out the definition of $\sigma^{2}$ (\cref{eq:sigma_exact_formula}),
we see
\begin{align*}
\sigma^{2} & =\frac{\|\Pi_{\xx}^{\perp}x\|^{2}}{n}=\frac{1}{n}x^{\trsp}x-(\frac{1}{n}x^{\trsp}\xx)(\xx^{+}x)\\
 & \equiv\frac{1}{n}x^{\trsp}x-(\frac{1}{n}x^{\trsp}\xx)\braket{\xx}{\xx}^{-1}\braket{\xx}x\\
 & \equiv\braket xx-\braket x{\xx}\braket{\xx}{\xx}^{-1}\braket{\xx}x\defeq\mathring{\sigma}^{2}
\end{align*}
where the second line follows by \cref{lem:pinv_equiv}, and the third
line follows by \cref{eq:distequiv_gxhy} and \cref{lem:inner_product_det_equiv}.

Note that $\mathring{\sigma}^{2}$ is deterministic. By \cref{assm:full_rank},
$\mathring{\sigma}^{2}>0$ (because $\mathring{\sigma}$ is a $1\times1$
Schur complement of $\Xi$). Then the square root function is locally
Lipschitz around $\mathring{\sigma}^{2}$. So by \cref{prop:localLip_equiv},
we get
\[
\sigma\equiv\mathring{\sigma}.
\]

\subsection{Showing 1)}

Finally, we can apply the same strategy to show $\omega=Ex$ is equivalent
to a linear combination of $\mathring{\gb}$ and $\mathring{\yy}$
with deterministic coefficients, albeit with slightly more calculations.

By \cref{eq:omega_exact_formula}, we can decompose
\begin{align*}
\omega & =(\gb\xx^{+}+\yy^{+\trsp}\hh^{\trsp}-\yy^{+\trsp}\hh^{\trsp}\xx\xx^{+})x\\
 & =A+B-C
\end{align*}
where
\[
A\defeq\gb\xx^{+}x,\quad B\defeq\yy(\frac{1}{n}\yy^{\trsp}\yy)^{+}(\frac{1}{n}\hh^{\trsp}x),\quad C\defeq\yy(\frac{1}{n}\yy^{\trsp}\yy)^{+}(\frac{1}{n}\hh^{\trsp}\xx)(\xx^{+}x).
\]
Then the same deduction as before (involving \cref{lem:inner_product_det_equiv}
and \cref{lem:pinv_equiv}) shows
\[
A\equiv\gb\braket{\xx}{\xx}^{-1}\braket{\xx}x,\quad B\equiv\yy\braket{\yy}{\yy}^{-1}\braket{\hh}x,\quad C\equiv\yy\braket{\yy}{\yy}^{-1}\braket{\hh}{\xx}\braket{\xx}{\xx}^{-1}\braket{\xx}x
\]

Thus, $A,B,C$, and thus $\omega$ are clearly equivalent to linear
combinations of (the columns of) $\gb$ and $\yy$. Consequently,
$\omega$ is also a pseudo-Lipschitz image of $g^{1},\ldots,g^{i}$
because $\yy$ is. This proves 1).

\section{\nexort{} Master Theorem, Vectorwise Convergence}
\label{sec:nexort_vectorwise_convergence}

\begin{defn}[Integrated Program]\label{defn:integratedprogram}
Given an uniformized \nexort{} program $\pi$ (\cref{eq:tensor_tp_iteration}),
we construct a parallel \emph{uniformized \netsort{} program $\ul{\pi}$}, called the \emph{integrated program}, as follows.
$\ul\pi$'s vectors will be
denoted with an underline $\ul g^{1},\ldots,\ul g^{M}$
(which will turn out to be equivalent to their counterparts without underlines).
\begin{itemize}
\item The initial matrices of $\pi$ and $\ul{\pi}$ are identical and using the same symbols $A^1, \ldots, A^{\na}$.
\item The initial vectors $\ul g^{1},\ldots,\ul g^{M_{0}}$ are the same
as those of the original program $g^{1},\ldots,g^{M_{0}}$.
\item new vectors $\ul g^{i}$ for $i=M_{0}+1,\ldots,M$ are generated iteratively
as follows: For each such $i$, define $\ul{\phi}^{i}:\R^{i-1}\to\R$
by 
\[
\ul{\phi}^{i}(\xbf)\defeq\EV_{\bx1\ldots \bx i}\phi^{i}(\xbf;\ket{\ul{\gb}}^{\bx{1}};\ldots;\ket{\ul{\gb}}^{\bx{i}};\mathring \cc)
\]
 for any $\xbf\in\R^{i-1}$, where $\underline{\gb}$ denotes $(\underline{g}^{1},\ldots,\underline{g}^{i-1})$, $\cc$ denotes $(c^{1},\ldots,c^{i-1})$, and $\phi^i$ is from \cref{eqn:xi_unexort}.
Then $\ul g^{i}$ is generated as
\begin{align*}
\underline{g}_{\alpha}^{i} & \gets\sum_{\beta=1}^{n}W_{\alpha\beta}^{i}\underline{x}_{\beta}^{i},\quad\text{where}\quad\underline{x}_{\alpha}^{i}=\ul{\phi}^{i}\left(\underline{\gb}_{\alpha}\right).
\end{align*}
\end{itemize}

Note all scalars $c^{i}$ of $\pi$ are replaced
by their deterministic limits $\mathring{c}^{i}\in\R$ as constructed
in \cref{defn:ket} and absorbed into the nonlinearities $\ul \phi^i$.
\end{defn}

The name \emph{integrated program} refers to the intuition that we are replacing the averaging $\langle\phi^{i}(\gb_{\alpha};\gb_{\beta_{1}};\cdots;\gb_{\beta_{i}};\cc)\rangle_{\bbeta}$ with integration $\EV_{\bx1\ldots \bx i}\phi^{i}(\gb_{\alpha};\ket\gb^{\bx1};\cdots;\ket \gb^{\bx i};\mathring \cc)$.

\begin{lemma}
\label{lem:otimes_tp_comparison}
Consider a uniformized \nexort{} program
$\pi$ in \cref{setup:uniformized_tp} and its integrated \netsort{} program $\ul{\pi}$ as constructed
above. Then for $\underline{\gb}=(\underline{g}^{1},\ldots,\underline{g}^{M})$,
we have $\ket{\gb}\disteq\ket{\ul{\gb}}$ (as random vectors in $\R^{M}$)
and 
\[
\gb\equiv\underline{\gb}\distequiv\mathring{\underline{\gb}}\distequiv\mathring{\gb}.
\]
\end{lemma}

In particular, this implies that for (non-uniformized) \nexort{} programs, $\xx \distequiv \mathring \xx$ for $\xx$ being the collection of all vectors, i.e., we have the vectorwise convergent version of the Master Theorem, a strengthening of the 2nd half of \cref{thm:MasterTheorem} (which assumes initial scalars are $\tilde O(n^{-1/2})$).

Before we prove this, let us remark that 
\cref{lem:otimes_tp_comparison} automatically tells us that identities of kets, such as \cite[Lemma L.3]{yang3}, that hold for \netsort{} programs automatically hold for \nexort{} programs:
\begin{lemma}[{{\cite[Lemma L.3]{yang3}}}]\label{lemma:adjunction}
  For any vectors $x, y$ and matrix $W$ in a \nexort{} program, the following identities hold.
  \begin{align*}
    \braket{x}{Wy} &= \braket{W^\trsp x}{y}\\
    \brakethat{x}{Wy} &= \dotbraket{W^\trsp x}{y}\\
  \end{align*}
\end{lemma}
\begin{proof}[Proof of \cref{lem:otimes_tp_comparison}]
The distributional equality $\ket{\gb}\disteq\ket{\ul{\gb}}$ follows
straightforwardly from the definition of these random vectors. Thus
we also have $\mathring{\gb}\disteq\mathring{\underline{\gb}}$ (since
they are just iid samples of $\ket{\gb}$ and $\ket{\ul{\gb}}$).
We also have from \cref{thm:scalarless_vMT} that $\underline{\gb}\distequiv\mathring{\underline{\gb}}$.
So it remains to prove that $\gb\equiv\underline{\gb}$.

We will induct on $i$ to show that $(g^{1},\ldots,g^{i})\equiv(\underline{g}^{1},\ldots,\underline{g}^{i})$.
Note that, by \cref{prop:equiv_pLip_tensor}, this implies $(x^{1},\ldots,x^{i+1})\equiv(\underline{x}^{1},\ldots,\underline{x}^{i+1})$,
and by \cref{lem:avg_equiv}, this also implies $(c^{1},\ldots,c^{i+1})\equiv(\mathring{c}^{1},\ldots,\mathring{c}^{i+1})$.

The base case of $i=M_{0}$ is trivial since both programs share their
initial vectors.

Assuming the induction hypothesis for $i$ (i.e., $(g^{1},\ldots,g^{i})\equiv(\underline{g}^{1},\ldots,\underline{g}^{i})$),
we shall prove it's also true for $i+1$: i.e., we need to show $g^{i+1}\equiv\ul g^{i+1}$
in addition.

For brevity, write $\phi=\phi^{i+1}$, $\ul{\phi}=\ul{\phi}^{i+1},W=W^{i+1}$,
$x=x^{i+1}$. Let $\gb$ denote $(g^{1},\ldots,g^{i})$; likewise
for $\underline{\gb}$. Let $\cc$ denote $(c^{1},\ldots,c^{i})$
and $\mathring{\cc}$ denote $(\mathring{c}^{1},\ldots,\mathring{c}^{i})$.
Then since $\gb\equiv\underline{\gb}$ and $\cc\equiv\mathring{\cc}$
by IH, we have by \cref{prop:equiv_pLip_tensor} that
\[
\phi(\gb;\ldots;\gb;\cc)\equiv\phi(\ul{\gb};\ldots;\ul{\gb};\mathring{\cc})\in\seqb^{i+2}.
\]
Finally, by \cref{lem:outer_power_avg_equiv_nonlin}, we get
\[
\langle\phi(\ul{\gb};\ul{\gb}_{\beta_{1}};\ldots;\ul{\gb}_{\beta_{i}};\mathring{\cc})\rangle_{\bbeta}\equiv\EV_{\bx{1}\ldots\bx{i}}\phi(\ul{\gb};\ket{\ul{\gb}}^{\bx{1}};\ldots;\ket{\ul{\gb}}^{\bx{i}};\mathring{\cc})=\ul{\phi}^{i+1}(\ul{\gb};\mathring{\cc})=\ul x.
\]
So, altogether,
\[
x=\langle\phi(\gb;\gb_{\beta_{1}};\ldots;\gb_{\beta_{i}};\cc)\rangle_{\bbeta}\equiv\ul x.
\]
Finally, by \cref{prop:S0_matmul}, we arrive at
\[
g^{i+1}=Wx\equiv W\ul x=\underline{g}^{i+1}
\]
as desired. 
\end{proof}

\begin{rem}
  In this proof, the preservation of equivalence under iid matrix multiplication \cref{prop:S0_matmul} (proved using off-the-shelf operator norm tail bounds) is crucial.
  It allowed us to circumvent a complex (and probably ill-fated) Gaussian conditioning argument inside a larger, more involved induction loop.
  This preservation property is a key factor behind the \cref{defn:vanish} of ``vanishing.''
  The slight downside to this is that initial scalars need to be $\tilde O(n^{-1/2})$ (to get the equivalence going in the first place), but that is satisfied most of the time in practice.
\end{rem}
\section{Proof of \nexort{} Master Theorem \texorpdfstring{\cref{thm:MasterTheorem}}{}}
  With \cref{lem:otimes_tp_comparison}, the only things left to prove are: 1) We can weaken the requirement that the initial scalars are $\tilde O(n^{-1/2})$ (\cref{defn:tildeO}) to just that they converge to 0 almost surely; in turn the conclusion needs to be weakened as well to almost sure convergence instead of the stronger $\tilde O(n^{-1/2})$ convergence.
  2) We can swap the Gaussian setup \cref{setup:nexort} with the non-Gaussian setup \cref{setup:nexort_nongaussian} and not only still obtain all of the above results but also show all scalars converge in $L^p$ as well, for all $p \in [1,\infty)$.

  \subsection{Relaxing Initial Scalar Requirement}

  This will follow from the ``uniformly locally Lipschitz'' property of \nexort{} programs:
  For any scalar $\theta$ in a program, we can treat $\theta$ as a random function $\theta_n(\cc)$ of initial scalars $\cc$ (with randomness coming from sampling of initial matrices and vectors) for any finite $n$.
  Then for any $\cc$, there is a neighborhood of $\cc$ and a constant $L>0$ such that $\theta_n$ is almost surely Lipschitz for all $n$ on that neighborhood with Lipschitz constant $L$.

  This property in particular implies that $\theta_n$ is equicontinuous over $n$ almost surely, and thus how fast or slow $\cc_n$ converges to $\mathring \cc$ does not affect the limit $\lim_{n \to \infty} \theta_n(\cc_n)$.

  This ``uniformly locally Lipschitz'' property can be shown by proving a sub-Weibull%
  \footnote{i.e., the probability of deviating $t$ away from the limit is bounded by a function of the form $\exp(-C (nt)^{\alpha})$ for some power $\alpha>0$ (usually $<1$).}
  tail bound on the convergence of any scalar in a program, and then performing a chaining argument \citep{talagrand_upper_2014}.
  The full details will be given in a future paper.%
  \footnote{Note that only the Dynamical Dichotomy Theorem (i.e., classification of abcd-parametrizations) needs to deal with initial scalars that converge slower than $\tilde O(n^{-1/2})$. In particular, the NT and $\mu$-limits only need the version of the Master Theorem with $\tilde O(n^{-1/2})$ initial scalars.}

\subsection{Proof of Non-Gaussian \nexort{} Master Theorem}

The non-Gaussian \netsort{} Master Theorem (\cite{tp3b}) states that, under \cref{setup:nexort_nongaussian}, the scalars of the program converge almost surely and in $L^p$ for every $p \in [1, \infty)$.
Our goal here is to adapt its proof to the more general \nexort{} case.

The overarching proof strategy of the non-Gaussian \netsort{} Master Theorem was to smoothly interpolate between the non-Gaussian program and the Gaussian program (c.f.\ \cite[Defn 5.1]{tp3b}) and to show that for large enough $n$, the scalars of the program do not change much along this interpolation.
In fact, the total change along this interpolation is $\tilde O(n^{-1/2})$ (see the last chain of equation/inequalities before \cite[Sec 5.1]{tp3b}).
So if we can adapt all of the reasoning to \nexort{}, then we can obtain the $\tilde O(n^{-1/2})$ convergence in \cref{thm:MasterTheorem} as well under \cref{setup:nexort_nongaussian}.

The key technical insight enabling all of this is a bound on the moments of mixed derivatives of a non-Gaussian program's scalars and vectors against the program's matrices.

To obtain the non-Gaussian \nexort{} Master Theorem, the main task is to generalize this key technical insight, with which the interpolation trick carries over to our case easily. This is done in \cref{lemma:apriori_moment_control,{lemma:dc_bound}} below, which are useful in their own right beyond this context.
The full proof of the non-Gaussian Master Theorem can be then be straightforwardly adapted from \cite{tp3b}.

\begin{setup}
  \label{setup:nongaussian_lax}
  Consider \cref{setup:nexort_nongaussian}, but allow the variances of matrix entries to differ from $n^{-1}$ but still bounded above by $\nu_2 n^{-1}$ for some $\nu_2 > 0$ common to all matrix entries.
\end{setup}

Below, we invoke the notion of \emph{oblivious constants} from \cite[Defn I.1]{tp3b}.
This is a technical notion needed to finish the proof of the the non-Gaussian Master Theorem, but the first time readers can ignore the comments on oblivious constants.
\newcommand{\PP}{\mathcal{P}}
\begin{lemma}[Expected Smoothness of Vectors]\label{lemma:apriori_moment_control}
  Consider a \nexort{} program under \cref{setup:nongaussian_lax}.
  Then, for any polynomially smooth $\psi: \R^{rM+M} \to \R$, any $p \ge 1$, and any multiset $\PP$ of the program's matrix entries $\{W_{\alpha\beta} \}_{\alpha,\beta, W \in \mathcal W}$,
  \begin{equation}
      \sup_{\aalpha \in [n]^r} \EV \left|\partial^\PP \psi(\gb_{\alpha_1}; \cdots; \gb_{\alpha_r}; \cc)\right|^p = O(1)
      \quad \text{as $n \to \infty$.}
      \label{eqn:apriori_moment_control}
  \end{equation}
  where constant in the big-O is $(p, \PP)$-oblivious wrt $\psi$ and the program.
\end{lemma}

This generalizes \cite[Lemma I.3]{tp3b} to \nexort{}.

\begin{proof}
  In the original proof in \cite[Sec I]{tp3b}, replace $\phi(x^{1},\ldots,x^{k})$
  with $\phi(x^{1};\cdots;x^{k})$ everywhere.
  As the original proof essentially factors through arguments about the ``supremum'' over all indices $\alpha \in [n]$, it can be adapted straightforwardly to the ``supremum'' over all multi-indices $\aalpha \in [n]^r$.
\end{proof}

\begin{lemma}[Expected Smoothness of Scalars]\label{lemma:dc_bound}
    Consider a program in \cref{setup:nongaussian_lax}.
    Then for any $p \ge 1$, any \emph{nonempty} multiset $\PP$ of the program's matrix entries $\{W_{\alpha\beta} \}_{\alpha,\beta, W \in \mathcal W}$,, and any scalar $c$ of the program,
    \begin{equation}
        \EV \left|\partial^\PP c \right|^p = O(n^{-p})
        \quad \text{as $n \to \infty$.}
    \end{equation}
    Furthermore, the constant in the big-O is $(p, \PP)$-oblivious wrt the program.
\end{lemma}

This generalizes \cite[Lemma J.6]{tp3b} to \nexort{}.

\begin{proof}
  Construct the backpropagation program (\cref{defn:backpropagation}), which is another \nexort{} program.
  Then proceed as in the proof of \cite[Lemma J.6]{tp3b}.
\end{proof}

\chapter{Experiments}

\section{Numerical Verification}

We perform numerical experiments to validate our theory. It is intractable to compute the exact infinite-width limits for general $\QQ$, since the expectations required to evaluate the infinite-width dynamics in both limits do not admit an analytical solution (even for the Neural Tangent Limit). We thus employ Monte Carlo simulations to approximate these expectations.

We test \cref{{thm:NT_MLP_memoryful}} and \cref{thm:mulimit_MLP_general} by training ReLU MLPs ($L=4$ for NTP and $L=2$ for $\mu$P) on $\R^{10}$ Gaussian inputs and a unit output. We use the standard L2 loss function, regressing to random targets. We train networks with different widths using the Adam optimizer with a learning rate of $\eta = 0.2$ and $\epsilon = 10^{-4}$, using 100 training samples and conducting 10 trials for each width. To account for varying initial outputs and loss derivatives per weight initialization, we subtract the initialized network output from the output for each sample, ensuring that the output is zero at initialization for all inputs.

We approximate the infinite-width training dynamics by estimating the expectations in \cref{defn:TangentOperator} and \cref{thm:mulimit_MLP_general} using Monte Carlo simulations. As the initial loss derivatives are deterministic with zero outputs, the infinite-width dynamics can be estimated without actually constructing a network. To compare the finite and infinite width neural networks' evolution, we assess the output on random inputs at each iteration. Our results are summarized in \cref{fig:antk} and \cref{fig:mu}. As anticipated, the training dynamics converge to the infinite-width dynamics as the width increases.

\begin{figure*}[h]
\centering
  \begin{tabular}{ccc}
    \includegraphics[width=0.33\linewidth]{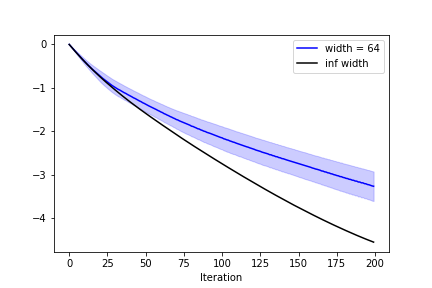} & 
    \includegraphics[width=0.33\linewidth]{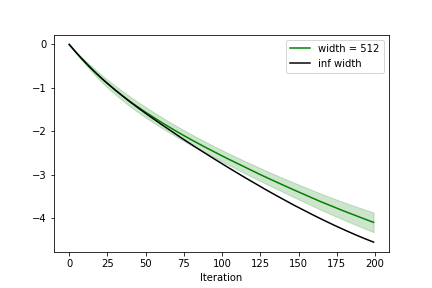} &  
    \includegraphics[width=0.33\linewidth]{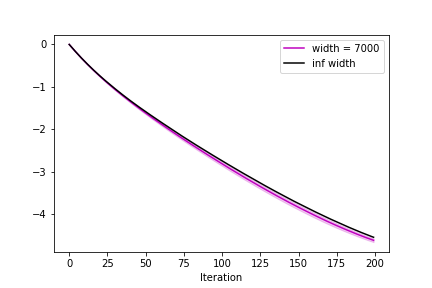} 
   \\
    \includegraphics[width=0.33\linewidth]{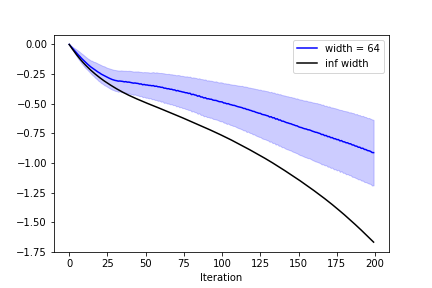} & 
    \includegraphics[width=0.33\linewidth]{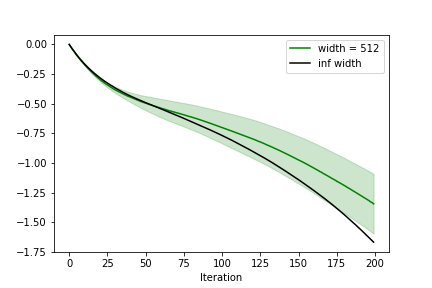} &  
    \includegraphics[width=0.33\linewidth]{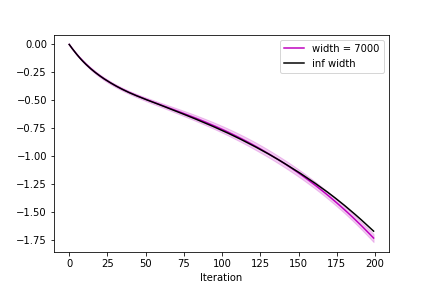} 
    \\
    \includegraphics[width=0.33\linewidth]{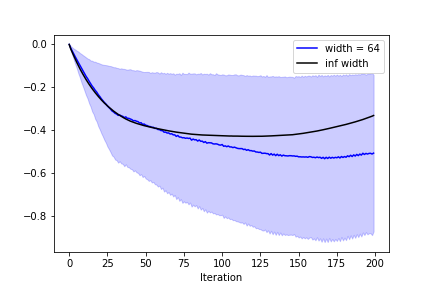} & 
    \includegraphics[width=0.33\linewidth]{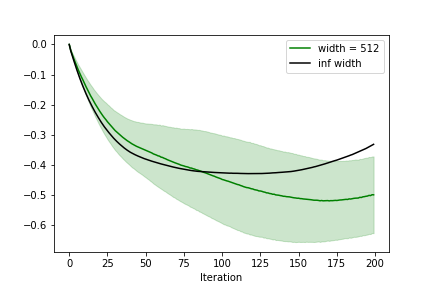} &  
    \includegraphics[width=0.33\linewidth]{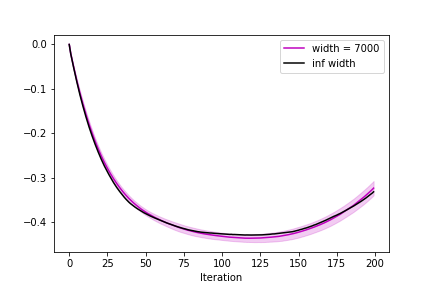} 
    \\
    \includegraphics[width=0.33\linewidth]{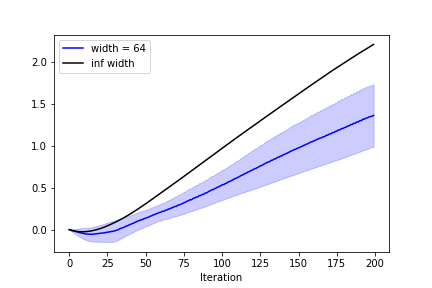} & 
    \includegraphics[width=0.33\linewidth]{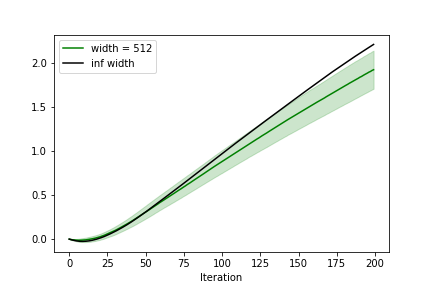} &  
    \includegraphics[width=0.33\linewidth]{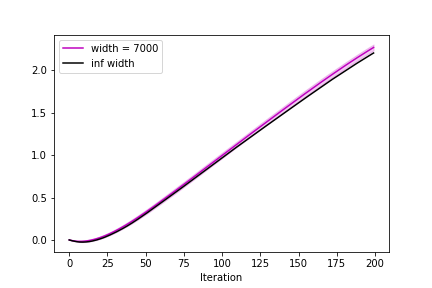} \\
    (a) & (b) & (c)\\
  \end{tabular}
  \caption{Adam training dynamics of finite and infinite-width networks in NTP. We train networks of widths 64 (a), 512 (b), 7000 (c), and track the outputs for 4 random inputs (one per row) at each iteration as the network trains. We compute the output distribution over 10 independent runs for each network, and compare with the infinite-width dynamics (black curve). As the width grows, the network function converges to that of the infinite-width dynamics captured in \cref{thm:NT_MLP_memoryful}.} 
  \label{fig:antk}
\end{figure*}

\begin{figure*}[h]
\centering
  \begin{tabular}{ccc}
    \includegraphics[width=0.33\linewidth]{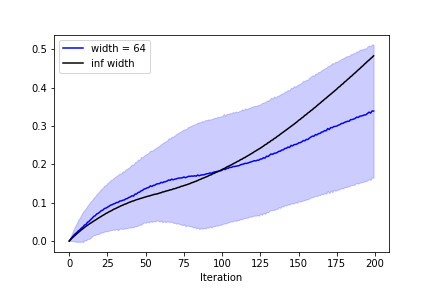} & 
    \includegraphics[width=0.33\linewidth]{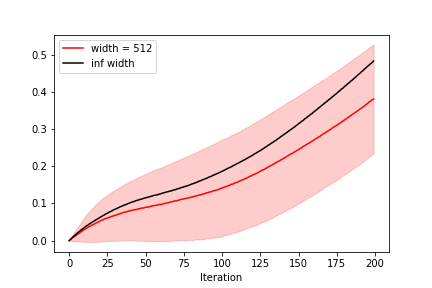} &  
    \includegraphics[width=0.33\linewidth]{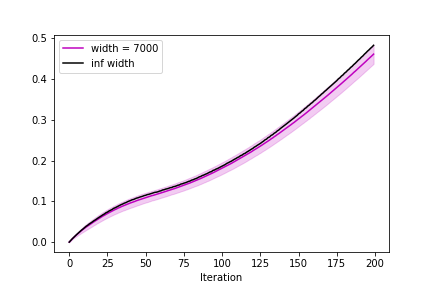} 
   \\
   \includegraphics[width=0.33\linewidth]{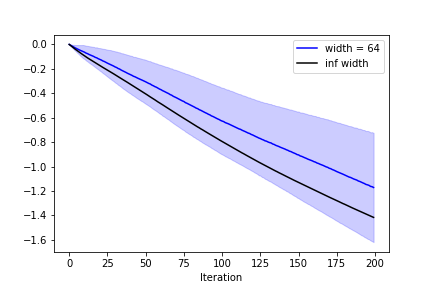} & 
    \includegraphics[width=0.33\linewidth]{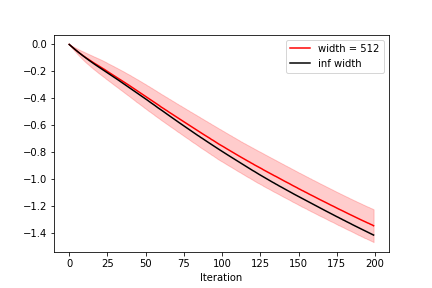} &  
    \includegraphics[width=0.33\linewidth]{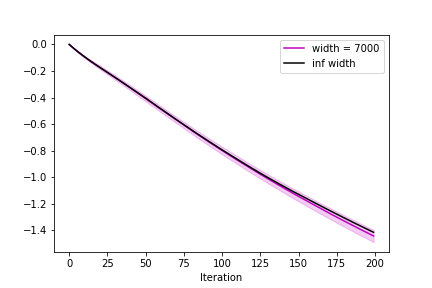} 
    \\
   \includegraphics[width=0.33\linewidth]{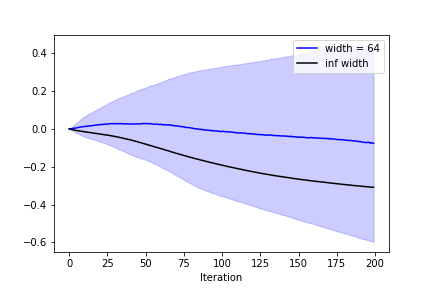} & 
    \includegraphics[width=0.33\linewidth]{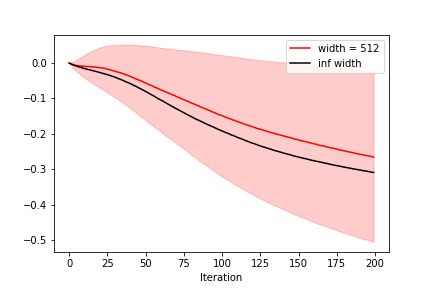} &  
    \includegraphics[width=0.33\linewidth]{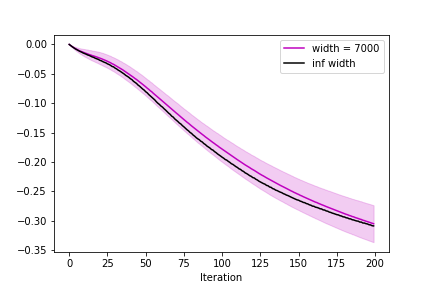} 
    \\
   \includegraphics[width=0.33\linewidth]{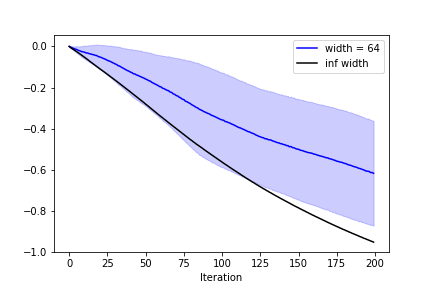} & 
    \includegraphics[width=0.33\linewidth]{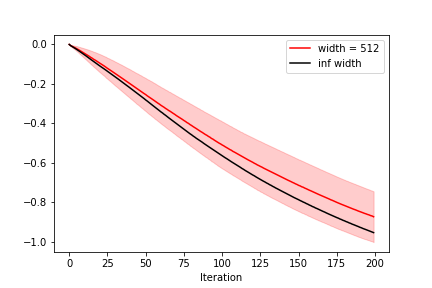} &  
    \includegraphics[width=0.33\linewidth]{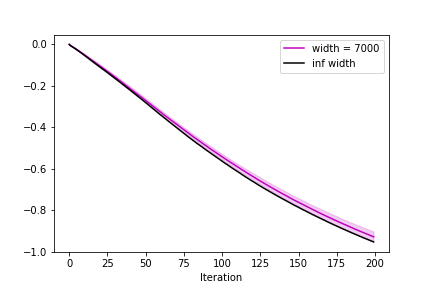} \\
    (a) & (b) & (c)\\
  \end{tabular}
  \caption{Adam training dynamics of finite and infinite-width networks in $\mu$P. We train networks of widths 64 (a), 512 (b), 7000 (c), and track the outputs for 4 random inputs (one per row) at each iteration as the network trains. We compute the output distribution over 10 independent runs for each network, and compare with the infinite-width dynamics (black curve). As the width grows, the network function converges to that of the infinite-width dynamics captured in \cref{thm:mulimit_MLP_general}.} 
  \label{fig:mu}
\end{figure*}
\clearpage
\section*{Acknowledgements}
In alphabetical order, we would like to thank Jeremy Bernstein, Nikhil Ghosh, Dror Ironi, Ariel Landau, Sadhika Malladi, Jamie Simon, and Josh Susskind for providing insightful comments and discussion.
\bibliography{iclr2023_conference}
\bibliographystyle{iclr2023_conference}

\newpage
\appendix

\newpage

\end{document}